%% file: main.tex
\pdfoutput=1
\documentclass{article}
\usepackage{microtype}
\usepackage{graphicx}
\usepackage{booktabs} %
\usepackage{multirow}
\usepackage{multicol}
\usepackage{tablefootnote}
\usepackage{xcolor}
\usepackage{fix-cm} %
\DeclareMathSizes{10}{10}{6}{4}
\usepackage{hyperref}

\usepackage[accepted]{icml2024}

\usepackage{amsmath}
\usepackage{amssymb}
\usepackage{mathtools}
\usepackage{amsthm}
\usepackage{caption}
\usepackage{subcaption}
\usepackage{booktabs}
\usepackage[flushleft]{threeparttable}
\usepackage{MnSymbol}
\usepackage{lipsum}

\usepackage{enumitem}

\usepackage[ruled, vlined, linesnumbered]{algorithm2e} %
\usepackage{mathtools}
\usepackage{array}
\usepackage{lipsum}
\usepackage{url}
\usepackage{tablefootnote}
\usepackage[export]{adjustbox}
\usepackage{footmisc}
\usepackage{balance}
\usepackage{amsmath,amssymb,amsfonts}%
\usepackage{bm}

\usepackage{contour}
\usepackage[bb=dsserif]{mathalpha}
\usepackage{bm}
\usepackage{bigstrut}
\usepackage{dirtytalk}
\usepackage{enumitem}
\usepackage{colortbl}
\usepackage{xurl}
\usepackage{soul}
\usepackage{adjustbox}
\usepackage{tcolorbox}
\usepackage{ulem}

\usepackage[capitalize,noabbrev]{cleveref}
\crefformat{footnote}{#2\footnotemark[#1]#3}

\makeatletter
\newtheorem*{rep@theorem}{\rep@title}
\newcommand{\newreptheorem}[2]{%
\newenvironment{rep#1}[1]{%
 \def\rep@title{#2 \ref{##1}}%
 \begin{rep@theorem}}%
 {\end{rep@theorem}}}
\makeatother

\theoremstyle{plain}
\newtheorem{theorem}{Theorem}
\newreptheorem{theorem}{Theorem}

\newtheorem{lemma}{Lemma}
\newreptheorem{lemma}{Lemma}
\newtheorem{idea}{Target}

\newreptheorem{corollary}{corollary}
\theoremstyle{definition}
\newtheorem{definition}{Definition}

\theoremstyle{remark}
\newtheorem{remark}{Remark}
\newtheorem{tgt}{High-level Target}

\usepackage[textsize=tiny]{todonotes}

\input{dfn.tex}

\icmltitlerunning{Tackling \revise{Prevalent} Conditions in Unsupervised Combinatorial Optimization}

\begin{document}

\twocolumn[
\icmltitle{Tackling \revise{Prevalent} Conditions in Unsupervised Combinatorial Optimization: 
Cardinality, Minimum, Covering, and More
}

\icmlsetsymbol{equal}{\dagger}

\begin{icmlauthorlist}
\icmlauthor{Fanchen Bu}{KAIST_EE,AI4CO}
\icmlauthor{Hyeonsoo Jo}{KAIST_AI}
\icmlauthor{Soo Yong Lee}{KAIST_AI}
\icmlauthor{Sungsoo Ahn}{POSTECH_AI,POSTECH_ECE}
\icmlauthor{Kijung Shin}{KAIST_AI,KAIST_EE}
\end{icmlauthorlist}

\icmlaffiliation{KAIST_AI}{Kim Jaechul Graduate School of Artificial Intelligence, KAIST, Seoul, Republic of Korea}
\icmlaffiliation{KAIST_EE}{School of Electrical Engineering, KAIST, Daejeon, Republic of Korea}
\icmlaffiliation{POSTECH_AI}{Graduate School of Artificial Intelligence, POSTECH, Pohang, Republic of Korea}
\icmlaffiliation{POSTECH_ECE}{Department of Computer Science and Engineering, Pohang, Republic of Korea}
\icmlaffiliation{AI4CO}{AI4CO Open-Source Community}

\icmlcorrespondingauthor{Kijung Shin}{kijungs@kaist.ac.kr}

\icmlkeywords{}

\vskip 0.3in
]

\printAffiliationsAndNotice{}  %

\begin{abstract}
\input{000abstr.tex}
\end{abstract}

\input{010intro.tex}

\input{020prelm.tex}

\input{040anlys.tex}
\input{050probl.tex}

\input{060exper.tex}
\input{070concl.tex}

\clearpage
\newpage

\section*{Acknowledgments}
The authors thank all the anonymous reviewers for their helpful comments.

The authors thank Dr. Runzhong Wang @ MIT, Federico Berto @ KAIST, Chuanbo Hua @ KAIST, and Dr. Junyoung Park @ Qualcomm for constructive discussions.

Fanchen Bu gives special thanks to Prof. Jaehoon Kim @ KAIST, Dr. Hong Liu @ IBS, and Yuhao Yao @ Huawei from whom Fanchen Bu studied the probabilistic method.

This work was supported by Institute of Information \& Communications Technology Planning \& Evaluation (IITP) grant funded by the Korea government (MSIT) (No. 2022-0-00871, Development of AI Autonomy and Knowledge Enhancement for AI Agent Collaboration) (No. 2019-0-00075, Artificial Intelligence Graduate School Program (KAIST)) (No. 2019-0-01906, Artificial Intelligence Graduate School Program (POSTECH)).

\section*{Impact Statement}
This paper presents work whose goal is to advance the field of Machine Learning, especially Machine Learning for Combinatorial Optimization.
There are many potential societal consequences of our work, none of which we feel must be specifically highlighted here.

\normalem
\bibliography{ref}
\bibliographystyle{icml2024}

\clearpage
\appendix   
\onecolumn

\input{A00appdx.tex}

\end{document}

%% file: dfn.tex
\def\mydefbb#1{\expandafter\def\csname bb#1\endcsname{\ensuremath{\mathbb{#1}}}}
\def\mydefallbb#1{\ifx#1\mydefallbb\else\mydefbb#1\expandafter\mydefallbb\fi}
\mydefallbb ABCDEFGHIJKLMNOPQRSTUVWXYZ\mydefallbb

\def\mydefcal#1{\expandafter\def\csname cal#1\endcsname{\ensuremath{\mathcal{#1}}}}
\def\mydefallcal#1{\ifx#1\mydefallcal\else\mydefcal#1\expandafter\mydefallcal\fi}
\mydefallcal ABCDEFGHIJKLMNOPQRSTUVWXYZ\mydefallcal

\newcommand{\Abs}[1]{\lvert #1 \rvert}

\newcommand{\Set}[1]{\{#1\}}

\newcommand{\PB}{\ensuremath{\operatorname{PoiBin}}}

\makeatletter
\let\save@mathaccent\mathaccent
\newcommand*\if@single[3]{%
  \setbox0\hbox{${\mathaccent"0362{#1}}^H$}%
  \setbox2\hbox{${\mathaccent"0362{\kern0pt#1}}^H$}%
  \ifdim\ht0=\ht2 #3\else #2\fi
  }
\newcommand*\rel@kern[1]{\kern#1\dimexpr\macc@kerna}
\newcommand*\widebar[1]{\@ifnextchar^{{\wide@bar{#1}{0}}}{\wide@bar{#1}{1}}}
\newcommand*\wide@bar[2]{\if@single{#1}{\wide@bar@{#1}{#2}{1}}{\wide@bar@{#1}{#2}{2}}}
\newcommand*\wide@bar@[3]{%
  \begingroup
  \def\mathaccent##1##2{%
    \let\mathaccent\save@mathaccent
    \if#32 \let\macc@nucleus\first@char \fi
    \setbox\z@\hbox{$\macc@style{\macc@nucleus}_{}$}%
    \setbox\tw@\hbox{$\macc@style{\macc@nucleus}{}_{}$}%
    \dimen@\wd\tw@
    \advance\dimen@-\wd\z@
    \divide\dimen@ 3
    \@tempdima\wd\tw@
    \advance\@tempdima-\scriptspace
    \divide\@tempdima 10
    \advance\dimen@-\@tempdima
    \ifdim\dimen@>\z@ \dimen@0pt\fi
    \rel@kern{0.6}\kern-\dimen@
    \if#31
      \overline{\rel@kern{-0.6}\kern\dimen@\macc@nucleus\rel@kern{0.4}\kern\dimen@}%
      \advance\dimen@0.4\dimexpr\macc@kerna
      \let\final@kern#2%
      \ifdim\dimen@<\z@ \let\final@kern1\fi
      \if\final@kern1 \kern-\dimen@\fi
    \else
      \overline{\rel@kern{-0.6}\kern\dimen@#1}%
    \fi
  }%
  \macc@depth\@ne
  \let\math@bgroup\@empty \let\math@egroup\macc@set@skewchar
  \mathsurround\z@ \frozen@everymath{\mathgroup\macc@group\relax}%
  \macc@set@skewchar\relax
  \let\mathaccentV\macc@nested@a
  \if#31
    \macc@nested@a\relax111{#1}%
  \else
    \def\gobble@till@marker##1\endmarker{}%
    \futurelet\first@char\gobble@till@marker#1\endmarker
    \ifcat\noexpand\first@char A\else
      \def\first@char{}%
    \fi
    \macc@nested@a\relax111{\first@char}%
  \fi
  \endgroup
}
\makeatother

\definecolor{kjred}{RGB}{228, 26, 28}
\definecolor{kjblue}{RGB}{55,126,184}
\definecolor{kjgreen}{RGB}{77,175,74}
\definecolor{kjpurple}{RGB}{152,78,163}
\definecolor{kjorange}{RGB}{255,127,0}

\definecolor{Red}{RGB}{244, 124, 124}
\definecolor{Green}{RGB}{162, 222, 147}
\definecolor{Blue}{RGB}{112, 161, 211}

\newcommand{\smallsection}[1]{{\noindent {\bf{\underline{\smash{#1}}}}}}

\definecolor{peace}{RGB}{228, 26, 28}
\definecolor{love}{RGB}{55, 126, 184}
\definecolor{joy}{RGB}{77, 175, 74}
\definecolor{kindness}{RGB}{152, 78, 163}

\newcommand{\der}{\operatorname{der}}%



\newcommand{\argmin}{\mathop{\mathrm{arg\,min}}}

\newcommand{\bus}[1]{\textbf{\underline{\smash{#1}}}}

\newcommand{\ours}{\textsc{UCom2}\xspace}

\newcommand{\citenameyear}[1]{\citeauthor{#1} \citeyearpar{#1}}

\renewcommand{\cite}{\citenameyear}

\newcommand{\revise}{}

%% file: 000abstr.tex
Combinatorial optimization (CO) is naturally discrete, making machine learning based on differentiable optimization inapplicable.
\citet{karalias2020erdos} adapted the probabilistic method 
to incorporate CO into differentiable optimization.
Their work ignited the research on unsupervised learning for CO, composed of two main components: probabilistic objectives and derandomization.
However, each component confronts unique challenges.
First, deriving objectives 
under \revise{various} conditions \revise{(e.g., cardinality constraints and minimum)}
is nontrivial.
Second, the derandomization process is underexplored, and the existing derandomization methods are either random sampling or naive rounding.
In this work, we aim to tackle \revise{prevalent (i.e., commonly involved)} conditions in unsupervised CO.
First, we concretize the targets for objective construction and derandomization with theoretical justification.
Then, for various conditions commonly involved in different CO problems, we derive nontrivial objectives and derandomization to meet the targets.
Finally, we apply the derivations to various CO problems.
Via extensive experiments on synthetic and real-world graphs, we validate the correctness of our derivations and show our empirical superiority w.r.t. both optimization quality and speed.

%% file: 010intro.tex
\section{Introduction}

Combinatorial optimization (CO) problems are \textit{discrete} by their nature.
Machine learning methods are based on differentiable optimization (e.g., gradient descent),
and applying them to CO is non-trivial.
In their pioneering work, \citet{karalias2020erdos} adapted the probabilistic method~\citep{erdos1974probabilistic,alon2016probabilistic} to incorporate discrete CO problems into differentiable optimization.
Specifically, they proposed to {evaluate} CO objectives {on a \textit{distribution} of discrete choices} (i.e., in a \textit{probabilistic} manner),
allowing for the differentiable optimization-based ML techniques to be applied to CO problems.
This ignited the line of research on unsupervised (i.e., not supervised by solutions) learning for combinatorial optimization (UL4CO).

There are two components in UL4CO:
(1) construction of \textit{probabilistic objectives} 
and
(2) \textit{derandomization} to obtain the final {discrete solutions}.
However, the prior works on UL4CO share multiple limitations.
First, although some desirable properties 
of probabilistic objectives 
\revise{(e.g., desirable objectives should be differentiable and align well with the original discrete objectives)} 
have been proposed,
how to derive objectives satisfying such properties is still unclear.
At the same time, the derandomization process is underexplored, without many practical techniques or theoretical discussions.
Specifically, the existing derandomization methods
are either random sampling or naive rounding.
Random sampling, by its nature, may cost us a large number of samplings (and good luck) to have good results.
For naive rounding, the performance may highly depend on the order of rounding and end up with mediocre solutions.
They only guarantee, at best, derandomized solutions are no worse than the given continuous solutions w.r.t. the corresponding probabilistic objectives. However, how to obtain stronger guarantees in an efficient way has been an open problem.

Motivated by the limitations, 
{\revise{in this work, we focus on objectives and constraints that have not been systematically handled within the UL4CO framework and are commonly involved in various CO problems.}}
We study and propose \ours (\bus{U}supervised \bus{Com}binatorial Optimization \bus{U}nder \revise{\bus{Com}monly-involved} Conditions).
Specifically, our contributions are four-fold.

\begin{itemize}[leftmargin=*,topsep=0pt]
\setlength\itemsep{0em}
\item {\textbf{We concretize the targets for objective construction and derandomization with theoretical justification (Sec.~\ref{sec:proposed_method}).}
We theoretically show that probabilistic objectives that can be rephrased as an expectation are desirable,
and propose a {fast} and {effective} derandomization scheme with a quality guarantee stronger than the existing ones.}
\item {\textbf{We derive non-trivial objectives and derandomization for various \revise{prevalent} conditions to meet the targets (Sec.~\ref{sec:analy_conds}).}
We focus on conditions that are mathematically hard to handle but commonly involved in CO problems, e.g., cardinality constraints, minimum, and covering.}
\item {\textbf{We apply our derivations to different CO problems involving such \revise{prevalent} conditions (Sec.~\ref{sec:problems}).} 
For each problem, we analyze what conditions are involved and derive objectives and derandomization by combining our derivations for the involved conditions.}
\item \textbf{We show the empirical superiority of \ours via experiments (Sec.~\ref{sec:experiments}).} 
Equipped with our derivations, our method \ours achieves better optimization quality and speed across different CO problems on both synthetic and real-world graphs, outperforming various baselines.
\end{itemize}

\smallsection{Reproducibility.}
The code and datasets are available in the online appendix~\citep{appendix}.\footnote{\url{https://github.com/ai4co/unsupervised-CO-ucom2}}

%% file: 020prelm.tex
\section{Preliminaries and Background}\label{sec:prelim}

\subsection{Preliminaries}\label{subsec:notations}
\smallsection{Graphs.}
A \textit{graph} $G = (V, E, W)$ is defined by a \textit{node set} $V$, an \textit{edge set} $E$, and \textit{edge weights} $W: E \to \bbR$.
We let $n = |V|$ denote the number of nodes (WLOG, $V = [n] \coloneqq \Set{1, 2, \ldots, n}$), and let $m = |E|$ denote the number of edges.

\smallsection{Combinatorial optimization (CO).} 
We consider {CO} problems on graphs with discrete \textit{decisions} on nodes.
Each CO problem can be represented by a tuple $(f, \calC, d)$ with
(1) an \textit{optimization objective} $f: d^n \to \bbR_{+}$,
(2) \textit{constraints} defined by a \textit{feasible set} $\calC \subseteq d^n$,
and 
(3) a set of \textit{possible decisions} $d$ (on each $v \in V$).
Given decisions $X_v \in d$ with $v \in V$, we have a \textit{full decision} $X \in d^n$.

For each graph $G = (V, E, W)$,
we can use the optimization objective function $f$ to evaluate each full decision $X \in d^n$ on $G$ by $f(X; G)$, and we aim to \textit{solve} $\min_{X \in \calC(G)} f(X; G)$.
By default, we consider CO problems with \textit{binary} decisions (i.e., $d = \Set{0, 1}$).\footnote{We will discuss non-binary cases in Sec.~\ref{subsec:analy_conds:coloring}.}
Given $X \in \Set{0, 1}^{n}$,
we call each node $v$ with $X_v = 1$ a \textit{chosen node}, and
call $V_X \coloneqq \Set{v \in V: X_v = 1} \subseteq V$ the \textit{chosen subset} (i.e., the set of chosen nodes).

\subsection{Background: UL4CO}\label{subsec:background_related_work}
We shall introduce the background of unsupervised learning for combinatorial optimization (UL4CO), including the overall pipeline and some existing ideas/techniques.

\subsubsection{The UL4CO Pipeline: Erd\H{o}s Goes Neural}\label{subsubsec:EGN_pipeline}

The UL4CO pipeline, Erd\H{o}s Goes Neural~\citep{karalias2020erdos}, is based on the probabilistic method~\citep{erdos1974probabilistic} with three components:
objective construction,
differentiable optimization, and
derandomization.

\smallsection{Probabilistic objective construction.}
The high-level idea is to evaluate discrete objectives on a distribution of decisions, which accepts continuous parameterization.
Specifically, given a CO problem $(f: \Set{0, 1}^n \to \bbR, \calC, d=\Set{0, 1})$,
we first construct a \textit{penalized} objective $f_{\mathrm{pen}}(X) = f(X) + \beta \mathbb{1}(X \notin \calC)$ with 
\textit{constraint coefficient} $\beta > 0$.
Then, a \textit{probabilistic objective} $\Tilde{f}: [0, 1]^n \to \bbR$ accepting probabilistic (and thus continuous) inputs is constructed such that
\begin{center}
\vspace{-1mm}
   $\Tilde{f}(p) \geq 
\bbE_{X \sim p} f_{\mathrm{pen}}(X) = 
\bbE_{X \sim p} f(X) + \beta \Pr\nolimits_{X \sim p}[X \notin \calC].$
\vspace{-1mm}
\end{center}

We see each $p \in [0, 1]^n$ as a vector of probabilities, with $p_v$'s being \textit{independent Bernoulli} variables.
Hence, we have
$\begin{aligned}
\Pr\nolimits_{p}[X] &= \prod\nolimits_{v \in V_X} p_v \prod\nolimits_{u \in V \setminus V_X} (1 - p_u),\\
\bbE_{X \sim p} f(X) &= \sum\nolimits_{X \in \Set{0,1}^n} \Pr\nolimits_{p}[X] f(X), \text{~and~}\\
\Pr\nolimits_{X \sim p}[X \notin \calC] &= \sum\nolimits_{X \in \Set{0,1}^n \setminus \calC} \Pr\nolimits_{p}[X] = 1 - \sum\nolimits_{X \in \calC} \Pr\nolimits_{p}[X].
\end{aligned}$

\begin{remark}
    {\revise{Assuming independent Bernoulli variables gives simplicity and tractability, while other ways to model decision distributions, e.g., other distributions~\citep{karalias2022neural} and dependency between variables~\citep{sanokowski2023variational}, are potential
    future directions.}}
\end{remark}

\smallsection{Differentiable optimization.}
For \textit{differentiable} optimization, we need to ensure that $\tilde{f}$ is differentiable (w.r.t. $p$).
At this moment, let us assume we have constructed such a $\tilde{f}$.
Then, given a graph $G$, we can use differentiable optimization (e.g., gradient descent) to obtain optimized probabilities $p_{\mathrm{o}}$ with
(ideally) small $\Tilde{f}(p_{\mathrm{o}}; G)$.

\smallsection{Derandomization.}
Finally, \textit{derandomization} is used to obtain deterministic full decisions.
For each test instance $G$, 
the derandomization process transforms each $p_{\mathrm{o}} \in [0,1]^n$ obtained by {probabilistic} optimization into a \textit{discrete} (i.e., deterministic) full decision $X_p \in \Set{0,1}^n$.
\citenameyear{karalias2020erdos} showed a quality guarantee of derandomization by \textit{random sampling}.
See App.~\ref{app:background} for more details.

\subsubsection{Local Derandomization}\label{subsubsec:entrywise_concave}
The theoretical quality guarantee by \citet{karalias2020erdos} is obtained by random sampling, and we may need a large number of samplings (and good luck) to have a good bound.
\citet{wang2022unsupervised} further 
\revise{proved a \textit{deterministic} (i.e., not relying on random sampling) quality guarantee by \textit{iterative rounding} (i.e., a series of local derandomization along with a node enumeration).}
The principle of iterative rounding involves two concepts: (1) \textit{local derandomization} of probabilities $p$ and 
(2) \textit{entry-wise concavity} of probabilistic objective $\Tilde{f}$.

\smallsection{Local derandomization.}
Given $p \in [0,1]^n$, $i \in [n]$, and $x \in \Set{0, 1}$,
$\operatorname{der}(i, x; p) \in [0, 1]^n$ is the
result after $p_i$ being \textit{locally derandomized} as $x$, i.e.,
\begin{center}    
$\begin{cases}
 \operatorname{der}(i, x; p)_i = x,\\
 \operatorname{der}(i, x; p)_j = p_j, \forall j \neq i.
\end{cases}$
\end{center}

\smallsection{Entry-wise concavity.}
A probabilistic objective $\Tilde{f}: [0, 1] \to \bbR$ is \textit{entry-wise concave} if $\forall p \in [0,1]^n$ and  $i \in [n]$,
\begin{center}
\vspace{-1mm}
    $p_i \Tilde{f}(\operatorname{der}(i, 1; p)) + (1 - p_i) \Tilde{f}(\operatorname{der}(i, 0; p)) \leq \Tilde{f}(p)$.
    \vspace{-1mm}
\end{center}
Applying a series of local derandomization with an entry-wise concave objective $\tilde{f}$ 
does not increase the objective.
{\revise{Notably, \citet{karalias2020erdos} essentially proposed iterative rounding, and \citet{wang2022unsupervised} first formalized a theoretical guarantee of iterative rounding with the condition of entry-wise concavity.}}
See App.~\ref{app:background} for more details.

%% file: 040anlys.tex
\section{Concretizing Targets: What Do We Need?}\label{sec:proposed_method}
{First, we concretize the targets for objective construction and derandomization} to guide our further derivations.

\subsection{Good objectives: Expectations are all you need}\label{subsec:idea:objective}

\smallsection{Good properties.}
We summarize some known \textit{good} properties of a probabilistic objective $\Tilde{f}$~\citep{karalias2020erdos, wang2022unsupervised}: 
\textbf{(P1)} $\Tilde{f}: [0, 1]^n \to \bbR$ accepts \textit{continuous} inputs $p \in [0, 1]^n$ (rather than discrete $X \in \Set{0, 1}^n$);
\textbf{(P2)} $\Tilde{f}$ is an \textit{upper bound} of the expectation of a penalized objective $f + \beta \mathbb{1}(X \notin \calC)$ for some $\beta > 0$;
\textbf{(P3)} $\tilde{f}$ is \textit{differentiable} w.r.t. $p$;
\textbf{(P4)} $\Tilde{f}$ is \textit{entry-wise concave} w.r.t. $p$;
\textbf{(P5)} $\Tilde{f}$ has the same minimum as $f$, i.e., $\min_{p} \Tilde{f}(p) = \min_{X} f(X)$
and $\arg\min_{p} \Tilde{f}(p) = \arg\min_{X} f(X)$.
The property (P5) has been discussed~\citep{karalias2020erdos,karalias2022neural,kollovieh2024expected} but has not been explicitly formalized for UL4CO.
With (P5), when we minimize $\Tilde{f}$, we also minimize 
the original objective $f$, which avoids meaningless $\Tilde{f}$, e.g., a constant function with a very high value (which satisfies (P1)-(P4) but not (P5)).
\begin{tcolorbox}[boxsep=0pt]
\small
\begin{tgt}[Good objectives]\label{target:principled_obj}
Given an optimization objective $f: \Set{0, 1}^n \to \bbR$ and constraints $X \in \calC$, 
we aim to construct a good probabilistic objective $\tilde{f}: [0, 1]^n \to \bbR$ to satisfy \textit{all the good properties} (P1)-(P5).
\end{tgt}
\end{tcolorbox}
Below, we show that a specific form of objectives satisfies all the good properties.
First, \textit{expectations are all you need}, i.e., any probabilistic objective that is the expectation of any discrete function satisfies properties
(P1), (P3), and (P4).

\begin{theorem}[Expectations are all you need]\label{thm:concave_exp_prob}
    For any $g: \Set{0, 1}^n \to \bbR$,
    $\tilde{g}: [0, 1]^n \to \bbR$ with $\Tilde{g}(p) = \bbE_{X \sim p} g(X)$
    is differentiable and entry-wise concave w.r.t. $p$.
\end{theorem}
\begin{proof}
    See App.~\ref{app:proofs} for all the proofs.
\end{proof}
\begin{remark}
    Differentiability and entry-wise concavity are closed under addition.
Hence, a linear combination of expectations is also differentiable and entry-wise concave.
Also, probabilities are special expectations of indicator functions.
\revise{The differentiability of expectation may not hold when $p_v$'s are not independent Bernoulli variables, e.g., when the expectation is taken with Lovasz extension~\citep{bach2013learning}.}
\end{remark}

To further satisfy (P2) and (P5),
we only need to find a tight upper bound (TUB) of a penalized objective.
\begin{definition}[Tight upper bounds]\label{def:tight_ub}
    Given $g \colon \Set{0, 1}^n \to \bbR$,
    we say $\hat{g} \colon \Set{0, 1}^n \to \bbR$ is a \textit{tight upper bound} (TUB) of $g$, iff (i.e., if and only if)
    $\hat{g}(X) \geq g(X), \forall X$ with 
    $\min_X \hat{g}(X) = \min_X {g}(X)$ and
    $\arg \min_X \hat{g}(X) = \arg \min_X {g}(X)$, where $\arg \min_X g(X) = \Set{X^* \colon g(X^*) = \min_X {g}(X)}$.
\end{definition}
\begin{remark}\label{rem:tight_ub}
    It is easy to see that $\hat{g} = g$ is always a TUB of $g$.
    When $g = \mathbb{1}[X \notin \calC]$ is an indicator function for the violation of constraints $X \in \calC$, the condition in Def.~\ref{def:tight_ub} is equivalent to 
    $\hat{g}(X) \geq 1, \forall X \notin \calC$ and
    $\hat{g}(X) = 0, \forall X \in \calC$.   
\end{remark}

To conclude, we propose the following {concretized target} to construct \textit{the expectation of a tight upper bound}.
\begin{idea}[Construct the expectation of a TUB]\label{idea:exp_pr_good_obj}
    Given $f: \Set{0, 1}^n \to \bbR$ with constraints $X \in \calC$,
    let $g(X) = \mathbb{1}(X \notin \calC)$,
    we aim to find $\hat{f}_1, \hat{f}_2: \Set{0, 1}^n \to \bbR$ such that
    $\hat{f}_1$ is a TUB of $f$ and
    $\hat{f}_2$ is a TUB of $g$,
    and to construct a probabilistic objective
    $\tilde{f}(p) \coloneqq \bbE_{X \sim p}\hat{f}_{1}(X) + \beta \bbE_{X \sim p}\hat{f}_2(X)$ with $\beta > 0$.
\end{idea}

\subsection{Fast and effective derandomization: Do it in a greedy and incremental manner}\label{subsec:idea:derand}

\begin{tcolorbox}[boxsep=0pt]
\small
\begin{tgt}
[Fast and effective derandomization]\label{target:good_derand}
We aim to propose a derandomization scheme that is \textit{fast} in speed and \textit{effective} in generating high-quality solutions.
\end{tgt}
\end{tcolorbox}

\smallsection{Greedy.} 
To this end, we generalize \textit{greedy} algorithms to {greedy derandomization} and propose an \textit{incremental} scheme to improve the speed.
For greedy derandomization, 
starting from $p_{\mathrm{cur}} = p_{\mathrm{o}}$,
we repeat the following steps:

(1) we greedily find the best local derandomization, i.e.,
\begin{center}
    \vspace{-1.5mm}
    $(i^*, x^*) \gets \argmin_{(i, x) \in [n] \times \Set{0, 1}} \Tilde{f}(\operatorname{der}(i, x; p_{\mathrm{cur}}))$;
    \vspace{-1.5mm}
\end{center}

(2) we conduct the best derandomization, i.e., 
\begin{center}
    \vspace{-1.5mm}
    $p_{\mathrm{cur}} \gets \operatorname{der}(i^*, x^*; p_{\mathrm{cur}})$.
\end{center}

\begin{theorem}[Goodness of greedy derandomization]\label{thm:greedy_like_der_good}
    For any entry-wise concave $\Tilde{f}$ and any $p_{\mathrm{o}} \in [0, 1]^n$, the above process of greedy derandomization can always reach a point where the final $p_{\mathrm{final}}$ is 
    \textbf{(G1)} discrete (i.e., $p_{\mathrm{final}} \in \Set{0, 1}^n$),
    \textbf{(G2)} no worse than $p_{\mathrm{o}}$ (i.e., $\Tilde{f}(p_{\mathrm{final}}) \leq \Tilde{f}(p_{\mathrm{o}})$),
    and \textbf{(G3)} a local minimum (i.e., $\Tilde{f}(p_{\mathrm{final}}) \leq \min_{(i, x) \in [n] \times \Set{0, 1}} \Tilde{f}(\operatorname{der}(i, x; p_{\mathrm{final}}))$).
\end{theorem}
\begin{remark}
    Our two theorems are synergic.
    Specifically,     Theorem~\ref{thm:concave_exp_prob} guarantees an entry-wise concave probabilistic objective $\tilde{f}$, which is used as a condition in Theorem~\ref{thm:greedy_like_der_good}.
\end{remark}
Greedy derandomization improves upon the existing derandomization methods.
Specifically,
random sampling~\citep{karalias2020erdos} guarantees (G1),
and iterative rounding~\citep{wang2022unsupervised} guarantees (G1) and (G2).
However, challenges arise regarding the time complexity since a naive way requires $2n$ evaluations of $\tilde{f}$ at each step.

\smallsection{Incremental.}
To this end, we propose to conduct the derandomization in an \textit{incremental} manner to increase the speed, which gives the following target.
Our intuition is that, usually, the incremental differences are simpler than the whole function, and the computation of incremental differences is easily parallelizable.
\begin{idea}[Conduct incremental greedy derandomization]\label{idea:greedy_fast}
We conduct \textit{greedy derandomization} and 
improve the speed by deriving the \textit{incremental differences} (IDs)
$\Delta \Tilde{f}(i, x, p_{\mathrm{cur}}) \coloneqq \Tilde{f}(\operatorname{der}(i, x; p_{\mathrm{cur}})) - \Tilde{f}(p_{cur})$ for all the $(i, x)$ pairs,
instead of evaluating the ``whole'' function, i.e., $\Tilde{f}(\operatorname{der}(i, x; p_{\mathrm{cur}}))$'s.
\end{idea}

\section{Deriving Formulae to Meet the Targets}\label{sec:analy_conds}
The targets in Sec.~\ref{sec:proposed_method} provide us guidelines, while deriving objectives and derandomization to meet those targets is nontrivial.
{In this work, we focus on \ours (\bus{U}supervised \bus{Com}binatorial Optimization \bus{U}nder \revise{\bus{Com}monly-involved} Conditions) and baptize our method with the same name.}
{For various conditions that are commonly involved in different CO problems (see Sec.~\ref{sec:problems} and App.~\ref{app:more_problems}), we shall derive 
(1) TUB-based probabilistic objectives $\Tilde{f}$ to meet Target~\ref{idea:exp_pr_good_obj} and
(2) incremental differences (IDs) of $\Tilde{f}$ to meet Target~\ref{idea:greedy_fast}.}
{Some conditions were encountered in the existing works but were not properly handled within the probabilistic UL4CO pipeline. 
See more discussions in App.~\ref{subapp:conditions_wrong}.}

We tackle each condition using the template below. 
Note that deriving TUB and IDs for each condition requires distinct, non-trivial ideas.
\begin{tcolorbox}[boxsep=0pt]
\small
\textbf{Construct a probabilistic objective to meet Target~\ref{idea:exp_pr_good_obj}:}
\vspace{-0.7em}
\begin{itemize}[leftmargin=*]
\setlength{\itemsep}{-0.5em}
    \item (\textbf{S1-1}) We find a TUB $\hat{f}$ for the condition
    \item[\text{i.e.}] Given an optimization objective $f$, we find $\hat{f}$ s.t. $\hat{f}(X) \geq f(X), \forall X$ with $\min_X \hat{f}(X) = \min_X f(X)$
    and
    $\arg\min_X \hat{f}(X) = \arg\min_X f(X)$
    \item[\textsc{or}] Given a constraint $X \in \calC$, we find $\hat{f}$ s.t. $\hat{f}(X) \geq \mathbb{1}(X \notin \calC), \forall X$ with $\hat{f}(X) = 0, \forall X \in \calC$
    \item (\textbf{S1-2}) After finding $\hat{f}$, we derive $\Tilde{f}(p) \coloneqq \bbE_{X \sim p} \hat{f}(X)$
\end{itemize}
\vspace{-0.5em}
    \textbf{Derive derandomization to meet Target~\ref{idea:greedy_fast}:}
    \vspace{-0.7em}
    \begin{itemize}[leftmargin=*]
\setlength{\itemsep}{-0.5em}
    \item (\textbf{S2}) We derive the formula of IDs $\Delta \Tilde{f}(\der(i, x; p))$
    \end{itemize}    
\end{tcolorbox}

\revise{
The conditions to be tackled below have both theoretical and empirical values.
Specifically, they are mathematically hard to handle for probabilistic UL4CO, and are commonly involved in many CO problems.}

\subsection{Cardinality constraints}\label{subsec:analy_conds:card}
\smallsection{Definition.} 
We consider constraints $X \in \calC$ with $\calC = \Set{X: \Abs{V_X} \in C_{c}}$.
Some typical cases are
$C_{c} = \Set{k}$ or
$C_{c} = \Set{t \in \bbN \colon t \leq k}$
for some $k \in \bbN$ ~\citep{buchbinder2014submodular}.

Given $p \in [0, 1]^n$, 
$\Abs{V_X} = \sum_{i \in V} X_i$ (see Sec.~\ref{sec:prelim}) follows a Poisson binomial distribution $\PB(p_1, p_2, \ldots, p_n)$ with parameters $(p_i)_{i \in [n]}$~\citep{wang1993number}.
The probability mass function (PMF) is for each $0 \leq t \leq n$,
\begin{center}
    $\Pr_{X \sim p}[\Abs{V_X} = t] = \sum_{V_t \subseteq V: \Abs{V_t} = t} \prod_{i \in V_t} p_i \prod_{j \in V \setminus V_t} (1-p_j)$.   
\end{center}

\smallsection{(S1-1).}
We find 
$\hat{f}_{\mathrm{card}}(X; C_c) \coloneqq \min_{k \in C_c} \Abs{\Abs{V_X} - k}$,
\revise{i.e., the minimum distance to the feasible cardinality set $C_c$.}
\begin{lemma}\label{lem:card_tight_ub}
    $\hat{f}_{\mathrm{card}}$ is a TUB of $\mathbb{1}[X \notin \calC]$.
\end{lemma}
\begin{remark}
    {We can directly compute $\Pr_{X \sim p}[\Abs{V_X} \notin C_c]$, but the formula we use practically performs better, which distinguishes different levels of violations. See similar ideas by, e.g., \cite{poganvcic2019differentiation}.}
\end{remark}

\smallsection{(S1-2).}
We derive $\Tilde{f}_{\mathrm{card}}(p; C_c) \coloneqq \bbE_{X \sim p} \hat{f}_{\mathrm{card}}(X; C_c)
= \sum_{t \in [n] \setminus C_c} \Pr_{X \sim p}[\Abs{V_X} = t] \min_{k \in C_c} \Abs{t - k}$.
The main technical difficulty is computing the PMF of a Poisson binomial distribution, for which
we adopt a discrete-Fourier-transform-based method.
The main formula of $\Pr_{X \sim p}[\Abs{V_X} = t]$ (See Eq. (6) by \citet{hong2013computing}) is
\begin{center}       
$\frac{1}{n+1} \sum_{s = 0}^n \exp(-\mathbf{i}\omega s t) \prod_{j = 1}^n (1 - p_j + p_j \exp(\mathbf{i}\omega s))$,
\end{center}
where $\mathbf{i} = \sqrt{-1}$ and $\omega =  \frac{2\pi}{n + 1}$.
See App.~\ref{subapp:technical_details:fourier_PB} for more details.

\smallsection{(S2).} 
We derive the IDs of $\tilde{f}_{\mathrm{card}}$, \revise{using the recursive formula of the Poisson binomial distribution.}
\begin{lemma}[IDs of $\Tilde{f}_{\mathrm{card}}$]\label{lem:card_incre}
    For any $p \in [0, 1]^n$, $i \in [n]$, and $0 \leq t \leq n$,
    let $q_s \coloneqq \Pr_{X \sim p}[\Abs{V_X} = s]$
    and $q'_s \coloneqq \Pr_{X \sim p}[\Abs{V_X \setminus \Set{i}} = s], \forall s$,    we have
    \setlength{\jot}{0pt}    
    \begin{align}    
    q'_t &= (1 - p_i)^{-1} \sum\nolimits_{s = 0}^t q_s \left(\frac{p_i}{p_i-1}\right)^{t-s} \text{~(if~ $p_i \neq 1$)} \label{eq:card_pi_neq_1} \\
    &= (p_i)^{-1} \sum\nolimits_{s = 0}^{n - t - 1} q_{t+s+1} \left(\frac{p_i-1}{p_i}\right)^{s} \text{~(if~ $p_i \neq 0$)}. \label{eq:card_pi_neq_0}
    \end{align}    
    Based on that, we have
    \begin{center}        
    $
    \begin{cases} 
    \Delta \Tilde{f}_{\mathrm{card}}(i, 0, p; C_c) = 
    \sum_{t \in [n] \setminus C_c} (q'_t - q_t) \min_{k \in C_c} \Abs{t - k}, \\
    \Delta \Tilde{f}_{\mathrm{card}}(i, 1, p; C_c) = 
    \sum_{t \in [n] \setminus C_c} (q'_{t-1} - q_t) \min_{k \in C_c} \Abs{t - k}.
    \end{cases}
    $
    \end{center}
\end{lemma}

\begin{remark}\label{rem:num_stab}
    {In practice, we always make sure each $p_i \in [\epsilon, 1 - \epsilon]$ for some small $\epsilon > 0$ for better numerical stability.}
    We use Eq.~\eqref{eq:card_pi_neq_1} for $p_i \leq 0.5$ and
    Eq.~\eqref{eq:card_pi_neq_0} for $p_i > 0.5$.
\end{remark}

\smallsection{Notes.}
{\revise{Enforcing cardinality constraints (e.g., taking top-$k$) is easy, but \textit{differentiable} training with cardinality constraints is nontrivial. Other than probabilistic-method UL4CO, Sinkhorn-related techniques~\citep{sinkhorn1967concerning, wang2022towards} are valid ways. See also App.~\ref{subapp:conditions_wrong}.}}

\subsection{Minimum (or maximum) w.r.t. a subset}\label{subsec:analy_conds:opt_wrt_subset}
\smallsection{Definition.}
We consider constraints where we have a pairwise score function (e.g., distance) $h: V \times V \to \bbR$
and we aim to compute $f_{\mathrm{ms}}(X) \coloneqq \min_{v_X \in V_X} h(i, v_X)$ for some $i \in V$ (e.g., the shortest distance to a set of points).

We fix $i \in V$ in the analysis below, and
let $v_1, v_2, \ldots, v_n$ be a permutation of $V = [n]$ such that    
$d_1 \leq d_2 \leq \cdots \leq d_n$,
where $d_j = h(i, v_j), \forall j \in [n]$.

\smallsection{(S1-1).}
We find $\hat{f}_{\mathrm{ms}}(X; i, h) \coloneqq \min_{v_X \in V_X} h(i, v_X)$, \revise{which is the original objective $f_{\mathrm{ms}}$.}
\begin{lemma}\label{lem:optim_tight_ub}
    $\hat{f}_{\mathrm{ms}}$ is a TUB of ${f}_{\mathrm{ms}}$.
\end{lemma}

\smallsection{(S1-2).} 
We derive $\Tilde{f}_{\mathrm{ms}}(p; i, h) \coloneqq \bbE_{X \sim p} \hat{f}_{\mathrm{ms}}(X; i, h)$, \revise{by decomposing the objective into sub-terms.}
\begin{lemma}\label{lem:optim_exp}
    For any $p \in [0, 1]^n$, 
    $\bbE_{X \sim p} \hat{f}_{\mathrm{ms}}(X; i, h) = p_{v_1} d_1 + (1 - p_{v_1})p_{v_2} d_2 + \cdots + (\prod_{j = 1}^{n-1} (1-p_{v_j}))p_{v_{n}} d_n$.
\end{lemma}

\smallsection{(S2).}
We derive the IDs of $\Tilde{f}_{\mathrm{ms}}$, \revise{by analyzing which sub-terms are changed after one step of local derandomization.}

\begin{lemma}[IDs of $\Tilde{f}_{\mathrm{ms}}$]\label{lem:opt_incre}
    For any $p \in [0, 1]^n$ 
    and $j \in [n]$,     
    let $q_j \coloneqq (\prod_{k = 1}^{j-1} (1-p_{v_{k}}))p_{v_{j}}$, the coefficient of $d_j$ in $\tilde{f}_{\mathrm{ms}}$. Then
    \begin{center}        
    $
    \begin{cases} 
    \Delta \Tilde{f}_{\mathrm{ms}}(v_j, 0, p; i, h) = - q_j d_j 
     + \frac{p_{v_j}}{1-p_{v_j}}\sum_{j' > j} q_{j'}d_{j'},\\
     \Delta \Tilde{f}_{\mathrm{ms}}(v_j, 1, p; i, h) = 
    \sum_{j' > j} q_{j'}(d_j - d_{j'}).
    \end{cases}
    $
    \end{center}
\end{lemma}
\begin{remark}    
When $p_{v_j} = 1$, we replace
$\frac{p_{v_j}}{1-p_{v_j}}\sum_{j' > j} q_{j'}d_{j'}$ by $\sum_{j' > j} (\prod_{1 \leq i' \leq i - 1, i' \neq j} (1-p_{v_{i'}}))p_{v_{i}} d_{j'}$.
Since we make sure each $p_i \neq 1$ (see Rem.~\ref{rem:num_stab}), this does not happen in practice.
\end{remark}

\subsection{Covering}\label{subsec:analy_conds:cover}
\smallsection{Definition.}
We consider conditions where
some $i \in V$ needs to be \textit{covered} (i.e., at least one neighbor of $i$ is chosen).
Formally, the constraints are $X \in \calC$ with
$\calC = \Set{X \colon \Set{v_X \in V_X \colon (v_X, i) \in E} \neq \emptyset}$.

\smallsection{(S1-1).}
We find $\hat{f}_{\mathrm{cv}}(X; i) \coloneqq \mathbb{1}(X \notin \calC)$, \revise{which is the indicative function of the original constraint.}
\begin{lemma}\label{lem:optim_tight_cover}
    $\hat{f}_{\mathrm{cv}}$ is a TUB of $\mathbb{1}(X \notin \calC)$.
\end{lemma}

\smallsection{(S1-2).}
We drive $\Tilde{f}_{\mathrm{cv}}(p; i) \coloneqq \bbE_{X \sim p} \hat{f}_{\mathrm{cv}}(X; i) = \Pr_{X \sim p}[\Set{v_X \in V_X \colon (v_X, i) \in E} \neq \emptyset]$, \revise{by decomposing the objective into sub-terms.}

\begin{lemma}\label{lem:cover_obj}
    For any $p \in [0, 1]^n$ and $i \in [n]$,    
    $\Pr_{X \sim p}[\Set{v_X \in V_X \colon (v_X, i) \in E} \neq \emptyset] = \prod_{v \in [n] \colon (v, i) \in E} (1 - p_v)$.
\end{lemma}

\smallsection{(S2).}
We derive the IDs of $\Tilde{f}_{\mathrm{cv}}$, \revise{by analyzing which sub-terms are changed after one step of local derandomization.}
\begin{lemma}[IDs of $\Tilde{f}_{\mathrm{cv}}$]\label{lem:cover_incre}
    For any $p \in [0, 1]^n$ and $i \in [n]$,
    if $(i, j) \notin E$,
    then $\Delta \Tilde{f}_{\mathrm{cv}}(j, 0, p; i) = \Delta \Tilde{f}_{\mathrm{cv}}(j, 1, p; i) = 0$;    
    if $(i, j) \in E$,
    then
    \begin{center}        
    $
    \begin{cases}
    \Delta \Tilde{f}_{\mathrm{cv}}(j, 0, p; i) = p_j \prod_{v \in N_i, v \neq j}(p_v - 1),\\
    \Delta \Tilde{f}_{\mathrm{cv}}(j, 1, p; i) = -\Tilde{f}_{\mathrm{cv}}(p; i).
    \end{cases}
    $    
    \end{center}
\end{lemma}

\subsection{Cliques (or independent sets)}\label{subsec:analy_conds:clique}

\smallsection{Definition.}
We consider conditions where the chosen nodes $V_X$ should form a clique.\footnote{Equivalently, an independent set in the complement graph.}
Formally, the constraints are
$X \in \calC$ with
$\calC = \Set{X \colon \binom{V_X}{2} \subseteq E}$.

\smallsection{(S1-1).}
We find $\hat{f}_{\mathrm{cq}}(X) \coloneqq \Abs{\Set{(u, v) \in \binom{V_X}{2} \colon (u, v) \notin E}}$, \revise{the number of chosen node pairs violating the constraints.}
\begin{lemma}\label{lem:clique_tight_ub}
    $\hat{f}_{\mathrm{cq}}$ is a TUB of $\mathbb{1}[X \notin \calC]$.
\end{lemma}

\smallsection{(S1-2).}
We derive $\tilde{f}_{\mathrm{cq}}(p) \coloneqq \bbE_{X \sim p} \hat{f}_{\mathrm{cq}}(X)$, \revise{by decomposing the objective into sub-terms.}
\begin{lemma}\label{lem:cliq_obj}
    For any $p \in [0, 1]^n$, 
    $\bbE_{X \sim p} \hat{f}_{\mathrm{cq}}(X)
    = \sum_{(u, v) \in \binom{V}{2} \setminus E} p_u p_v$.
\end{lemma}

\smallsection{(S2).}
We derive the IDs of $\Tilde{f}_{\mathrm{cq}}$, \revise{by analyzing which sub-terms are changed after one step of local derandomization.}

\begin{lemma}[IDs of $\Tilde{f}_{\mathrm{cq}}$]
\label{lem:cliq_incre}
    For any $p \in [0, 1]^n$ and $i \in [n]$,
    \[
    \begin{cases}
     \Delta \tilde{f}_{\mathrm{cq}}(i, 0, p) =
    - p_i \sum_{j \in [n], j \neq i, (i, j) \notin E} p_j, \\
    \Delta \tilde{f}_{\mathrm{cq}}(i, 1, p) =
    (1 - p_i) \sum_{j \in [n], j \neq i, (i, j) \notin E} p_j.
    \end{cases}
    \]
\end{lemma}

\revise{
\smallsection{Notes.}
\cite{karalias2020erdos} and \cite{min2022can} essentially considered the ``cliques'' conditions and derived similar formula of $\tilde{f}_{\mathrm{cq}}(p)$.
Our derivation of the IDs is novel, and 
we will also extend this to non-binary cases, which were not discussed in existing works.
Moreover, our high-level targets and templates provide insights into obtaining and interpreting the derivation in a principled way.
}

\subsection{Non-binary decisions}\label{subsec:analy_conds:coloring}

\smallsection{Definition.}
We consider \textit{non-binary decisions}, i.e., (potentially) more than two decisions ($\Abs{d} \geq 2$), e.g., problems with partition or coloring.
\revise{Our theoretical analysis can be extended to non-binary cases. See App.~\ref{subapp:theory:non_binary} for more details.}

\subsection{Uncertainty}\label{subsec:analy_conds:uncertainty}
We {also} consider \textit{uncertainty} in edge existence, i.e., edge probabilities
$P: E \to [0, 1]$.
Due to the generality of non-binary conditions and uncertainty, the details of objective construction and derandomization will be deferred to where each specific problem is analyzed in Sec.~\ref{sec:problems}.

\subsection{Notes and insights}\label{subsec:derive_notes}
Throughout the section, two commonly used ideas for constructing TUBs are:
(1) using a function itself (Lemmas~\ref{lem:optim_tight_ub} \& \ref{lem:optim_tight_cover}), and
(2) relaxing the binary ``a constraint is violated'' to ``the number of violations'' (Lemmas~\ref{lem:card_tight_ub} \& \ref{lem:clique_tight_ub}).

\revise{Moreover, techniques commonly used in our derivations include (1) decomposing objectives into sub-terms, and (2) analyzing which sub-terms are changed after one step of local derandomization.}
\revise{Such ``local decomposibility'' allows us to use linearity of expectation. See, e.g., similar ideas by~\citet{ahn2020learning} and \citet{Jo2023robust}. See App.~\ref{subapp:discussions:local_decomp} for more discussions.}

\revise{We acknowledge that we have not covered all conditions involved in CO, but we expect that similar ideas would be applicable to some other conditions.
See App.~\ref{subapp:cycles_trees} for discussions on some conditions not fully covered in this work, e.g., cycles and trees.}

%% file: 050probl.tex
\section{Applying the Derivations to CO Problems}\label{sec:problems}
In this section, we apply \ours to different CO problems (facility location, maximum coverage, and robust coloring) with both theoretical values, NP-hardness~\citep{mihelic2004facility,yanez2003robust}, and real-world implications.
{See App.~\ref{app:more_problems} for the applications to four more problems (robust $k$-clique, robust dominating set, clique cover, and minimum spanning tree).}
{Specifically, for each specific problem, we shall 
(1) check what conditions are involved and
(2) construct the probabilistic objective and derandomization process by combining the analyses in Sec.~\ref{sec:analy_conds}.}
\begin{tcolorbox}[boxsep=0pt]
\small
\begin{itemize}[leftmargin=*]
    \item (\textbf{1}) Find the conditions involved in the optimization objective ($f = \sum_i f_i$) and the constraints ($X \in \bigcap_i \calC_i$).
    \vspace{-0.7em}    
    \item (\textbf{2}) Construct the final objective: 
    $\sum_i \Tilde{f}_i + \beta \sum_{j} \Tilde{g}_j$ with constraint coefficient $\beta > 0$ by combining the probabilistic functions 
    $\Tilde{f}_i$'s and $\Tilde{g}_i$'s
    for the optimization objectives and constraints, respectively.
\end{itemize}
\end{tcolorbox}

\subsection{Facility location}\label{subsec:problems:facility_location}
The \textit{facility location} problem is abstracted from real-world scenarios, where the goal is to find some good locations among candidate locations~\citep{drezner2004facility}.

\smallsection{Definition.}
Given (1) a complete weighted graph $G = (V = [n], E = \binom{V}{2}, W)$,
where the distance between each pair $(u, v)$ of nodes is $W(u, v) > 0$,
and $W(v, v) = 0, \forall v \in V$,
and (2) the number $k$ of locations to choose,
we aim to find a subset $V_X \subseteq V$ such that
(c1) $\Abs{V_X} = k$,
and 
(c2) $\sum_{v \in V} \min_{v_X \in V_X} W(v, v_X)$ is minimized.

\setlength{\parskip}{.3 pc}
\smallsection{Involved conditions:}
(1) cardinality constraints and (2) minimum w.r.t. a subset (see Secs. \ref{subsec:analy_conds:card} \& \ref{subsec:analy_conds:opt_wrt_subset}).

\smallsection{Details.}
Given $p \in [0, 1]^n$ and $\beta > 0$, 
\begin{center}
\vspace{-1.5mm}
$\tilde{f}_{\mathrm{FL}}(p; G, k) = (\sum\nolimits_{v \in V} \Tilde{f}_{\mathrm{ms}}(p; v, W)) + \beta \Tilde{f}_{\mathrm{card}}(p; \Set{k})$.
\vspace{-1.5mm}
\end{center}
For $i \in [n]$ and $x \in \Set{0, 1}$, the ID is
$\Delta \tilde{f}_{\mathrm{FL}}(i, x, p; G, k) = \sum_{v \in V} \Delta \Tilde{f}_{\mathrm{ms}}(i, x, p; v, W) + \beta \Delta \Tilde{f}_{\mathrm{card}}(i, x, p; \Set{k})$.

\subsection{Maximum coverage}\label{subsec:problems:max_cover}
The \textit{maximum coverage} problem~\citep{khuller1999budgeted} is a classical CO problem with real-world applications, including public traffic management~\citep{ali2017coverage}, web management~\citep{saha2009maximum}, and scheduling~\citep{marchiori2000evolutionary}.

\smallsection{Definition.}
Given 
(1) $m$ items (WLOG, assume the items are $[m]$), each with weight $W_j, \forall j \in [m]$,
(2) a family of $n$ sets $\calS = \Set{S_1, S_2, \ldots, S_n}$ with each $S_i \subseteq [m]$ and 
(3) the number $k$ of sets to choose,
we aim to choose $\calS_X \subseteq \calS$ from the given sets such that
(c1) $\Abs{\calS_X} = k$
and
(c2) the total weights of the covered items
$\sum_{j \in T_X} W_j$ is maximized,
where $T_X \coloneqq \bigcup_{S_i \in \calS_X} S_i$ is the set of covered items.

\smallsection{Involved conditions:}
(1) cardinality constraints and (2) covering 
(see Secs. \ref{subsec:analy_conds:card} \& \ref{subsec:analy_conds:cover}).

\smallsection{Details.}
Construct a bipartite graph $G_{\calS} = (V = \calS \cup [m], E)$, where $(S_i, j) \in E$ iff $j \in S_i$.
For $p \in [0, 1]^n$ and $\beta > 0$,
\begin{center}
\vspace{-1mm}
$\tilde{f}_{\mathrm{MC}}(p; \calS, k) = 
\sum_{j \in [m]} W_j \Tilde{f}_{\mathrm{cv}}(p; j, G_{\calS}) + \beta \Tilde{f}_{\mathrm{card}}(p; \Set{k})$.
\vspace{-1mm}
\end{center}
For $i \in [n]$ and $x \in \Set{0, 1}$, the ID is
$\Delta \tilde{f}_{\mathrm{MC}}(i, x, p; \calS, k) = \sum_{j \in [m]} W_j \Delta \Tilde{f}_{\mathrm{cv}}(i, x, p; j, G_{\calS}) +
\beta 
\Delta \Tilde{f}_{\mathrm{card}}(i, x, p; \Set{k})$.

\subsection{Robust coloring}\label{subsec:problems:robust_coloring}
The \textit{robust coloring} problem~\citep{yanez2003robust} generalizes the coloring problem~\citep{jensen2011graph}. It is motivated by real-world scheduling problems where some conflicts can be uncertain, with notable applications to supply chain management~\citep{lim2005robust}.

\input{FIG/fig_tradeoff_rand}

\smallsection{Definition.}
Given 
(1) an uncertain graph $G = (V, E, P)$ and
(2) the number $c$ of colors,
let $E_h \coloneqq \Set{e \in E \colon P(e) = 1}$ represent \textit{hard conflicts} which we \textit{must} avoid, and let $E_s \coloneqq \Set{e \in E \colon P(e) < 1}$ represent \textit{soft conflicts} which possibly happen,
we aim to find a $c$-coloring $X$ on $V$, where each node $v \in V$ has a color $X_v \in d \coloneqq \Set{0, 1, \ldots, c - 1}$,
such that 
\textbf{(c1)} no hard conflicts are violated
(i.e., $X_u \neq X_v, \forall (u, v) \in E_h$), and
\textbf{(c2)} the probability that no violated soft conflicts happen (i.e., $\prod_{e = (u, v) \in E_s \colon X_u = X_v} (1 - P(e))$) is maximized.

We fix $G$ and $c$ in the analysis below.

\smallsection{Involved conditions:}
(1) independent sets,\footnote{The nodes in each color group should be an independent set.} (2) uncertainty, and (3) non-binary decisions
(see Secs.~\ref{subsec:analy_conds:clique} to \ref{subsec:analy_conds:uncertainty}).

\smallsection{Details.}
Regarding (c1), we extend the derivations in Sec.~\ref{subsec:analy_conds:clique} to non-binary cases.
We use
\begin{center}
\vspace{-1mm}
    $\hat{g}_{1}(X) \coloneqq \Abs{\Set{(u, v) \in E_h \colon X_u = X_v}}$,
    \vspace{-1mm}
\end{center}
which is a TUB of $g_1(X) \coloneqq \mathbb{1}(X \notin \calC_1)$ 
with $\calC_1 = \Set{X \colon \text{(c1) is satisfied}}$,
and use $\Tilde{g}_{1}(p) \coloneqq \bbE_{X \sim p} \hat{g}_{1}(X)$.
Regarding (c2), maximizing $\prod_{e = (u, v) \in E_s \colon X_u = X_v} (1 - P(e))$ is equivalent to minimizing 
\begin{center}
\vspace{-1mm}
$f_{2}(X) 
= -\sum_{e = (u, v) \in E_s \colon X_u = X_v} \log (1 - P(e))$.    
\vspace{-1mm}
\end{center}
With $\hat{f}_{2}(X) \coloneqq f_{2}(X)$ and
$\Tilde{f}_{2}(p) \coloneqq \bbE_{X \sim p} \hat{f}_{2}(X)$,
the final objective is $\Tilde{f}_{\mathrm{RC}} = \Tilde{f}_{2} + \beta \Tilde{g}_{1}$ with constraint coefficient $\beta > 0$.
\begin{lemma}\label{lem:rc_tub}
    $\hat{g}_1$ is a TUB of $g_1$ and 
    $\hat{f}_2$ is a TUB of $f_2$.
\end{lemma}
\begin{proof}[Proof sketch] 
    We extend the ideas for independent sets in Sec.~\ref{subsec:analy_conds:clique}, especially Lemma~\ref{lem:clique_tight_ub}.
\end{proof}
We derive each term in $\Tilde{f}_{\mathrm{RC}}$ as follows.
\begin{lemma}\label{lem:rc_obj}
    For any $p \in [0, 1]^{n \times c}$, $\tilde{f}_2(p) = \bbE_{X \sim p} \hat{f}_{2}(X)$ and
    $\tilde{g}_1(p) = \bbE_{X \sim p} \hat{g}_{1}(X)$ with
    \vspace{-1mm}
    \begin{center}
     $\bbE_{X \sim p} \hat{f}_{2}(X) = -\sum_{e = (u, v) \in E_s} \sum_{r = 0}^{c-1} p_{ur}p_{vr} \log (1 - P(e))$,
    \end{center}
    \begin{center}
        and $\bbE_{X \sim p} \hat{g}_{1}(X) = \sum_{(u, v) \in E_h} \sum_{r = 0}^{c-1} p_{ur} p_{vr}$.
    \end{center} 
    \vspace{-1mm}
\end{lemma}
\begin{proof}[Proof sketch]
    We extend the ideas in Lemma~\ref{lem:cliq_obj}.
\end{proof}

We then derive the IDs of each term in $\Tilde{f}_{\mathrm{RC}}$.
\begin{lemma}[IDs of the terms in $\Tilde{f}_{\mathrm{RC}}$]\label{lem:rc_incre}
    For any $p \in [0, 1]^{n \times c}$, $i \in [n]$, and $x \in d$,
    \vspace{-1mm}
    \begin{center}    
        $\Delta \Tilde{g}_{1}(i, x; p) = 
    \sum_{x' \in d \setminus \Set{x}} p_{ix'} \sum_{(i, j) \in E_h} (p_{jx} - p_{jx'})$, and      \resizebox{0.49\textwidth}{!}{
    $\Delta \Tilde{f}_{2}(i, x; p) = 
    \sum_{x' \in d \setminus \Set{x}} p_{ix'} \sum_{(i, j) \in E_s} (p_{jx'} - p_{jx}) \log(1 - P(i, j))$.
    }
    \end{center}
    \vspace{-1mm}
\end{lemma}
\begin{proof}[Proof sketch]
    We extend the ideas in 
    Lemma~\ref{lem:cliq_incre}.
    Changing $p_i$ only affects $j$ with $(i, j) \in E_{*}$, and we compute the differences for each $j$.
\end{proof}
We finally get the overall IDs
$\Delta \Tilde{f}_{\mathrm{RC}} = 
\Delta \Tilde{f}_{2} + 
\beta \Delta \Tilde{g}_{1}$.

\revise{
\smallsection{Notes.}
When practitioners encounter new problems that involve the conditions covered in this work, they can simply combine our derivations for the involved conditions, just as we did in this section.
Indeed, we believe many other CO problems involve the conditions considered in this work.}

%% file: FIG/fig_tradeoff_rand.tex
\begin{figure*}[t!]
    \centering
    \includegraphics[width=0.75\textwidth]{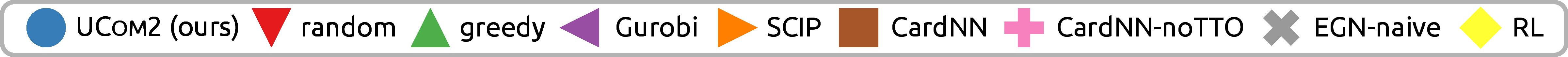}\\
    \vspace{-0.5mm}
    \begin{subfigure}[b]{0.48\linewidth}
        \includegraphics[scale=0.18]{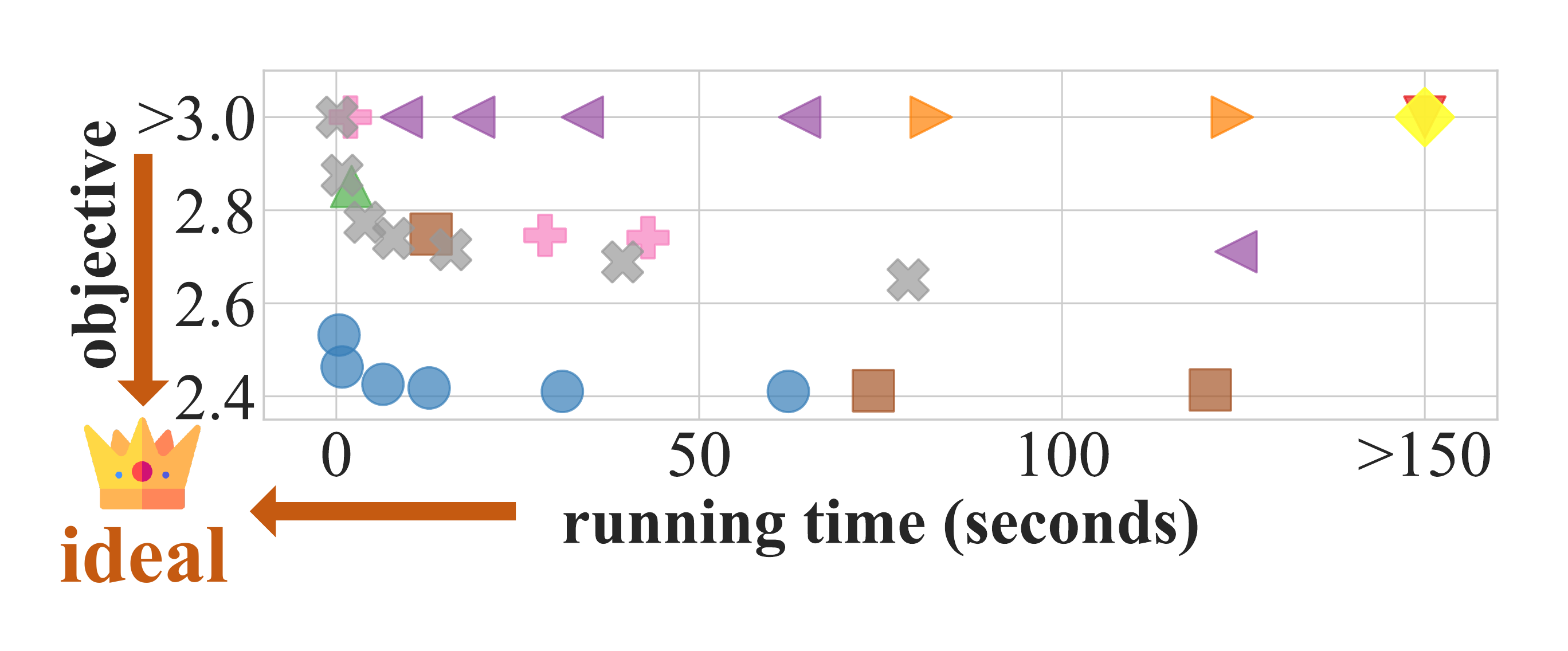}
        \vspace{-10mm}
        \caption{FL: rand500}
        \label{fig:fl_rand500_tradeoff}
    \end{subfigure}
    \begin{subfigure}[b]{0.48\linewidth}
        \includegraphics[scale=0.18]{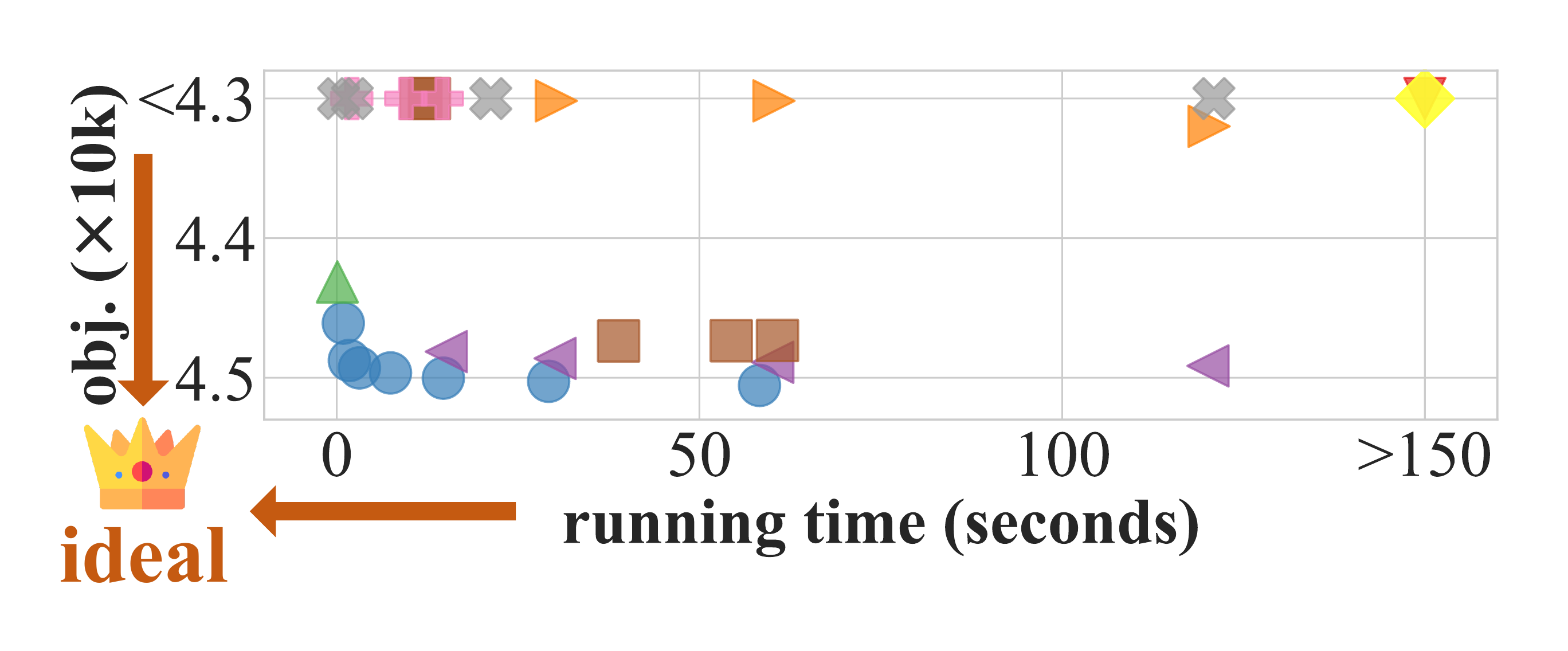}
        \vspace{-10mm}
        \caption{MC: rand500}
        \label{fig:mc_rand500_tradeoff}
    \end{subfigure}
    \begin{subfigure}[b]{0.48\linewidth}        
        \includegraphics[scale=0.18]{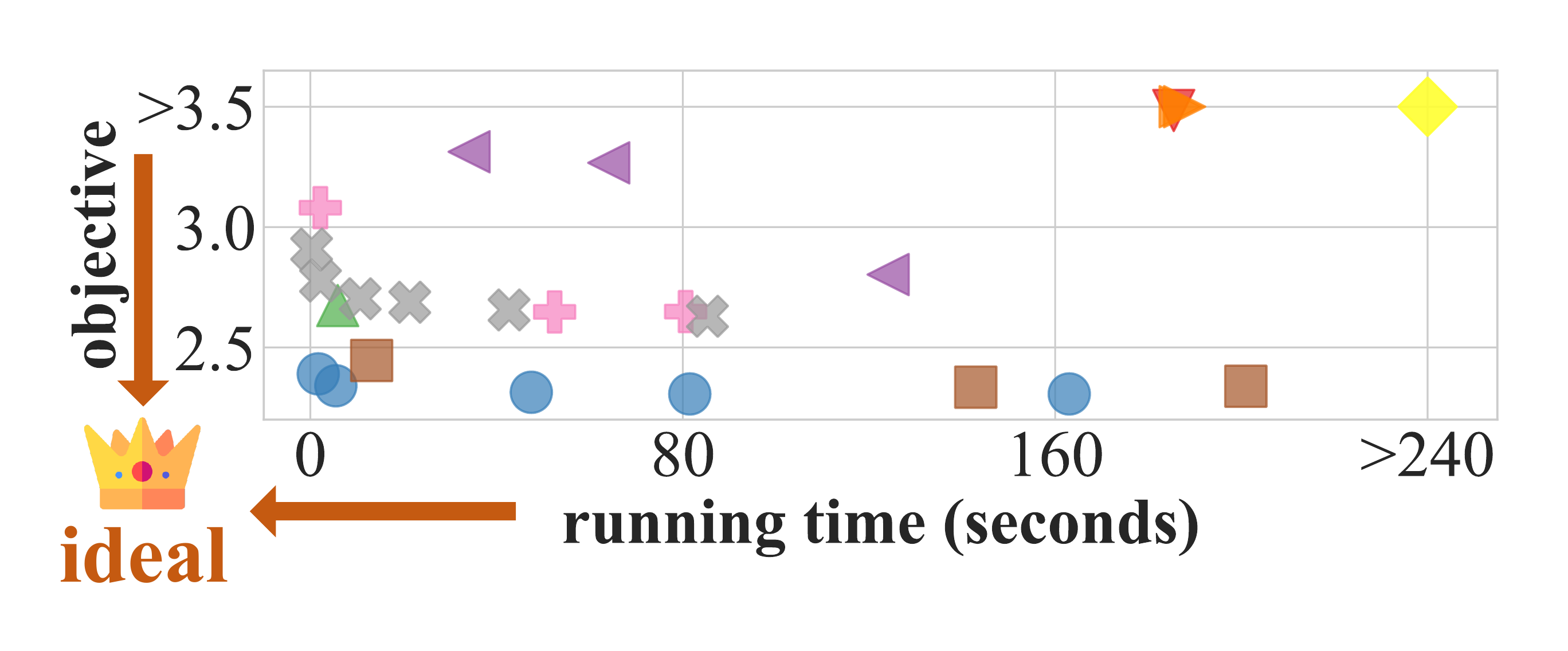}
        \vspace{-10mm}
        \caption{FL: rand800}
        \label{fig:fl_rand800_tradeoff}
    \end{subfigure}
    \begin{subfigure}[b]{0.48\linewidth}        
        \includegraphics[scale=0.18]{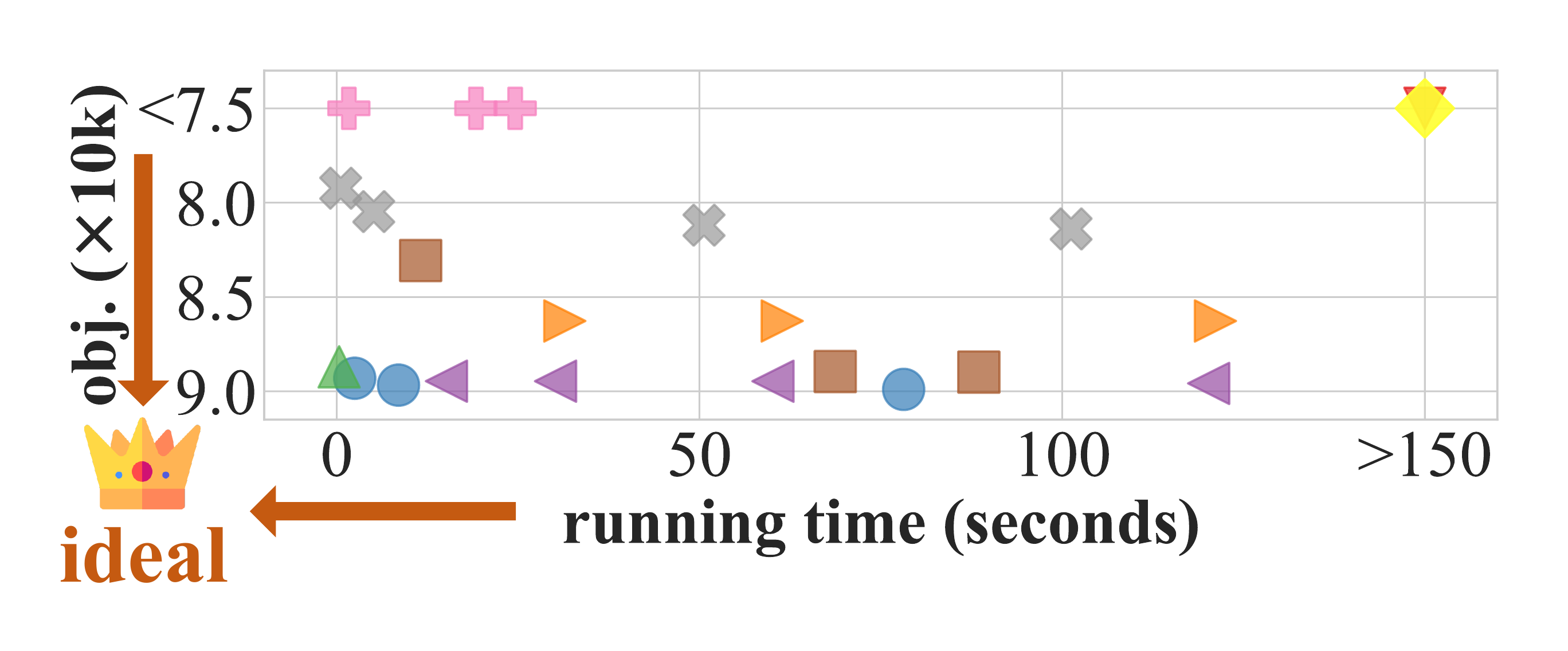}
        \vspace{-10mm}
        \caption{MC: rand1000}
        \label{fig:mc_rand1000_tradeoff}
    \end{subfigure}     
    \caption{The speed-quality trade-offs on facility location (FL) and maximum coverage (MC).
    Running time ($x$-axis): smaller the better.
    Objective ($y$-axis): for FL, the smaller the better; for MC, the larger the better.
    For MC, we reverse the $y$-axis so that the ideal point is always at the bottom left corner.
    }
    \label{fig:tradeoff_rand}
\end{figure*}

%% file: 060exper.tex
\section{Experiments}\label{sec:experiments}
Through \revise{extensive experiments on various problems}, we shall show the effectiveness of \ours. 

\subsection{Facility location and maximum coverage}
We conduct experiments on the facility location problem
and the maximum coverage problem
(Secs.~\ref{subsec:problems:facility_location} \& \ref{subsec:problems:max_cover}).
For the experimental settings,
we mainly follow an existing work~\citep{wang2022towards}, with additional datasets and baselines.
For fair comparisons, we consider inductive settings (training and test sets are different) and use the same GNN architectures as \citet{wang2022towards}.
{See App.~\ref{subapp:discussions:inductive_and_transductive} for discussions on transductive settings.}
{See App.~\ref{subapp:exp_settings} for the detailed experimental settings.}

\input{TAB/res_facility_location}

\input{TAB/res_max_cover}

\smallsection{Methods.}
We compare \ours with:
(1) \textbf{random}: $k$ locations or sets are picked uniformly at random;
(2) \textbf{greedy}: deterministic greedy algorithms;
(3-4) \textbf{Gurobi}~\citep{gurobi} and \textbf{SCIP}~\citep{bestuzheva2021scip,ortools}: the problems are formulated as MIPs and the two solvers are used;
(5) \textbf{{CardNN}}~\citep{wang2022towards}: a SOTA UL4CO method with three variants;
(6) \textbf{{CardNN}-noTTO}: 
{CardNN} directly optimizes on each test graph in test time, and these are variants of {CardNN} without test-time optimization;
(7) \textbf{EGN-naive}: EGN~\citep{karalias2020erdos} with a naive probabilistic objective construction and iterative rounding;
(8) \textbf{RL}: a reinforcement-learning method~\citep{kool2018attention}.
{See App.~\ref{subapp:discussions:rl_and_egn} for discussions on reinforcement learning.}

\input{TAB/res_robust_clr}

\smallsection{Datasets.}
We consider both synthetic and real-world graphs.
\begin{itemize}[leftmargin=*]
\setlength\itemsep{0em}
    \item \textbf{Random graphs}: 
    The number after ``rand'' represents the size of the random graphs. 
    Each group of random graphs contains 100 graphs from the same distribution.
    \item \textbf{Real-world graphs}: For facility location, each graph contains real-world entities with locations (\textit{starbucks}, \textit{mcd}, \textit{subway}). For maximum coverage, each graph contains real-world sets (\textit{twitch}, \textit{railway}). Each group of real-world graphs contains multiple graphs from the same source.
\end{itemize}

\smallsection{Speed-quality trade-offs.}
For several methods (including \ours),
we can grant more running time to obtain better optimization quality.
For \ours, we use test-time augmentation~\citep{jin2022empowering} on the test graphs by adding perturbations into graph topology and features.
\revise{\ours-short does not test-time augmentation, while the other two variants} use different numbers of augmented data.

\smallsection{{Evaluation.}}
For each group of datasets and each method, we report the average optimization objective and running time over five trials.
We also report the overall objective, time, and ranks, averaged over all the groups of datasets.
The average rank ``sum'' (ARS) is the summation of the average ranks w.r.t. objective and time.
See App.~\ref{app:additional_results} for the full results with standard deviations and ablation studies.

\smallsection{{Results.}}
On both problems, \ours achieves the best trade-offs overall (Tables~\ref{tab:results_facility_location} \& \ref{tab:results_max_cover}).
On facility location, the top-$3$ methods w.r.t. ARS are the three variants of \ours.
On maximum coverage, the three variants rank 1, 3, and 4 w.r.t. ARS, respectively.
In Figure~\ref{fig:tradeoff_rand}, we report the detailed trade-offs of different methods on the random graphs,
visually illustrating the best trade-off overall by \ours.

\subsection{Robust coloring}
{We conduct experiments on the robust coloring problem (see Sec.~\ref{subsec:problems:robust_coloring})
under transductive settings directly optimizing probabilistic decisions $p$.
See App.~\ref{subapp:exp_settings} for more details.}

\smallsection{Methods.}
We compare \ours with four baseline methods.
(1-2) \textbf{Greedy-RD} and \textbf{greedy-GA}: both methods decide the colors following an enumeration of nodes, where greedy-RD follows a random (RD) permutation of the nodes while greedy-GA uses a genetic algorithm (GA) to learn the permutation;\footnote{Greedy-GA is the method proposed by~\cite{yanez2003robust} in the original paper of robust coloring.}
(3) \textbf{Deterministic coloring (DC)}: 
a deterministic greedy coloring algorithm~\citep{kosowski2004classical} is used to avoid all the hard conflicts, and it tries to avoid as many soft conflicts as possible;
(4) \textbf{Gurobi}: the problem is formulated as an MIP and the solver is used.

\smallsection{Datasets.}
We use four real-world uncertain graphs:
(1) \textbf{collins},
(2) \textbf{gavin},
(3) \textbf{krogan}, and
(4) \textbf{PPI}.

\smallsection{Speed-quality trade-offs.}
We record the running time of \ours using only CPUs and using GPUs.
For \ours, we use multiple initial probabilities.
We make sure that even with only CPUs, \ours uses less time than each baseline.

\smallsection{Evaluation.}
For each group of datasets and each method,
we report the average optimization objective and running time over five trials.
The average ranks are computed in the same way as for facility location and maximum coverage.

\smallsection{Results.}
As shown in Table~\ref{tab:results_robust_coloring}, with the least running time, \ours consistently achieves (1) better optimization quality than the two \textbf{greedy} baselines and \textbf{DC}
and (2) better optimization quality than \textbf{Gurobi}
in most cases.
This superiority holds even when we only use CPUs for \ours.
When using GPUs, \ours is even faster.

\revise{
\subsection{Ablation studies}
We analyze different components in \ours and show that
(1) good probabilistic objectives are helpful,
(2) greedy derandomization is more effective than iterative rounding, and
(3) incremental derandomization improves the speed.
See App.~\ref{subapp:ablation_study} for more details.
}

%% file: TAB/res_facility_location.tex
\bgroup
\begin{table*}[t!]
\caption{Results on facility location.
Running time (time; normalized): the smaller the better.
Objective (obj; normalized): the smaller the better.
In each column, 
\textcolor{Red!50}{$\blacksquare$} indicates ranking the 1st, \textcolor{Blue!50}{$\blacksquare$} the 2nd, and
\textcolor{Green!50}{$\blacksquare$} the 3rd.} \label{tab:results_facility_location}
\small
\centering
\scalebox{0.78}{

\newcommand{\rkone}[1]{{\setlength{\fboxsep}{1.5pt}\colorbox{Red!50}{#1}}}
\newcommand{\rktwo}[1]{{\setlength{\fboxsep}{1.5pt}\colorbox{Blue!50}{#1}}}
\newcommand{\rkthree}[1]{{\setlength{\fboxsep}{1.5pt}\colorbox{Green!50}{#1}}}

\begin{tabular}{l||c|c||c|c||c|c||c|c||c|c||c|c||c|c|c}
\hline
\multirow{2}[3]{*}{Method} & \multicolumn{2}{c||}{rand500} & \multicolumn{2}{c||}{rand800} & \multicolumn{2}{c||}{starbucks} & \multicolumn{2}{c||}{mcd} & \multicolumn{2}{c||}{subway} & \multicolumn{2}{c||}{average} & \multicolumn{3}{c}{average rank} \bigstrut\\
\cline{2-16}      & obj$\downarrow$   & time$\downarrow$  & obj$\downarrow$   & time$\downarrow$  & obj$\downarrow$   & time$\downarrow$  & obj$\downarrow$   & time$\downarrow$  & obj$\downarrow$   & time$\downarrow$  & obj$\downarrow$   & time$\downarrow$  & obj$\downarrow$   & time$\downarrow$  & sum$\downarrow$ \bigstrut\\
\hline
\hline
random & 1.43  & 263.29 & 1.51  & 125.65 & 1.86  & 461.54 & 1.64  & 122.45 & 1.52  & 119.40 & 1.61  & 240.34 & 11.6  & 13.6  & 25.2 \bigstrut[t]\\
greedy & 1.19  & \rkthree{2.30} & 1.16  & \rkthree{3.08} & 1.21  & 12.52  & 1.19  & 5.87  & 1.11  & 12.94 & 1.17  & \rkthree{7.34} & 8.6  & \rkthree{4.0}   & 12.6 \\
Gurobi & 1.07  & 133.68 & 1.26  & 65.47 & 1.07  & 197.08 & 1.51  & 63.88 & 2.63  & 69.08 & 1.51  & 105.84 & 9.0   & 11.8  & 20.8 \\
SCIP  & 1.73  & 103.55 & 2.35  & 100.34 & 19.76 & 154.83 & 55.10 & 247.91 & 55.01 & 366.47 & 26.79 & 194.62 & 15.8  & 12.8  & 28.6 \\
CardNN-S & 1.14  & 15.29 & 1.06  & 8.45  & 1.62  & 36.98 & 1.16  & 11.99 & 1.08  & 10.14 & 1.21  & 16.57 & 7.4   & 5.4   & 12.8 \\
CardNN-GS & \rkthree{1.00}  & 78.38 & \rkthree{1.01}  & 74.22 & 1.07  & 76.69 & 1.15  & 21.60 & \rkthree{1.03}  & 14.99 & 1.05  & 53.18 & 4.0   & 9.0   & 13.0 \\
CardNN-HGS & 1.00  & 110.14 & 1.01  & 95.11 & 1.07  & 174.87 & 1.15  & 49.20 & 1.02  & 28.48 & 1.05  & 91.56 & \rkthree{3.4}   & 11.2  & 14.6 \\
CardNN-noTTO-S & 1.43  & \rktwo{2.23}  & 1.55  & \rktwo{1.05}  & 3.34  & \rktwo{3.90}  & 3.90  & \rkone{1.00}  & 3.54  & \rkone{1.00}  & 2.75  & \rkone{1.84}  & 14.0  & \rktwo{1.6}   & 15.6 \\
CardNN-noTTO-GS & 1.14  & 31.39 & 1.15  & 27.18 & 1.52  & 15.73 & 1.26  & 8.03  & 1.23  & \rkthree{2.21}  & 1.26  & 16.91 & 8.8   & 4.8   & 13.6 \\
CardNN-noTTO-HGS & 1.14  & 40.97 & 1.15  & 32.30 & 1.45  & 29.98 & 1.27  & 14.44 & 1.21  & 3.67  & 1.24  & 24.27 & 8.4   & 6.6   & 15.0 \\
EGN-na\"ive & 1.10  & 86.45 & 1.14  & 44.66 & 1.14  & 232.44 & 1.66  & 24.53 & 1.47  & 60.13 & 1.30  & 89.64 & 8.6   & 10.6  & 19.2 \\
RL-transductive & 2.32  & 329.11 & 2.24  & 157.07 & 10.24 & 3461.54 & 2.77  & 918.37 & 2.51  & 895.52 & 4.02  & 1152.32 & 14.6  & 15.6  & 30.2 \\
RL-inductive & 1.70  & 329.17 & 1.85  & 157.35 & 2.72  & 577.00 & 2.55  & 153.08 & 2.36  & 149.28 & 2.24  & 273.18 & 13.2  & 15.0  & 28.2 \bigstrut[b] \\
\hline
\hline
\ours-short & 1.05  & \rkone{1.00}  & 1.03  & \rkone{1.00}  & \rkthree{1.03}  & \rkone{1.00}  & \rkthree{1.05}  & \rktwo{1.31}  & 1.04  & \rktwo{1.77}  & \rkthree{1.04}  & \rktwo{1.89}  & 4.2   & \rkone{1.4}   & \rkone{5.6} \bigstrut[t] \\
\ours-middle & \rktwo{1.00}  & 32.56 & \rktwo{1.00}  & 15.65 & \rktwo{1.03}  & \rkthree{4.35}  & \rktwo{1.01}  & \rkthree{4.47}  & \rktwo{1.01}  & 13.05 & \rktwo{1.01}  & 14.02 & \rktwo{2.0}   & 4.8   & \rktwo{6.8} \\
\ours-long & \rkone{1.00}  & 81.03 & \rkone{1.00}  & 31.12 & \rkone{1.00}  & 20.27 & \rkone{1.00}  & 19.41 & \rkone{1.00}  & 22.88 & \rkone{1.00}  & 34.94 & \rkone{1.0}   & 7.8   & \rkthree{8.8} \bigstrut[b] \\
\hline
\end{tabular}%
}
\end{table*}
\egroup

%% file: TAB/res_max_cover.tex
\newcommand{\rkone}[1]{{\setlength{\fboxsep}{1.5pt}\colorbox{Red!50}{#1}}}
\newcommand{\rktwo}[1]{{\setlength{\fboxsep}{1.5pt}\colorbox{Blue!50}{#1}}}
\newcommand{\rkthree}[1]{{\setlength{\fboxsep}{1.5pt}\colorbox{Green!50}{#1}}}

\begin{table*}[t!]
\caption{Results on {maximum coverage}.
Running time (time; normalized): smaller the better.
Objective (obj; normalized): the larger the better.
In each column, 
\textcolor{Red!50}{$\blacksquare$} indicates ranking the 1st, \textcolor{Blue!50}{$\blacksquare$} the 2nd, and
\textcolor{Green!50}{$\blacksquare$} the 3rd.
} \label{tab:results_max_cover}
\small
\centering
\scalebox{0.78}{
\begin{tabular}{l||c|c||c|c||c|c||c|c||c|c||c|c|c}
\hline
\multirow{2}[3]{*}{Method} & \multicolumn{2}{c||}{rand500} & \multicolumn{2}{c||}{rand1000} & \multicolumn{2}{c||}{twitch} & \multicolumn{2}{c||}{railway} & \multicolumn{2}{c||}{average} & \multicolumn{3}{c}{average rank} \bigstrut \\
\cline{2-14}      
& obj$\uparrow$   & time$\downarrow$ 
& obj$\uparrow$   & time$\downarrow$ 
& obj$\uparrow$   & time$\downarrow$ 
& obj$\uparrow$   & time$\downarrow$ 
& obj$\uparrow$   & time$\downarrow$ 
& obj$\uparrow$   & time$\downarrow$ 
& sum$\downarrow$ \bigstrut\\
\hline
\hline
random & 0.82  & 2666.67 & 0.79  & 727.27 & 0.52  & 369.23 & 0.96  & 315.79 & 0.77  & 1019.74 & 12.8  & 14.0  & 26.8 \bigstrut[t]\\
greedy & 0.98  & \rkone{1.00} & 0.99  & \rkone{1.00} & 1.00  & \rktwo{1.06} & 1.00  & \rkone{1.00} & 0.99  & \rkone{1.02} & 7.3   & \rkone{1.3} & \rkone{8.5} \\
Gurobi & \rkthree{1.00} & 1333.89 & \rktwo{1.00} & 363.94 & \rkone{1.00} & \rkone{1.00} & 1.00  & 158.87 & 1.00  & 464.42 & \rktwo{3.3} & 8.8   & 12.0 \\
SCIP  & 0.97  & 1334.11 & 0.96  & 362.09 & \rkone{1.00} & 5.05  & 1.00  & 159.84 & 0.98  & 465.27 & 6.3   & 10.8  & 17.0 \\
CardNN-S & 0.93  & 130.33 & 0.93  & 35.94 & 1.00  & 12.25 & 0.97  & 3.71  & 0.96  & 45.56 & 8.0   & 5.5   & 13.5 \\
CardNN-GS & 0.99  & 448.11 & 1.00  & 169.55 & 1.00  & 25.38 & \rkthree{1.00} & 23.21 & 1.00  & 166.56 & 4.3   & 9.0   & 13.3 \\
CardNN-HGS & 0.99  & 618.22 & 1.00  & 248.33 & \rkthree{1.00} & 47.28 & \rktwo{1.00} & 35.86 & 1.00  & 237.42 & \rktwo{3.3} & 10.5  & 13.8 \\
CardNN-noTTO-S & 0.70  & \rkthree{20.33} & 0.69  & \rkthree{6.18} & 0.01  & \rkthree{1.43} & 0.94  & \rktwo{1.54} & 0.58  & \rkthree{7.37} & 16.0  & \rktwo{2.8} & 18.8 \\
CardNN-noTTO-GS & 0.82  & 115.56 & 0.79  & 61.18 & 0.02  & 2.97  & 0.96  & 7.53  & 0.65  & 46.81 & 13.5  & 5.0   & 18.5 \\
CardNN-noTTO-HGS & 0.82  & 132.56 & 0.79  & 75.15 & 0.19  & 3.62  & 0.96  & 12.14 & 0.69  & 55.87 & 12.3  & 6.5   & 18.8 \\
EGN-naive & 0.92  & 1334.56 & 0.91  & 364.45 & 0.13  & 185.22 & 0.97  & 159.13 & 0.73  & 510.84 & 11.3  & 12.8  & 24.0 \\
RL-transductive & 0.92  & 3333.33 & 0.82  & 909.09 & 0.95  & 2769.23 & 0.96  & 2368.42 & 0.91  & 2345.02 & 11.5  & 15.5  & 27.0 \\
RL-inductive & 0.77  & 3334.00 & 0.77  & 909.64 & 0.59  & 464.48 & 0.96  & 397.08 & 0.77  & 1276.30 & 13.8  & 15.5  & 29.3 \bigstrut[b]\\
\hline
\hline
\ours-short & 0.99  & \rktwo{10.67} & 0.99  & \rktwo{5.55} & 1.00  & 2.80  & 1.00  & \rkthree{2.63} & \rkthree{1.00} & \rktwo{5.41} & 5.8   & \rktwo{2.8} & \rkone{8.5} \bigstrut[t] \\
\ours-middle & \rktwo{1.00} & 168.44 & \rkthree{1.00} & 23.76 & 1.00  & 17.58 & 1.00  & 10.75 & \rktwo{1.00} & 55.13 & 3.8   & 6.5   & \rktwo{10.3} \\
\ours-long & \rkone{1.00} & 333.56 & \rkone{1.00} & 222.85 & 1.00  & 29.77 & \rkone{1.00} & 21.11 & \rkone{1.00} & 151.82 & \rkone{2.3} & 9.0   & 11.3 \bigstrut[b]\\
\hline
\end{tabular}%
}
\end{table*}

%% file: TAB/res_robust_clr.tex
\begingroup
\setlength{\tabcolsep}{1.7pt}
\begin{table*}[t!]
\caption{Results on robust coloring.
Running time (time; in seconds): the smaller the better.
Objective (obj): the smaller the better.
In each column, 
\textcolor{Red!50}{$\blacksquare$} indicates ranking the 1st, and \textcolor{Blue!50}{$\blacksquare$} the 2nd.
} \label{tab:results_robust_coloring}

\small
\centering
\scalebox{0.75}{
\begin{tabular}{l||c|c||c|c||c|c||c|c||c|c||c|c||c|c||c|c||c|c|c}
\hline
      \multirow{2}[3]{*}{Method} & \multicolumn{2}{c||}{collins, 18 colors} & \multicolumn{2}{c||}{collins, 25 colors} & \multicolumn{2}{c||}{gavin, 8 colors} & \multicolumn{2}{c||}{gavin, 15 colors} & \multicolumn{2}{c||}{krogan, 8 colors} & \multicolumn{2}{c||}{krogan, 15 colors} & \multicolumn{2}{c||}{ppi, 47 colors} & \multicolumn{2}{c||}{ppi, 50 colors} & \multicolumn{3}{c}{average rank} \bigstrut \\
\cline{2-20}
& obj$\downarrow$   & time$\downarrow$  & obj$\downarrow$   & time$\downarrow$  & obj$\downarrow$   & time$\downarrow$  & obj$\downarrow$   & time$\downarrow$  & obj$\downarrow$   & time$\downarrow$  & obj$\downarrow$   & time$\downarrow$  & obj$\downarrow$   & time$\downarrow$  & obj$\downarrow$   & time$\downarrow$  & obj$\downarrow$   & time$\downarrow$  & sum$\downarrow$ \bigstrut\\
\hline
\hline
greedy-RD & 115.33 & 300.34 & 23.42 & 300.79 & 66.51 & 300.53 & 7.36  & 301.46 & 117.47 & 300.06 & {\setlength{\fboxsep}{1.5pt}\colorbox{Red!50}{0.87}} & 301.24 & 4.16  & 301.31 & {\setlength{\fboxsep}{1.5pt}\colorbox{Blue!50}{1.23}} & 301.24 & 2.88  & 3.25  & {\setlength{\fboxsep}{1.5pt}\colorbox{Blue!50}{6.13}} \bigstrut[t]\\
greedy-GA & 114.36 & {\setlength{\fboxsep}{1.5pt}\colorbox{Blue!50}{188.21}} & 22.20 & {\setlength{\fboxsep}{1.5pt}\colorbox{Blue!50}{243.93}} & 66.51 & 398.90 & 7.36  & 540.62 & 117.47 & 941.35 & {\setlength{\fboxsep}{1.5pt}\colorbox{Red!50}{0.87}} & 1256.66 & {\setlength{\fboxsep}{1.5pt}\colorbox{Blue!50}{3.66}} & 1416.38 & {\setlength{\fboxsep}{1.5pt}\colorbox{Blue!50}{1.23}} & 1484.27 & {\setlength{\fboxsep}{1.5pt}\colorbox{Blue!50}{2.50}} & 4.25  & 6.75 \\
DC    & 586.56 & 300.28 & 159.15 & 300.38 & 311.91 & {\setlength{\fboxsep}{1.5pt}\colorbox{Blue!50}{300.11}} & 58.10 & {\setlength{\fboxsep}{1.5pt}\colorbox{Blue!50}{300.12}} & 1065.52 & {\setlength{\fboxsep}{1.5pt}\colorbox{Blue!50}{300.07}} & 1.76  & 300.46 & 43.35 & {\setlength{\fboxsep}{1.5pt}\colorbox{Blue!50}{300.13}} & 6.72  & {\setlength{\fboxsep}{1.5pt}\colorbox{Blue!50}{300.76}} & 5.00  & {\setlength{\fboxsep}{1.5pt}\colorbox{Blue!50}{2.50}} & 7.50 \\
Gurobi & {\setlength{\fboxsep}{1.5pt}\colorbox{Blue!50}{87.28}} & 301.71 & {\setlength{\fboxsep}{1.5pt}\colorbox{Blue!50}{16.23}} & 306.10 & {\setlength{\fboxsep}{1.5pt}\colorbox{Red!50}{42.41}} & 300.80 & {\setlength{\fboxsep}{1.5pt}\colorbox{Blue!50}{7.28}} & 303.50 & {\setlength{\fboxsep}{1.5pt}\colorbox{Red!50}{46.78}} & 300.80 & {\setlength{\fboxsep}{1.5pt}\colorbox{Red!50}{0.87}} & {\setlength{\fboxsep}{1.5pt}\colorbox{Blue!50}{51.70}} & 4.60  & 328.48 & 1.31  & 313.23 & {\setlength{\fboxsep}{1.5pt}\colorbox{Blue!50}{2.50}} & 4.00  & 6.50 \\
\hline
\hline
\ours (CPU) & 
\multirow{2}[3]{*}{{\setlength{\fboxsep}{1.5pt}\colorbox{Red!50}{82.26}}} & {\setlength{\fboxsep}{1.5pt}\colorbox{Red!50}{79.36}} & 
\multirow{2}[3]{*}{{\setlength{\fboxsep}{1.5pt}\colorbox{Red!50}{15.16}}} & 
{\setlength{\fboxsep}{1.5pt}\colorbox{Red!50}{54.37}} & 
\multirow{2}[3]{*}{{\setlength{\fboxsep}{1.5pt}\colorbox{Blue!50}{42.99}}} &
{\setlength{\fboxsep}{1.5pt}\colorbox{Red!50}{152.20}} & 
\multirow{2}[3]{*}{{\setlength{\fboxsep}{1.5pt}\colorbox{Red!50}{6.72}}} &
{\setlength{\fboxsep}{1.5pt}\colorbox{Red!50}{260.90}} & 
\multirow{2}[3]{*}{{\setlength{\fboxsep}{1.5pt}\colorbox{Blue!50}{53.44}}} & 
{\setlength{\fboxsep}{1.5pt}\colorbox{Red!50}{211.43}} & 
\multirow{2}[3]{*}{{\setlength{\fboxsep}{1.5pt}\colorbox{Red!50}{0.87}}} & 
{\setlength{\fboxsep}{1.5pt}\colorbox{Red!50}{8.55}} & 
\multirow{2}[3]{*}{{\setlength{\fboxsep}{1.5pt}\colorbox{Red!50}{2.93}}} & 
{\setlength{\fboxsep}{1.5pt}\colorbox{Red!50}{116.54}} &
\multirow{2}[3]{*}{{\setlength{\fboxsep}{1.5pt}\colorbox{Red!50}{1.01}}} & 
{\setlength{\fboxsep}{1.5pt}\colorbox{Red!50}{120.56}} &
\multirow{2}[3]{*}{{\setlength{\fboxsep}{1.5pt}\colorbox{Red!50}{1.50}}} & 
\multirow{2}[3]{*}{{\setlength{\fboxsep}{1.5pt}\colorbox{Red!50}{1.00}}} & 
\multirow{2}[3]{*}{{\setlength{\fboxsep}{1.5pt}\colorbox{Red!50}{2.50}}} \bigstrut[t] \\
\ours (GPU) &  & {\setlength{\fboxsep}{1.5pt}\colorbox{Red!50}{7.09}} & & {\setlength{\fboxsep}{1.5pt}\colorbox{Red!50}{8.03}} & & {\setlength{\fboxsep}{1.5pt}\colorbox{Red!50}{13.28}} & & {\setlength{\fboxsep}{1.5pt}\colorbox{Red!50}{17.25}} & & {\setlength{\fboxsep}{1.5pt}\colorbox{Red!50}{13.73}} & & {\setlength{\fboxsep}{1.5pt}\colorbox{Red!50}{1.91}} & & {\setlength{\fboxsep}{1.5pt}\colorbox{Red!50}{5.24}} & & {\setlength{\fboxsep}{1.5pt}\colorbox{Red!50}{5.48}} & & & \bigstrut[b]\\
\hline
\end{tabular}%
}
\vspace{-5mm}
\end{table*}
\endgroup

%% file: 070concl.tex
\section{General Related Work: Learning for CO}

We shall discuss more general related works, including other learning-based methods for CO problems.

\smallsection{Reinforcement learning for CO.}
Typical techniques include reinforcement learning (RL).
The pioneers who applied RL to CO  include~\cite{bello2016neural} and \cite{khalil2017learning}.
Most reinforcement-learning-for-combinatorial-optimization methods focus on routing problems such as the traveling salesman problem (TSP) and the vehicle routing problem (VRP)~\citep{berto2023rl4co,kool2018attention,kim2021learning,delarue2020reinforcement,qiu2022dimes,nazari2018reinforcement,ye2023deepaco,chalumeau2023copa,luo2023heavydecoder,grinsztajn2023winner,ye2024glop,xiao2024distilling}, as well as maximum independent sets (MIP)~\citep{ahn2020learning,qiu2022dimes,sun2023difusco,li2023distribution}.
See also some recent surveys on RL4CO~\citep{mazyavkina2021reinforcement,bengio2021machine,cappart2023combinatorial,munikoti2023challenges} for more details.
The existing RL-based methods still suffer from efficiency issues.
See the discussions by~\cite{wang2022unsupervised} and \cite{wang2022towards}.
See also App.~\ref{subapp:discussions:rl_and_egn} for more discussions.

\smallsection{Other machine-learning techniques for CO.}
Some other machine-learning techniques have been proposed for CO.
There is recent progress based on
search~\citep{choo2022simulation,son2023meta,li2023distribution},
sampling~\citep{sun2023revisiting},
graph-based diffusion~\citep{sun2023difusco},
generative flow networks~\citep{zhang2023let},
meta-learning~\citep{qiu2022dimes,wang2023unsupervised},
and quantum machine learning~\citep{ye2023towards}.
Physics-inspired machine learning has also been considered by researchers~\citep{schuetz2022combinatorial,aramon2019physics,schuetz2022graph}.
There is also a line of research on perturbation-based methods for CO~\citep{poganvcic2019differentiation,berthet2020learning,paulus2021comboptnet,ferber2023surco}.

\section{Conclusion and Discussion}
In this work, we study and propose 
\ours (\bus{U}supervised \bus{Com}binatorial Optimization \bus{U}nder \revise{\bus{Com}monly-involved} Conditions).
Specifically, we concretize the targets 
for probabilistic objective construction and derandomization 
(Sec.~\ref{sec:proposed_method}) with theoretical justification (Theorems~\ref{thm:concave_exp_prob} and \ref{thm:greedy_like_der_good}),
derive non-trivial objectives and derandomization for various conditions (e.g., cardinality constraints and minimum) to meet the targets (Sec.~\ref{sec:analy_conds}; Lemmas~\ref{lem:card_tight_ub} to \ref{lem:cliq_incre}),
apply the derivations to different problems involving such conditions (Sec.~\ref{sec:problems}),
and finally show the empirical superiority of our method via extensive experiments (Sec.~\ref{sec:experiments}).
For reproducibility, we share the code and datasets online \citep{appendix}.

\revise{
As discussed in Sec.~\ref{subsec:derive_notes}, we have not covered all conditions involved in CO in this work, while we believe that our high-level ideas are applicable to other conditions and problems.
The performance of \ours and general UL4CO on other conditions~\citep{min2023unsupervised,lachapelle2019gradient} and hard instances~\citep{xu2007random,li2023hardsatgen} are interesting topics for future exploration.
}

%% file: A00appdx.tex
\input{A01proof.tex}
\input{A02bkgrd.tex}

\input{030relwk.tex}
\input{A03tcncl.tex}

\input{A07theor.tex}

\input{A04prblm.tex}

\input{A05reslt.tex}

\input{A06dscss.tex}

%% file: A01proof.tex
\section{Proofs}\label{app:proofs}
Here, we provide proof for each theoretical statement in the main text.

\begin{reptheorem}{thm:concave_exp_prob}[Expectations are all you need]
    For any $g: \Set{0, 1}^n \to \bbR$,
    $\tilde{g}: [0, 1]^n \to \bbR$ with $\Tilde{g}(p) = \bbE_{X \sim p} g(X)$
    is differentiable and entry-wise concave w.r.t. $p$.
\end{reptheorem}
\begin{proof}
    For any $p$ and $i$, we have     
    \begin{align*}
     \Tilde{g}(p)
    &=\bbE_{X \sim p} g(X) \\
    &=\sum_{X \in \Set{0,1}^n} \Pr\nolimits_{p}[X] g(X) \\
    &=\sum_{X} \prod_{v \in V_X} p_v \prod_{u \in [n] \setminus V_X} (1 - p_u) g(X) \\
    &=\sum_{X \colon i \in V_X}\left(\prod_{v\in V_X,v\neq i}p_v\prod_{u\in[n]\setminus V_X}(1-p_u)\right) p_i g(X) 
    + \sum_{X \colon i \notin V_X}\left(\prod_{v\in V_X}p_v\prod_{u\in[n]\setminus V_X,u\neq i}(1-p_u)\right) (1 - p_i) g(X) \\
    &=p_i \sum_{X \colon i \in V_X}\prod_{v\in V_X,v\neq i}p_v\prod_{u\in[n]\setminus V_X}(1-p_u) g(X) + (1 - p_i) \sum_{X \colon i \notin V_X}\prod_{v\in V_X}p_v\prod_{u\in[n]\setminus V_X,u\neq i}(1-p_u) g(X)\\
    &=p_i \tilde{g}(\der(i, 1; p)) + (1 - p_i) \tilde{g}(\der(i, 0; p))\\
    &\geq p_i \tilde{g}(\der(i, 1; p)) + (1 - p_i) \tilde{g}(\der(i, 0; p)),
    \end{align*}    
    completing the proof on entry-wise concavity.
    Regarding differentiability, 
    since 
    \[
    \bbE_{X \sim p} g(X) = 
    \sum_{X \in \Set{0,1}^n} \Pr\nolimits_{p}[X] g(X),
    \]    
    it suffices to show that
    \[
    \Pr\nolimits_{p}[X] g(X) = \prod_{v \in V_X} p_v \prod_{u \in [n] \setminus V_X} (1 - p_u) g(X)
    \]
    is differentiable w.r.t $p$ for each $X \in \Set{0,1}^n$.
    Indeed, fix any $X$, $\prod_{v \in V_X} p_v \prod_{u \in [n] \setminus V_X} (1 - p_u) g(X)$ is a polynomial of $p_i$'s,
    and is thus differentiable w.r.t. $p$.    
\end{proof}

\begin{reptheorem}{thm:greedy_like_der_good}[Goodness of greedy derandomization]
    For any entry-wise concave $\Tilde{f}$ and any $p_{\mathrm{o}} \in [0, 1]^n$, the above process of greedy derandomization can always reach a point where the final $p_{\mathrm{final}}$ is 
    \textbf{(G1)} discrete (i.e., $p_{\mathrm{final}} \in \Set{0, 1}^n$),
    \textbf{(G2)} no worse than $p_{\mathrm{o}}$ (i.e., $\Tilde{f}(p_{\mathrm{final}}) \leq \Tilde{f}(p_{\mathrm{o}})$),
    and \textbf{(G3)} a local minimum (i.e., $\Tilde{f}(p_{\mathrm{final}}) \leq \min_{(i, x) \in [n] \times \Set{0, 1}} \Tilde{f}(\operatorname{der}(i, x; p_{\mathrm{final}}))$).
\end{reptheorem}
\begin{proof}[Proof of Theorem~\ref{thm:greedy_like_der_good}]
    First, we claim that for any non-discrete $p_{\mathrm{cur}} \notin \Set{0, 1}^n$, we can always derandomize it through a series of local derandomization while the value of $\tilde{f}$ does not increase.
    This is guaranteed by the entry-wise concavity of $\Tilde{f}$.
    Specifically, since 
    \[
    p_i \Tilde{f}(\operatorname{der}(i, 1; p)) + (1 - p_i) \Tilde{f}(\operatorname{der}(i, 0; p)) \leq \Tilde{f}(p), \forall p, i,
    \]
    we have 
    \[
    \min(\operatorname{der}(i, 1; p), \operatorname{der}(i, 0; p)) \leq \Tilde{f}(p), \forall p, i,
    \]
    which implies that we can always derandomize a non-discrete entry without increasing the value of $\tilde{f}$.
    Therefore, if we greedily improve $\Tilde{f}$ via local derandomization, we can always terminate at a discrete point, completing the proof for point (G1).
    Point (G2) holds since at each step we make sure that the value of $\Tilde{f}$ does not increase.
    Point (G3) holds from the way we conduct local derandomization.
    Specifically, if the current $p_{\mathrm{cur}}$ is not a local minimum, we can always find a possible local derandomization step to proceed with the process while strictly decreasing the value of $\Tilde{f}$.
\end{proof}

\begin{replemma}{lem:card_tight_ub}
    $\hat{f}_{\mathrm{card}}$ is a tight upper bound of $\mathbb{1}[\Abs{V_X} \notin C_c]$.
\end{replemma}
\begin{proof}
    When $X \notin C_c$, 
    $\Abs{V_X} \neq k, \forall k \in C_c$, and thus
    $\hat{f}_{\mathrm{card}}(X; C_c) = \min_{k \in C_c} \Abs{\Abs{V_X} - k} > 0$ and thus $\hat{f}_{\mathrm{card}}(X; C_c) \geq 1$ since $\Abs{V_X} \neq k$ is an integer.
    When $X \in C_c$,
    $\exists k \in C_c, \Abs{V_X} = k$ and thus
    $\hat{f}_{\mathrm{card}}(X; C_c) = \min_{k \in C_c} \Abs{\Abs{V_X} - k} = 0$.
\end{proof}

\begin{replemma}{lem:card_incre}[IDs of $\Tilde{f}_{\mathrm{card}}$]
    For any $p, i, t$,
    let $q_s \coloneqq \Pr_{X \sim p}[\Abs{V_X} = s]$
    and $q'_s \coloneqq \Pr_{X \sim p}[\Abs{V_X \setminus \Set{i}} = s], \forall s$,    
    \setlength{\jot}{0pt}
    \begin{align*}
    q'_t &= (1 - p_i)^{-1} \sum_{s = 0}^t q_s \left(\frac{p_i}{p_i-1}\right)^{t-s} \text{~(if~ $p_i \neq 1$)} \\
    &= (p_i)^{-1} \sum_{s = 0}^{n - t - 1} q_{t+s+1} \left(\frac{p_i-1}{p_i}\right)^{s} \text{~(if~ $p_i \neq 0$)}.
    \end{align*}
    Based on that,
    \[
    \begin{cases} 
    \Delta \Tilde{f}_{\mathrm{card}}(i, 0, p; C_c) = 
    \sum_{t \in [n] \setminus C_c} (q'_t - q_t) \min_{k \in C_c} \Abs{t - k}, \\
    \Delta \Tilde{f}_{\mathrm{card}}(i, 1, p; C_c) = 
    \sum_{t \in [n] \setminus C_c} (q'_{t-1} - q_t) \min_{k \in C_c} \Abs{t - k}.
    \end{cases}
    \]        
\end{replemma}
\begin{proof}
    Fix any $i \in [n]$ and any $p \in [0, 1]^n$, we have
    \begin{equation}\label{eq:poi_bin_rmv_node}
    \begin{cases}
        \Pr_{X \sim p}[\Abs{V_X} = 0] = \Pr_{X \sim p}[\Abs{V_X \setminus \Set{i}} = 0] (1 - p_i) \\
        \Pr_{X \sim p}[\Abs{V_X} = t] = 
        \Pr_{X \sim p}[\Abs{V_X \setminus \Set{i}} = t] (1 - p_i) + 
        \Pr_{X \sim p}[\Abs{V_X \setminus \Set{i}} = t - 1] p_i,
        \forall t \\
        \Pr_{X \sim p}[\Abs{V_X} = n] = \Pr_{X \sim p}[\Abs{V_X \setminus \Set{i}} = n - 1] p_i
    \end{cases}.
    \end{equation}
    Let $q_s$ denote $\Pr_{X \sim p}[\Abs{V_X} = s]$ for each $s$ as in the statement,
    and also let $\tilde{q}_s$ denote $\Pr_{X \sim p}[\Abs{V_X \setminus \Set{i}} = s]$.
    By Equation~\eqref{eq:poi_bin_rmv_node}, if we start from 
    $q_0 = \tilde{q}_0 (1 - p_i)$,
    we have
    \[
    \tilde{q}_0 = \frac{q_0}{1 - p_i},    
    \tilde{q}_1 = \frac{q_1 - p_i \tilde{q}_0}{1 - p_i} = \frac{q_1 (1 - p_i) - q_0 p_i}{(1 - p_i)^2}, \cdots,
    \]
    which satisfies $\tilde{q}_t
    = (1 - p_i)^{-1} \sum_{s = 0}^t q_s \left(\frac{p_i}{p_i-1}\right)^{t-s}$.
    Now, if 
    \[
    \tilde{q}_t
    = (1 - p_i)^{-1} \sum_{s = 0}^t q_s \left(\frac{p_i}{p_i-1}\right)^{t-s}
    \]
    holds for all $t \leq T - 1$, we aim to show that it also holds for $t = T$, which shall prove the statement by mathematical induction.
    Indeed, we have 
    \[
    \tilde{q}_T = \frac{q_T - p_i \tilde{q}_{T - 1}}{1 - p_i} = \frac{q_T - p_i (1 - p_i)^{-1} \sum_{s = 0}^{T - 1} q_s \left(\frac{p_i}{p_i-1}\right)^{T - 1 -s}}{1 - p_i} = (1 - p_i)^{-1} \sum_{s = 0}^T q_s \left(\frac{p_i}{p_i-1}\right)^{T-s},
    \]
    completing the proof.
    If we start from $q_n = \Tilde{q}_{n-1} p_i$, we can obtain the another term (i.e., $(p_i)^{-1} \sum_{s = 0}^{n - t - 1} q_{t+s+1} \left(\frac{p_i-1}{p_i}\right)^{s}$) in the statement in a similar way.
    
    In practice, we use $(1 - p_i)^{-1} \sum_{s = 0}^t q_{t-s} \left(\frac{p_i}{p_i-1}\right)^{t - s}$ for $0 \leq p_i \leq 0.5$ and $(p_i)^{-1} \sum_{s = 0}^{n - t - 1} q_{t+s+1} \left(\frac{p_i-1}{p_i}\right)^{s}$ for $0.5 < p_i \leq 1$, which results in higher numerical stability.
\end{proof}

\begin{replemma}{lem:optim_tight_ub}
    $\hat{f}_{\mathrm{ms}}$ is a tight upper bound of ${f}_{\mathrm{ms}}$.
\end{replemma}
\begin{proof}
    As mentioned in Rem.~\ref{rem:tight_ub}, since $\hat{f}_{\mathrm{ms}} = {f}_{\mathrm{ms}}$, $\hat{f}$ is a tight upper bound of ${f}_{\mathrm{ms}}$.
\end{proof}

\begin{replemma}{lem:optim_exp}
    For any $p \in [0, 1]^n$, 
    $\bbE_{X \sim p} \hat{f}_{\mathrm{ms}}(X; i, h) = p_{v_1} d_1 + (1 - p_{v_1})p_{v_2} d_2 + \cdots + (\prod_{j = 1}^{n-1} (1-p_{v_j}))p_{v_{n}} d_n$.
\end{replemma}
\begin{proof}
    By the definition of expectation,
    \[
    \bbE_{X \sim p} \hat{f}_{\mathrm{ms}}(X; i, h) = \sum_t \Pr[\min_{v_X \in V_X} h(i, v_X) = d_t] d_t,
    \]
    where
    $\min_{v_X \in V_X} h(i, v_X) = d_t$ if and only if $v_{t'} \notin V_x$ for each $t' < t$ and $v_t \in V_x$, which has probability $(\prod_{j = 1}^{t-1} (1-p_{v_j}))p_{v_{t}}$.
    Hence,
    \[
    \bbE_{X \sim p} \hat{f}_{\mathrm{ms}}(X; i, h) = p_{v_1} d_1 + (1 - p_{v_1})p_{v_2} d_2 + \cdots + (\prod_{j = 1}^{n-1} (1-p_{v_j}))p_{v_{n}} d_n.
    \]    
\end{proof}

\begin{replemma}{lem:opt_incre}[IDs of $\Tilde{f}_{\mathrm{ms}}$]
    For any $p \in [0, 1]^n$ 
    and $j \in [n]$,     
    let $q_j \coloneqq (\prod_{k = 1}^{j-1} (1-p_{v_{k}}))p_{v_{j}}$, the coefficient of $d_j$ in $\tilde{f}_{\mathrm{ms}}$. Then
    \[
    \begin{cases} 
    \Delta \Tilde{f}_{\mathrm{ms}}(v_j, 0, p; i, h) = - q_j d_j 
     + \frac{p_{v_j}}{1-p_{v_j}}\sum_{j' > j} q_{j'}d_{j'},\\
     \Delta \Tilde{f}_{\mathrm{ms}}(v_j, 1, p; i, h) = 
    \sum_{j' > j} q_{j'}(d_j - d_{j'}).
    \end{cases}
    \]    
\end{replemma}
\begin{proof}
    When $p' = \operatorname{der}(v_j, 0; p)$, we have
    \begin{align*}
    \Tilde{f}_{\mathrm{ms}}(p'; i,h) 
    &= p'_{v_1} d_1 + (1 - p'_{v_1})p'_{v_2} d_2 + \cdots + (\prod_{s = 1}^{n-1} (1-p'_{v_s}))p'_{v_{n}} d_n \\
    &= \sum_{s < j} \prod_{k = 1}^{s - 1} (1-p_{v_k})p_{v_{s}} d_s + 0 + \sum_{t > j} \prod_{1 \leq k \leq t - 1, k \neq j} (1-p_{v_{k}}) p_{v_{t}} d_{t}\\
    &= \sum_{s < j} q_s d_s + 0 + \sum_{j' > j} \frac{1}{1 - p_{v_j}} q_{j'} d_{j'}\\
    &= \sum_{s = 1}^{n} q_s d_s - q_jd_j + \sum_{j' > j} \frac{p_{v_{j}}}{1 - p_{v_{j}}} q_{j'} d_{j'}\\
    &= \Tilde{f}_{\mathrm{ms}}(p; i,h)
     - q_j d_j 
     + \frac{p_{v_j}}{1-p_{v_j}}\sum_{j' > j} q_{j'}d_{j'}.
    \end{align*}    
    When $p' = \operatorname{der}(v_j, 1; p)$, we have
    \begin{align*}
    \Tilde{f}_{\mathrm{ms}}(p'; i,h)
    &= p'_{v_1} d_1 + (1 - p'_{v_1})p'_{v_2} d_2 + \cdots + (\prod_{s = 1}^{n-1} (1-p'_{v_s}))p'_{v_{n}} d_n \\
    &= \left(\sum_{s < j} \prod_{k = 1}^{s - 1} (1-p_{v_k}) 
    p_{v_{s}} d_s\right) + \prod_{k' = 1}^{j - 1} (1-p_{v_k'}) d_j \\
    &= \left(\sum_{s < j} \prod_{k = 1}^{s - 1} (1-p_{v_k}) 
    p_{v_{s}} d_s\right) + \sum_{j' \geq j} \prod_{k' = 1}^{j' - 1} (1-p_{v_{k'}}) p_{v_{j'}} d_j \\    
    &= \sum_{s < j} q_s d_s + \sum_{j' \geq j} q_{j'} d_j \\    
    &= \sum_{s = 1}^n q_s d_s + 0 + \sum_{j' > j} q_{j'} (d_j - d_{j'})\\
    &= \Tilde{f}_{\mathrm{ms}}(p; i,h)
    + \sum_{j' > j} q_{j'}(d_j - d_{j'}).
    \end{align*}
\end{proof}

\begin{replemma}{lem:optim_tight_cover}
    $\hat{f}_{\mathrm{cv}}$ is a TUB of $\mathbb{1}(X \notin \calC)$.
\end{replemma}
\begin{proof}
    As mentioned in Rem.~\ref{rem:tight_ub}, since $\hat{f}_{\mathrm{cv}} = \mathbb{1}(X \notin \calC)$, 
    $\hat{f}_{\mathrm{cv}}$ is a tight upper bound of $\mathbb{1}(X \notin \calC)$.
\end{proof}

\begin{replemma}{lem:cover_obj}
    For any $p \in [0, 1]^n$ and $i \in [n]$,    
    $\Pr_{X \sim p}[\Set{v_X \in V_X \colon (v_X, i) \in E} \neq \emptyset] = \prod_{v \in [n] \colon (v, i) \in E} (1 - p_v)$.
\end{replemma}
\begin{proof}
    We decompose the event $\Set{v_X \in V_X \colon (v_X, i) \in E} \neq \emptyset = \bigwedge_{v \colon (v, i) \in E} (v \notin V_X)$.
    Since the subevents $v \notin V_X$ are mutually independent, 
    \[
    \Pr\nolimits_{X \sim p}[\Set{v_X \in V_X \colon (v_X, i) \in E} \neq \emptyset] = \prod_{v \in [n] \colon (v, i) \in E} (1 - p_v).
    \]    
\end{proof}

\begin{replemma}{lem:cover_incre}[IDs of $\Tilde{f}_{\mathrm{cv}}$]
    For any $p \in [0, 1]^n$ and $i \in [n]$,
    if $(i, j) \notin E$,
    then $\Delta \Tilde{f}_{\mathrm{cv}}(j, 0, p; i) = \Delta \Tilde{f}_{\mathrm{cv}}(j, 1, p; i) = 0$;    
    if $(i, j) \in E$,
    then
    \[
    \begin{cases}
    \Delta \Tilde{f}_{\mathrm{cv}}(j, 0, p; i) = p_j \prod_{v \in N_i, v \neq j}(p_v - 1),\\
    \Delta \Tilde{f}_{\mathrm{cv}}(j, 1, p; i) = -\Tilde{f}_{\mathrm{cv}}(p; i).
    \end{cases}
    \]    
\end{replemma}
\begin{proof}
    If $(i, j) \notin E$, the value of $p_j$ does not affect $\Tilde{f}_{\mathrm{cv}}(p; i)$ since $p_j$ is not involved in the value of $\Tilde{f}_{\mathrm{cv}}(p; i)$.
    When $(i, j) \in E$, 
    if $p' = \operatorname{der}(j, 0; p)$,
    \[
    \prod_{v \in N_i} (1 - p'_v) = \prod_{v \in N_i, v \neq j}(1 - p_v) = \Tilde{f}_{\mathrm{cv}}(p; i) - p_j \prod_{v \in N_i, v \neq j}(1 - p_v);
    \]    
    if $p' = \operatorname{der}(j, 1; p)$,
    \[
    \prod_{v \in N_i} (1 - p'_v) = 0.
    \]    
\end{proof}

\begin{replemma}{lem:cliq_obj}
    For any $p \in [0, 1]^n$, 
    $\bbE_{X \sim p} \hat{f}_{\mathrm{cq}}(X)
    = \sum_{(u, v) \in \binom{V}{2} \setminus E} p_u p_v$.
\end{replemma}
\begin{proof}
    By linearity of expectation and double counting,
    \[
    \hat{f}_{\mathrm{cq}}(X) = \sum_{(u, v) \in \binom{V_X}{2}} \mathbb{1}[(u, v) \notin E] = \sum_{(u, v) \notin E} \mathbb{1}[(u, v) \in \binom{V_X}{2}].
    \]
    Then we take the expectation and use the mutual independency among $v \in V_X$'s to get
    \[
    \bbE_{X \sim p}[\hat{f}_{\mathrm{cq}}(X)] = 
    \sum_{(u, v) \notin E} \Pr[(u, v) \in \binom{V_X}{2}] = 
    \sum_{(u, v) \notin E} p_u p_v.
    \]
\end{proof}

\begin{replemma}{lem:cliq_incre}[IDs of $\Tilde{f}_{\mathrm{cq}}$]
    For any $p \in [0, 1]^n$ and $i \in [n]$,
    \[
    \begin{cases}
     \Delta \tilde{f}_{\mathrm{cq}}(i, 0, p) =
    - p_i \sum_{j \in [n], j \neq i, (i, j) \notin E} p_j, \\
    \Delta \tilde{f}_{\mathrm{cq}}(i, 1, p) =
    (1 - p_i) \sum_{j \in [n], j \neq i, (i, j) \notin E} p_j.
    \end{cases}
    \]
\end{replemma}
\begin{proof}
    When $p' = \operatorname{der}(i, 0; p)$,
    \[
    \tilde{f}_{cq}(p') 
    = \sum_{(u, v) \in \binom{V}{2} \setminus E} p'_u p'_v
    = \sum_{(u, v) \in \binom{V}{2} \setminus E, u \neq i, v \neq i} p_u p_v
    = \tilde{f}_{cq}(p) - p_i \sum_{j \in [n], j \neq i, (i, j) \notin E} p_j.
    \]
    When $p' = \operatorname{der}(i, 1; p)$,
    \[
    \tilde{f}_{cq}(p')
    = \sum_{(u, v) \in \binom{V}{2} \setminus E, u \neq i, v \neq i} p_u p_v + 
    \sum_{(i, j) \in \binom{V}{2} \setminus E} p_j
    = \tilde{f}_{cq}(p) + (1 - p_i) \sum_{j \in [n], j \neq i, (i, j) \notin E} p_j.
    \]
\end{proof}

\begin{replemma}{lem:rc_tub}
    $\hat{g}_1$ is a TUB of $g_1$ and 
    $\hat{f}_2$ is a TUB of $f_2$.
\end{replemma}
\begin{proof}    
    When $X \notin \calC_1$,
    at least one edge in $E_h$ is violated, i.e.,
    $\hat{g}_1(X) \geq 1$.
    When $X \in \calC_1$,
    no edge in $E_h$ is violated, i.e.,
    $\hat{g}_1(X) = 0$.

    $\hat{f}_2 = f_2$ is a TUB of itself.
\end{proof}

\begin{replemma}{lem:rc_obj}
    For any $p \in [0, 1]^{n \times c}$,    
    $\tilde{f}_2(p) = \bbE_{X \sim p} \hat{f}_{2}(X) = -\sum_{e = (u, v) \in E_s} \sum_{r = 0}^{c-1} p_{ur}p_{vr} \log (1 - P(e))$ and
    $\tilde{g}_1(p) = \bbE_{X \sim p} \hat{g}_{1}(X) = \sum_{(u, v) \in E_h} \sum_{r = 0}^{c-1} p_{ur} p_{vr}$.
\end{replemma}
\begin{proof}
    We have
    \begin{align*}
        \bbE_{X \sim p} \hat{f}_{2}(X) 
        &= \bbE_{X \sim p} -\sum_{e = (u, v) \in E_s \colon X_u = X_v} \log (1 - P(e)) \\
        &= \bbE_{X \sim p} -\sum_{e = (u, v) \in E_s} \mathbb{1}(X_u = X_v) \log (1 - P(e)) \\
        &= -\sum_{e = (u, v) \in E_s} \Pr\nolimits_{X \sim p}[X_u = X_v] \log (1 - P(e)) \\
        &= -\sum_{e = (u, v) \in E_s} \sum_{r = 0}^{c-1} \Pr\nolimits_{X \sim p}[\text{both $u$ and $v$ have color $r$}] \log (1 - P(e)) \\
        &= -\sum_{e = (u, v) \in E_s} \sum_{r = 0}^{c-1} p_{ur} p_{vr} \log (1 - P(e))
    \end{align*}
    and
    \begin{align*}
        \bbE_{X \sim p} \hat{g}_{1}(X)
        &= \bbE_{X \sim p} \Abs{\Set{(u, v) \in E_h \colon X_u = X_v}} \\
        &= \bbE_{X \sim p} \sum_{(u, v) \in E_h} \mathbb{1}(X_u = X_v) \\
        &= \sum_{(u, v) \in E_h} \Pr\nolimits_{X \sim p}[X_u = X_v] \\
        &= \sum_{(u, v) \in E_h} \Pr\nolimits_{X \sim p}[\text{both $u$ and $v$ have color $r$}] \\
        &= \sum_{(u, v) \in E_h} \sum_{r = 0}^{c-1} p_{ur} p_{vr}.
    \end{align*}    
\end{proof}

\begin{replemma}{lem:rc_incre}[IDs of the terms in $\Tilde{f}_{\mathrm{RC}}$]
    For any $p \in [0, 1]^{n \times c}$, $i \in [n]$, and $x \in d$,
    $\Delta \Tilde{g}_{1}(i, x; p) = 
    \sum_{x' \in d \setminus \Set{x}} p_{ix'} \sum_{(i, j) \in E_h} (p_{jx} - p_{jx'})$
    and
    $\Delta \Tilde{f}_{2}(i, x; p) = 
    \sum_{x' \in d \setminus \Set{x}} p_{ix'} \sum_{(i, j) \in E_s} (p_{jx'} - p_{jx}) \log(1 - P(i, j))$.
\end{replemma}
\begin{proof}
    When $p' = \der(i, x; p)$,
    \begin{align*}
     \Tilde{g}_{1}(\der(i, x; p)) 
     &= \sum_{(u, v) \in E_h} \sum_{r = 0}^{c-1} p'_{ur} p'_{vr} \\
     &= \sum_{(u, v) \in E_h} \sum_{r = 0}^{c-1} p'_{ur} p'_{vr} \\
     &= \sum_{(u, v) \in E_h, u \neq i, v \neq i} \sum_{r = 0}^{c-1} p_{ur} p_{vr} + \sum_{(i, j) \in E_h} p_{jx}\\
     &= \Tilde{g}_{1}(p) + (1 - p_{ix}) \sum_{(i, j) \in E_h} p_{jx} - \sum_{x' \in d \setminus \Set{x}} p_{ix'} \sum_{(i, j) \in E_h} p_{jx'} \\
     &= \Tilde{g}_{1}(p) + \sum_{x' \in d \setminus \Set{x}} p_{ix'} \sum_{(i, j) \in E_h} (p_{jx} - p_{jx'}),
    \end{align*}
    where $1 - p_{ix} = \sum_{x' \in d \setminus \Set{x}} p_{ix'}$ has been used.
    Similarly,
    \begin{align*}    
    \Tilde{f}_{2}(\der(i, x; p))
    &=  -\sum_{e = (u, v) \in E_s} \sum_r p'_{ur}p'_{vr} \log (1 - P(e)) \\
    &=  -\sum_{e = (u, v) \in E_s, u \neq i, v \neq i} \sum_r p'_{ur}p'_{vr} \log (1 - P(e)) -\sum_{e = (i, j) \in E_s} p_{jx} \log (1 - P(e)) \\
    &= \Tilde{f}_{2}(p) - (1 - p_{ix}) \sum_{(i, j) \in E_s} p_{jx} \log(1 - P(i, j)) + \sum_{x' \in d \setminus \Set{x}} p_{ix'} \sum_{(i, j) \in E_s} p_{jx'} \log(1 - P(i, j)) \\
    &= \Tilde{f}_{2}(p) + \sum_{x' \in d \setminus \Set{x}} p_{ix'} \sum_{(i, j) \in E_s} (p_{jx'} -p_{jx})\log(1 - P(i, j)).
    \end{align*}    
\end{proof}

%% file: A02bkgrd.tex
\section{Additional Details on the Background}\label{app:background}

We would like to provide some additional details on the background (Section~\ref{subsec:background_related_work}).

\subsection{On the ``Differentiable Optimization'' in the Pipeline (Section~\ref{subsubsec:EGN_pipeline})}\label{subapp:background:diff_optim}
One can directly optimize a probabilistic decision $p$ on each test instance $G_{\mathrm{test}}$, i.e., aim to find $p^* \approx \argmin_p \Tilde{f}(p; G_{\mathrm{test}})$.
One can also train an encoder (e.g., a graph neural network) parameterized by parameters $\theta$ on a training set $\mathcal{D_{\mathrm{train}}}$ to learn to output ``good'' (probabilistic) decisions for each training instance, i.e., aim to find $\theta^* \approx \argmin_\theta \sum_{G \in \mathcal{D}_{\mathrm{train}}} \Tilde{f}(\operatorname{ENCODER}(G; {\theta}); G)$. Such a trained encoder can be applied to each test instance $G_{\mathrm{test}}$ and output a (probabilistic) decision $p = \operatorname{ENCODER}(G_{\mathrm{test}}; {\theta})$.
Training such an encoder is optional, but if trained well, it can save time for unseen cases since we do not need to optimize $p$ for each test instance from scratch.\footnote{See some related discussions at \url{https://github.com/Stalence/erdos_neu}.}
Even when using such an encoder, one can still further directly optimize the probabilistic decisions on each test instance.
See more discussions on inductive settings and transductive settings in Appendix~\ref{subapp:discussions:inductive_and_transductive}.

\subsection{Formal Theoretical Results in the Existing Works}
Here, we would like to provide the detailed formal theoretical results in the existing works by~\cite{karalias2020erdos} and \cite{wang2022unsupervised}.
Recall that \cite{karalias2020erdos} showed a quality guarantee by \textit{random sampling}.

\begin{theorem}[Theorem 1 by \cite{karalias2020erdos}]\label{thm:egn_orig_qlt_grt}
    Assume that $f$ is non-negative.\footnote{We can always ensure this for any bounded $f$ by adding a sufficiently large positive constant to $f$.}
    Fix any $\beta > \max_{X \in \calC} f(X; G)$, $\epsilon > 0$, and $t \in (0, 1]$ such that $(1 - t)\epsilon < \beta$.
    For each $p \in [0, 1]^n$,
    if $\Tilde{f}(p; G) < \beta$, then 
    $\Pr_{X \sim p}[f(X; G) < \epsilon \land X \in \calC] \geq t$.
\end{theorem}

Recall that \cite{wang2022unsupervised} further proposed \textit{iterative rounding}.
Also, recall the following definitions:
given a probability decision $p \in [0,1]^n$, an index $i \in [n]$, and $x \in \Set{0, 1}$,
let $\operatorname{der}(i, x; p)$ denoted the 
result after the $i$-th entry of $p$ being \textit{locally derandomized} as $x$.
Formally,
$\operatorname{der}(i, x; p)_i = x$, and
$\operatorname{der}(i, x; p)_j = p_j, \forall j \neq i$.
A probabilistic objective $\Tilde{f}$ is \textit{entry-wise concave} if $p_i \Tilde{f}(\operatorname{der}(i, 1; p);G) + (1 - p_i) \Tilde{f}(\operatorname{der}(i, 0; p);G) \leq \Tilde{f}(p;G), \forall G, p, i$.
\begin{theorem}[Theorem 1 by \cite{wang2022unsupervised}]\label{thm:qlt_grt_wang_entrywise_concave}
    If $\Tilde{f}(p) \geq \bbE_{X \sim p} f(X) + \beta \Pr_{X \sim p}[X \notin \calC], \forall p$ and $\tilde{f}$ is entry-wise concave and non-negative with $\beta > \max(\Tilde{f}(p_{\mathrm{init}}), \max_{X \in \calC} f(X))$, then 
    for any permutation $\pi: [n] \to [n]$,
    starting from $p_{\mathrm{cur}} = p_{\mathrm{init}}$ and for $i \in [n]$ doing
    (1) $x^* \gets \argmin_{x \in \Set{0, 1}} \Tilde{f}(\operatorname{der}(\pi(i), x; p_{\mathrm{cur}}))$ and
    (2) $p_{\mathrm{cur}} \gets \operatorname{der}(i, x^*; p_{\mathrm{cur}})$
    will finally give a discrete $p_{\mathrm{final}} \in \calC$ such that    
    $f(p_{\mathrm{final}}) < \Tilde{f}(p_{\mathrm{init}})$.
\end{theorem}

\subsection{\revise{Prevalent} Conditions in Existing Works}\label{subapp:conditions_wrong}

As mentioned in Section~\ref{sec:analy_conds}, several conditions have been encountered in existing works.
Here, for each condition analyzed in Section~\ref{sec:analy_conds}, we shall discuss how the existing works try to handle it.

\smallsection{Cardinality constraints.}
\cite{wang2022towards} specifically considered cardinality constraints.
However, they used optimal transport soft top-$k$ instead of the probabilistic-method UL4CO we focus on in this work.
Also, our derivation is more general since it can handle general cardinality constraints other than choosing a specific number of entities (i.e., top-$k$).
\cite{wang2022towards} claimed that cardinality constraints cannot be handled in the EGN pipeline, but this work shows that cardinality constraints can actually be properly handled by our derivations.
\revise{\cite{karalias2020erdos} used iterative re-scaling to impose cardinality constraints. However, the operation involves clamping which may cause gradient vanishing and it is only guaranteed that the summation of the probabilities is within the desired range (i.e., cardinality constraints). However, this does not mean the whole distribution represented by the probabilities is within the desired range.\footnote{For example, if there are $n = 10$ nodes and we want to choose $k = 2$ nodes. After re-scaling we might get probabilities $(\frac{1}{5},\frac{1}{5},\ldots,\frac{1}{5})$ with summation exactly equal to $2$, but $\Pr[\text{exactly $2$ nodes are chosen}] = \binom{10}{2} \times (\frac{1}{5})^2 \times (\frac{4}{5})^8 \approx 0.302$ is far lower than $1$.}}

\smallsection{Minimum (maximum) w.r.t. a subset.}
\cite{wang2022towards} also encountered such a condition in the facility location problem which they considered.
They used the $\operatorname{softmin}$ to approximate the $\min$ operation, which indeed provides an upper bound.
However, the result of $\operatorname{softmin}$ is not entry-wise concave, and thus fails to satisfy the good property required by~\cite{wang2022unsupervised}, while our derivation satisfies all the good properties.

\smallsection{Covering.}
\cite{wang2022towards} also encountered such a condition in the maximum coverage problem which they considered.
They used $\min(1, \sum_{v \in N_i} p_v)$ as an approximation for the probability of $i$ being covered, where $N_i = \Set{v \colon (v, i) \in E}$.
In other words, they used
$\max(0, 1 -\sum_{v \in N_i} p_v))$ to approximate the probability that
$i$ is not covered.
As we have shown, the probability that
$i$ is not covered is exactly
$\prod_{v \in N_i} (1 - p_v)$.
However, 
$\max(0, 1 -\sum_{v \in N_i} p_v))$ is not an \textit{upper bound} of $\prod_{v \in N_i} (1 - p_v)$ but a \textit{lower bound}.
Therefore, the derivation by~\cite{wang2022towards} does not satisfy the conditions required for the probabilistic-method UL4CO pipeline and thus does not satisfy the good properties.

\smallsection{Cliques (or independent sets).}
\cite{karalias2020erdos} also considered the maximum clique problem, while our high-level targets provide insights into interpreting the derivation.
Our derivation of incremental differences is novel, and we also showed how we can extend this to non-binary cases.

\smallsection{Other problems.}
Recently, UL4CO on the traveling salesman problem (TSP) has also been considered~\citep{gaile2022unsupervised,min2023unsupervised}, but their derivation does not satisfy the conditions required for the probabilistic-method UL4CO pipeline (see Section~\ref{subsubsec:EGN_pipeline}).
\revise{We see the potential application of probabilistic-method UL4CO on TSP by seeing the conditions in TSP as a combination of (1) non-binary decisions and (2) cardinality constraints, both of which are already covered in this work. Specifically, if we aim to put $n$ nodes in a cycle as the solution, then this can be understood as (1) deciding a position $X_v \in \{0, 1, \ldots, n-1\}$ for each node $v \in [n]$ such that (2) each position contains exactly one node.
See similar ideas in the (integer) linear programming formulations of TSP~\citep{diaby2006traveling,yannakakis1988expressing}.
}

%% file: A03tcncl.tex
\section{Additional technical details}\label{app:technical_details}
Here, we provide some additional technical details that are omitted in the main text.

\subsection{Computation of the Poisson Binomial Distribution}\label{subapp:technical_details:fourier_PB}
Here, we provide some implementation details on the computation of the Poisson binomial distribution, which is used in Section~\ref{subsec:analy_conds:card}.
We mainly follow the original paper~\citep{hong2013computing} and an existing implementation online~\citep{Straka2017}.

The main formula is
\[\Pr_{X \sim p}[\Abs{V_X} = t] = \frac{1}{n+1} \sum_{s = 0}^n \exp(-\mathbf{i}\omega s t) \prod_{j = 1}^n (1 - p_j + p_j \exp(\mathbf{i}\omega s)),\]
where $\mathbf{i} = \sqrt{-1}$ and $\omega =  \frac{2\pi}{n + 1}$.
See the original paper~\citep{hong2013computing} for more technical details.

%% file: A07theor.tex
\section{Additional Theoretical Results}\label{app:theory}
Here, we provide additional theoretical results.

\subsection{Additional Results on Non-Binary Decisions}\label{subapp:theory:non_binary}
Here, we provide the details of our theoretical results regarding non-binary decisions.

\smallsection{Notations.}
{
With non-binary decisions $d = \Set{0, 1, \ldots, c - 1}$, 
we use $p \in [0, 1]^{n \times c}$ with $\sum_{r = 0}^{c-1} p_{ir} = 1, \forall i \in [n]$ to represent the probabilities of possible decisions,
where each $p_{ir} = \Pr[X_i = r]$.
Now, $\der(i, x; p)$ is the result after the $i$-th row of $p$ being locally derandomized w.r.t. its $x$-th entry, i.e., 
$\begin{cases}
 \der(i, x; p)_{ix} = 1, \\
 \der(i, y; p)_{iy} = 0, \forall y \neq x, \text{~and~} \\ 
 \der(i, x; p)_{jz} = p_{jz}, \forall j \neq i, \forall z.
\end{cases}$

\smallsection{Theoretical analysis on non-binary cases.}
Our theoretical results (Thms.~\ref{thm:concave_exp_prob} \& \ref{thm:greedy_like_der_good}) can be extended to non-binary cases.\footnote{See App.~\ref{subapp:theory:non_binary} for the detailed statements, where we also extend the theoretical results in the existing works by \cite{karalias2020erdos} and \cite{wang2022unsupervised} to non-binary cases.} 
With non-binary decisions, a probabilistic objective $\Tilde{f}: [0, 1]^{n \times c} \to \bbR$ is
\textit{entry-wise concave} if 
\begin{center}
$\sum_{r \in d} p_{ir} \Tilde{f}(\der(i, r; p)) \leq \Tilde{f}(p), \forall p \in [0, 1]^{n \times c}, i \in [n]$,
\end{center}
and the process of greedy derandomization is:\\
$\begin{cases}
\text{(1) $(i^*, x^*) \gets \argmin_{(i, x) \in [n] \times d} \Tilde{f}(\operatorname{der}(i, x; p_{\mathrm{cur}}))$ and} \\
\text{(2) $p_{\mathrm{cur}} \gets \operatorname{der}(i^*, x^*; p_{\mathrm{cur}})$.}
\end{cases}$

\begin{theorem}[Expectations are all you need (non-binary version)]\label{thm:concave_exp_prob_nonbin}
    For any function $g: d^n \to \bbR$,
    $\tilde{g}: [0, 1]^{n \times c} \to \bbR$ with 
    $\Tilde{g}(p) = \bbE_{X \sim p} g(X)$ is differentiable and entry-wise concave, where $\bbE_{X \sim p} g(X) = \sum_{X \in d^n} \Pr_p[X] g(X)$ with $\Pr_p[X]=
    \prod_{v \in [n]} p_{vX_v}$.
\end{theorem}
\begin{proof}
    For any $p$ and $i$, we have 
    \begin{align*}
     \Tilde{g}(p)
    &=\bbE_{X \sim p} g(X)\\
    &=\sum_{X \in d^n} \Pr_{p}[X] g(X)\\
    &=\sum_{X \in d^n} \prod_{v \in [n]} p_{vX_v} g(X)\\
    &=\sum_{X \in d^n} (\prod_{v \in [n] \setminus \Set{i}} p_{vX_v}) p_{iX_i} g(X)\\  
    &=\sum_{r \in d} \sum_{X \colon X_i = r} (\prod_{v \in [n] \setminus \Set{i}} p_{vX_v}) p_{iX_i} g(X)\\
    &=\sum_{r \in d} \sum_{X \colon X_i = r} (\prod_{v \in [n] \setminus \Set{i}} p_{vX_v}) p_{ir} g(X)\\
    &=\sum_{r \in d} p_{ir} \sum_{X \colon X_i = r}  (\prod_{v \in [n] \setminus \Set{i}} p_{vX_v}) g(X)\\
    &=\sum_{r \in d} p_{ir} \sum_{X}  (\prod_{v \in [n] \setminus \Set{i}} p_{vX_v}) \mathbb{1}(X_i = r) g(X)\\
    &=\sum_{r \in d} p_{ir} \Tilde{g}(\der(i, r; p))\\ 
    &\geq \sum_{r \in d} p_{ir} \Tilde{g}(\der(i, r; p)),
    \end{align*}
    completing the proof on entry-wise concavity.   
    Regarding differentiability, 
    since $\bbE_{X \sim p} g(X) = 
    \sum_{X \in d^n} \Pr_{p}[X] g(X)$,
    it suffices to show that
    $\Pr_{p}[X] g(X) = \sum_{X \in d^n} \prod_{v \in [n]} p_{vX_v} g(X)$ is differentiable w.r.t $p$ for each $X \in \Set{0,1}^n$.
    Indeed, fix any $X$, $\sum_{X \in d^n} \prod_{v \in [n]} p_{vX_v} g(X)$ is a polynomial of $p_{ir}$'s,
    and is thus differentiable.
\end{proof}

With non-binary decisions, the process of greedy derandomization is extended as follows:\\
$\begin{cases}
\text{(1) $(i^*, x^*) \gets \argmin_{(i, x) \in [n] \times d} \Tilde{f}(\operatorname{der}(i, x; p_{\mathrm{cur}}))$ and} \\
\text{(2) $p_{\mathrm{cur}} \gets \operatorname{der}(i^*, x^*; p_{\mathrm{cur}})$.}
\end{cases}$
\begin{theorem}[Goodness of greedy derandomization (non-binary version)]\label{thm:greedy_like_der_good_nonbin}           Theorem~\ref{thm:greedy_like_der_good} still holds in non-binary cases, i.e., with $\Set{0, 1}$ being replaced by any non-binary $d$.
    Specifically, for any entry-wise concave $\Tilde{f}$ and $p_{\mathrm{init}}$, the above process can always reach a point where the final $p_{\mathrm{final}}$ is (1) discrete (i.e., $p_{\mathrm{final}} \in d^n$),
    (2) no-worse than $p_{\mathrm{init}}$ (i.e., $\Tilde{f}(p_{\mathrm{final}}) \leq \Tilde{f}(p_{\mathrm{init}})$),
    and (3) is a local minimum (i.e., $\Tilde{f}(p_{\mathrm{final}}) = \min_{(i, x) \in [n] \times d} \Tilde{f}(\operatorname{der}(\pi(i), x; p_{\mathrm{final}}))$).
\end{theorem}
\begin{proof}
    See the proof for Theorem~\ref{thm:greedy_like_der_good}.
    It is easy to see that the reasoning still holds with $\Set{0,1}$ being replaced by any non-binary $d$.
\end{proof}

We also extend the theoretical results in the existing works~\citep{karalias2020erdos,wang2022unsupervised} to non-binary cases.

Recall the theoretical results (Theorem~\ref{thm:egn_orig_qlt_grt}) by~\cite{karalias2020erdos}.

\textbf{Theorem~\ref{thm:egn_orig_qlt_grt}} (Theorem 1 by \cite{karalias2020erdos})
    Assume that $f$ is non-negative.
    Fix any $\beta > \max_{X \in \calC} f(X; G)$, $\epsilon > 0$, and $t \in (0, 1]$ such that $(1 - t)\epsilon < \beta$.
    If $\Tilde{f}(p_{\mathrm{init}}; G) < \beta$, then 
    $\Pr_{X \sim p_{\mathrm{init}}}[f(X; G) < \epsilon \land X \in \calC] \geq t$.

We extend Theorem~\ref{thm:egn_orig_qlt_grt} to non-binary cases.

\begin{theorem}[Non-binary extension of Theorem~\ref{thm:egn_orig_qlt_grt}]\label{thm:egn_orig_qlt_grt_nonbin}
    Assume that $f$ is non-negative.
    Fix any $\beta > \max_{X \in \calC} f(X; G)$, $\epsilon > 0$, and $t \in (0, 1]$ such that $(1 - t)\epsilon < \beta$.
    If $\Tilde{f}(p_{\mathrm{init}}; G) < \beta$, then 
    $\Pr_{X \sim p_{\mathrm{init}}}[f(X; G) < \epsilon \land X \in \calC] \geq t$.
\end{theorem}
\begin{proof}
    We shall follow the main idea in the original proof of Theorem~\ref{thm:egn_orig_qlt_grt} by~\cite{karalias2020erdos}, which is based on Markov's inequality.
    The key point is that the reasoning still holds when the decisions are non-binary.
    Specifically, we can define a probabilistic penalty function
    $\hat{f}(X; G) = {f}(X; G) + \beta \mathbb{1}(X \in \calC)$.
    Since $\beta > \max_{X \in \calC} f(X; G)$, we have $\hat{f}(X; G) < \epsilon$ if and only if $f(X; G) < \epsilon$ and $X \in \calC$.
    Therefore, using Markov's inequality, we have
    \begin{align*}
    \Pr\nolimits_{X \sim p_{\mathrm{init}}}[(f(X; G) < \epsilon) \land (X \in \calC)]
    &= \Pr\nolimits_{X \sim p_{\mathrm{init}}}[\hat{f}(X; G) < \epsilon]  \\
    &> 1 - \frac{1}{\epsilon} \bbE_{X \sim p_{\mathrm{init}}}[\hat{f}(X; G)] \\
    &= 1 - \frac{1}{\epsilon} \bbE_{X \sim p_{\mathrm{init}}}[{f}(X; G) + \beta \mathbb{1}(X \in \calC)] \\
    &> 1 - \frac{1}{\epsilon} (\beta) \\
    &> t.   
    \end{align*}    
\end{proof}

Recall the theoretical results (Theorem~\ref{thm:qlt_grt_wang_entrywise_concave}) by~\cite{wang2022unsupervised}.

\textbf{Theorem~\ref{thm:qlt_grt_wang_entrywise_concave}} (Theorem 1 by \cite{wang2022unsupervised})
    If $\Tilde{f}(p) \geq \bbE_{X \sim p} f(X) + \beta \Pr_{X \sim p}[X \notin \calC], \forall p$ is entry-wise concave and non-negative with $\beta > \max(\Tilde{f}(p_{\mathrm{init}}), \max_{X \in \calC} f(X))$, then 
    for any permutation $\pi: [n] \to [n]$,
    starting from $p_{\mathrm{cur}} = p_{\mathrm{init}}$ and for $i \in [n]$ doing
    (1) $x^* \gets \argmin_{x \in \Set{0, 1}} \Tilde{f}(\operatorname{der}(\pi(i), x; p_{\mathrm{cur}}))$ and
    (2) $p_{\mathrm{cur}} \gets \operatorname{der}(i, x^*; p_{\mathrm{cur}})$
    will finally give a discrete $p_{\mathrm{final}} \in \calC$ such that 
    $f(p_{\mathrm{final}}) \leq \Tilde{f}(p_{\mathrm{init}})$.

We shall show that Theorem~\ref{thm:qlt_grt_wang_entrywise_concave} can be extended to non-binary cases.

\begin{theorem}[Non-binary extension of Theorem~\ref{thm:qlt_grt_wang_entrywise_concave}]\label{thm:qlt_grt_wang_entrywise_concave_nonbin}
    If $\Tilde{f}(p) \geq \bbE_{X \sim p} f(X) + \beta \Pr_{X \sim p}[X \notin \calC], \forall p$ is entry-wise concave and non-negative with $\beta > \max(\Tilde{f}(p_{\mathrm{init}}), \max_{X \in \calC} f(X))$, then 
    for any permutation $\pi: [n] \to [n]$,
    starting from $p_{\mathrm{cur}} = p_{\mathrm{init}}$ and for $i \in [n]$ doing
    (1) $x^* \gets \argmin_{x \in d = \Set{0, 1, 2,\ldots, c - 1}} \Tilde{f}(\operatorname{der}(\pi(i), x; p_{\mathrm{cur}}))$ and
    (2) $p_{\mathrm{cur}} \gets \operatorname{der}(i, x^*; p_{\mathrm{cur}})$
    will finally give a discrete $p_{\mathrm{final}} \in \calC$ such that 
    $f(p_{\mathrm{final}}) \leq \Tilde{f}(p_{\mathrm{init}})$.
\end{theorem}
\begin{proof}
    We shall follow the main idea in the original proof of Theorem~\ref{thm:qlt_grt_wang_entrywise_concave} by~\cite{wang2022unsupervised}, where the key idea was that
    entry-wise concavity ensures that local derandomization does not increase the objective.
    This key idea still holds with non-binary decisions.
    First, since after the series of local derandomization, for each $i$, it is locally derandomized exactly once, the final derandomized result should be discrete.
    Regarding $p_{\mathrm{final}} \in \calC$ and
    $f(p_{\mathrm{final}}) \leq \Tilde{f}(p_{\mathrm{init}})$,
    we claim that ``local derandomization does not increase the objective''.
    Specifically,
    since $\Tilde{f}$ is entry-wise concave, i.e.,    
    \[\sum_{r \in d} p_{ir} \Tilde{f}(\der(i, r; p); G) \leq \Tilde{f}(p; G), \forall G, p, i,\]
    and $\sum_{r \in d} p_{ir} = 1$,
    we have
    \[\min_{r \in d} \Tilde{f}(\der(i, r; p); G) \leq \sum_{r \in d} p_{ir} \Tilde{f}(\der(i, r; p); G) \leq \Tilde{f}(p; G), \forall G, p, i.\]
    Hence, indeed, ``local derandomization does not increase the objective'', and the final
    \[{f}(X; G) + \beta \mathbb{1}(X \notin \calC) \leq \Tilde{f}(p_{\mathrm{init}}) < \beta,\]
    which implies that
    ${f}(X; G) \leq \Tilde{f}(p_{\mathrm{init}})$
    and $\mathbb{1}(X \notin \calC) = 0$, i.e.,
    $X \in \calC$, completing the proof.
\end{proof}

%% file: A04prblm.tex
\section{Additional Problems}\label{app:more_problems}

The robust $k$-clique problem generalizes the maximum $k$-clique problem~\citep{bomze1999maximum} and it can be seen as an uncertain variant of the heaviest $k$-subgraph problem~\citep{feige2001dense,billionnet2005different}.

\subsection{Robust $k$-Clique}
\smallsection{Definition.}
Given 
(1) an uncertain graph $G = (V, E, P)$, and
(2) $k \in \bbN$,
we aim to find a subset of nodes $V_X \subseteq V$ such that
(c1) $\Abs{V_X} = k$,
(c2) $V_X$ forms a clique,
and 
(c3) $\Pr[\text{all the edges between nodes in $V_X$ exist}]$ is maximized.

\smallsection{Involved conditions:}
(1) cardinality constraints, (2) cliques,  and (3) uncertainty 
(see Sections \ref{subsec:analy_conds:card}, \ref{subsec:analy_conds:clique} \& \ref{subsec:analy_conds:uncertainty}).

\smallsection{Details.}
Regarding conditions (c1)-(c2), we can directly use 
the derivations for them.
Regarding condition (c3), fix any $V_X$, the probability that all the edges between nodes in $V_X$ exist is
\[\prod_{(u, v) \in \binom{V_c}{2} \cap E} P_{uv}.\]
Maximizing the probability is equivalent to minimizing 
\[f_1(X) \coloneqq -\sum_{(u, v) \in \binom{V_c}{2} \cap E} \log P_{uv}.\]
We let $\hat{f}_1(X) \coloneqq f_1(X)$
and let \[\Tilde{f}_1(p) \coloneqq \bbE_{X \sim p} \hat{f}_1(X) = -\sum_{(u, v) \in E} p_u p_v \log P_{uv}.\]
The final objective is
\[\Tilde{f_{\mathrm{RQ}}}(p) = \Tilde{f}_1(p) + \beta_1 \Tilde{f}_{\mathrm{cq}}(p) + \beta_2 \Tilde{f}_{\mathrm{card}}(p; \Set{k})\] with constraint coefficients $\beta_1, \beta_2 > 0$.

Regarding the incremental differences,
we only need to derive the incremental differences of $\Tilde{f}_1$, which is
\[\Delta \Tilde{f}_1(i,1,p) = 
(p_i - 1) \sum_{v \colon (i, v) \in E} p_v \log P_{iv},\]
and
\[\Delta \Tilde{f}_1(i,0,p) = 
- p_i \sum_{v \colon (i, v) \in E} p_v \log P_{iv}.\]

\subsection{Robust Dominating Set}
The robust dominating set problem generalizes the minimal dominating set problem~\cite{guha1998approximation} and can also be seen as an uncertain version of set covering~\cite{caprara2000algorithms}.

\smallsection{Definition.}
Given (1) an uncertain graph $G = (V, E, P)$, and
(2) $k \in \bbN$, we aim to find a subset of nodes $V_X \subseteq V$ such that
(c1) $\Abs{V_X} = k$,
(c2) $V_X$ is a dominating set in the underlying deterministic graph, that is, for each $v \in V$, either $v \in V_X$ or $v$ has a neighbor in $V_X$, and
(c3) the probability that $V_X$ is indeed a dominating set when considering the edge uncertainty, i.e. $\Pr[\bigwedge_{v \in V \setminus V_X} \bigvee_{u \in V_X} A_{uv} ]$ is maximized.
For each edge $(u, v) \in E$, $A_uv$ is the event that $(u, v)$ exists under edge certainty, which happens with probability $P_{uv}$.

\smallsection{Involved conditions:}
(1) cardinality constraints,
(2) covering, and
(3) uncertainty
(see Sections~\ref{subsec:analy_conds:card}, \ref{subsec:analy_conds:cover}, \& \ref{subsec:analy_conds:uncertainty}).

\smallsection{Details.}
Regarding conditions (c1), we can directly use the derivations for it.
Specifically, $\Tilde{f}_1(p) = \Tilde{f}_{\mathrm{card}}(p; \Set{k})$.

Conditions (c2) and (c3) can be combined together.
We first add self-loops on each node $v \in V$ (so that each node $v$ can cover $v$ itself), and then consider the condition as $X \in \calC$ with 
\[\calC = \Set{X \colon \text{each node $v \in V$ is covered}}.\]
Then we define 
$\hat{f}_2(X)$ as the expected number of nodes that are not covered (when taking the edge uncertain into consideration).
It is easy to see that $\hat{f}_2(X) \geq \mathbb{1}(X \notin \calC), \forall X \in \Set{0,1}^n$.
Note that here the uncertainty comes from the edge probabilities while the decisions are discrete.
The formula of $\hat{f}_2$ is
\[\hat{f}_2(X) = \sum_{i \in V} \Pr[\text{$i$ is not covered}] = \sum_{i \in V \setminus V_X} \prod_{v \in N_i} (1 - P_{iv}),\]
where $N_i = \Set{v \in V \colon (i, v) \in E}$ is the neighborhood of $i$.
We then define $\Tilde{f}_2(p) = \bbE_{X \sim p} \hat{f}_2(X)$, and its formula is
\[\Tilde{f}_2(p) 
= \sum_{i \in V} \Pr[i \notin V_X] \prod_{v \in N_i} (1 - P_{iv}) 
= \sum_{i \in V} (1 - p_i) \prod_{v \in N_i} (1 - P_{iv}).\]
Combining all the conditions,
the final probabilistic objective is
\[\Tilde{f}_{\mathrm{RDS}}(p) = \Tilde{f}_2(p) + \beta \Tilde{f}_1(p)\] with constraint coefficient $\beta > 0$.

Regarding the incremental differences,
we only need to derive the incremental differences of $\Tilde{f}_2$, which is
\[\Delta \Tilde{f}_2(i, 1, p) = 
(p_i - 1) \prod_{v \in N_i} (1 - P_{iv})\]
and 
\[\Delta \Tilde{f}_2(i, 0, p) = 
- p_i \prod_{v \in N_i} (1 - P_{iv}).\]

\subsection{Clique Cover}
The clique cover problem~\citep{gramm2009data} is a classical NP-hard combinatorial problem.
We consider its decision version, which is NP-complete.

\smallsection{Definition.}
Given (1) a graph $G = (V, E)$ and 
(2) $c \in \bbN$, we aim to 
partition the nodes into $c$ groups, such that each group forms a clique.

\smallsection{Involved conditions:}
(1) cliques and (2) non-binary decisions (see Sections~\ref{subsec:analy_conds:clique} \& \ref{subsec:analy_conds:coloring}).

\smallsection{Details.}
This is basically the non-binary extension of the ``cliques'' condition.
For each $r \in d = \Set{0, 1, 2, \ldots, c - 1}$,
the condition holds for group-$r$ if the group is either empty or forms a clique.
The group-$r$ is empty with probability
$\prod_{i \in V} (1 - p_{ir})$,
and we can use 
\[\Tilde{f}_{\mathrm{cq}}(p_{\cdot, r}) \geq \Pr\nolimits_{X \sim p}[\text{group-$r$ does not form a clique}],\]
where ${p}_{\cdot, r} \in [0, 1]^n$ with 
$({p}_{\cdot, r})_j = p_{j, r}$.
Then the violation probability
\begin{align*}
 \Pr[\text{violation}] &= \Pr[\text{not empty} \land \text{does not form a clique}] \\
 &\leq \Pr[\text{not empty}] + \Pr[\text{does not form a clique}].   
\end{align*}
Therefore, we can have the final probabilistic objective
\[
\Tilde{f}_{\mathrm{cc}}(p) = \sum_{r = 0}^{c - 1} 1 - \prod_{i \in V} (1 - p_{ir}) + \Tilde{f}_{\mathrm{cq}}(p_{\cdot, r}).\]
If we create a complete graph $K_V$ with self-loops on $V$, then 
\[\prod_{i \in V} (1 - p_{ir})= \Tilde{f}_{\mathrm{cv}}(p_{\cdot, r}; v, K_V)\]
for any $v \in V$.
Hence, we have
\[\Tilde{f}_{\mathrm{CC}}(p) = \sum_{r = 0}^{c - 1} 1 - \Tilde{f}_{\mathrm{cv}}(p_{\cdot, r}; v, K_V) + \Tilde{f}_{\mathrm{cq}}(p_{\cdot, r}),\]
and the incremental differences can be handled by those of $\Tilde{f}_{\mathrm{cv}}$ and $\Tilde{f}_{\mathrm{cq}}$.

\subsection{Minimum Spanning Tree}
The minimum spanning tree problem~\citep{graham1985history} is a classical combinatorial problem.
Notably, it is not theoretically difficult and we have fast algorithms~\citep{pettie2002optimal,zhong2015fast} for the problem.
But it is still interesting to see that our method can be applied to such a problem.

\smallsection{Definition.}
Given
a graph $G = (V, E, W)$, 
we aim to find a subset of edges to form a connected tree (i.e., without cycles) containing all the nodes such that the total edge weights in the tree are minimized.
Instead of considering choosing edges, we consider the decisions on nodes.
Specifically, we put the nodes into different layers.
Let $c \leq n$ be the number of layers,
it is a non-binary problem, where each node $v$ is put into layer-$X_v$ with $X_v \in d = \Set{0, 1, 2, \ldots, c - 1}$.
For each node $v_\ell$ in layer $\ell > 0$,
it would be connected to a parent $v_{prev}$ in the previous layer-$(\ell - 1)$ so that the edge weight of $(v_\ell, v_{prev})$ is minimized.
The conditions are:
(c1) each node is either in layer-$0$, or it can find a parent in the previous layer, and
(c2) the total edge weights are minimized.

\smallsection{Involved conditions:}
(1) minimum (maximum) w.r.t. a subset, (2) covering and (3) non-binary decisions (see Sections~\ref{subsec:analy_conds:opt_wrt_subset}, \ref{subsec:analy_conds:cover}, and \ref{subsec:analy_conds:coloring}).

\smallsection{Details.}
Regarding (c1), we let $\hat{f}_1$ be the number of nodes for which (c1) is violated.
For each node $i$, it is in layer-$0$ with probability $p_{v0}$ and it can find at least one parent with probability
\begin{align*}
  \sum_{\ell = 1}^{c - 1} \Pr[\text{$i$ is in layer-$\ell$}] \Pr[\text{at least one of $i$'s neighbors is in layer-$(\ell - 1)$}]
&= \sum_{\ell = 1}^{c - 1} p_{i\ell} (1 - \prod_{v \in N_i} (1 - p_{v, \ell - 1})) \\
&= \sum_{\ell = 1}^{c - 1} p_{i\ell} (1 - \Tilde{f}_{\mathrm{cv}}(p{\cdot, \ell - 1}; i)),  
\end{align*}
where ${p}_{\cdot, \ell - 1} \in [0, 1]^n$ with 
$({p}_{\cdot, \ell - 1})_j = p_{j, \ell - 1}$.
Again, $N_i = \Set{v \in V \colon (i, v) \in E}$ is the neighborhood of $i$.
Note how the idea of ``covering'' is used here.
Therefore, the probability that (c1) is violated for the node $i$ is
\[1 - p_{i0} - \sum_{\ell = 1}^{c - 1} p_{i\ell} (1 - \Tilde{f}_{\mathrm{cv}}(p{\cdot, \ell - 1}; i)).\]
Now we are ready to compute
\[\tilde{f}_1(p) = \bbE_{X \sim p} \hat{f}_1(X) = \sum_{i \in V} (1 - p_{v0} - \sum_{\ell = 1}^{c - 1} p_{v\ell} (1 - \Tilde{f}_{\mathrm{cv}}(p{\cdot, \ell - 1}; i))).\]

Regarding (c2), we use the idea of ``minimum (maximum) w.r.t. a subset''.
For a spanning tree, the total edge weights are
\[\sum_{i \in V \colon \text{$i$ not the root}} W(i, \text{the parent of $i$}).\]
Note that in a minimum spanning tree, each non-root node should have a single parent.
For each node $i$, the expected $W(i, \text{the parent of $i$})$ is
\[\sum_{\ell = 1}^{c - 1} p_{i\ell} \Tilde{f}_{\mathrm{ms}}({p}_{\dot, \ell - 1}; i, W),\]
where ${p}_{\cdot, \ell - 1} \in [0, 1]^n$ with 
$({p}_{\cdot, \ell - 1})_j = p_{j, \ell - 1}$.
The idea of ``minimum (maximum) w.r.t. a subset'' has been used, where we consider the nodes being chosen into layer-$(\ell - 1)$.
Therefore, we have
\[\Tilde{f}_2(p) = \sum_{i \in V} \sum_{\ell = 1}^{c - 1} p_{i\ell} \Tilde{f}_{\mathrm{ms}}({p}_{\cdot, \ell - 1}; i, W).\]
Combining the conditions, the final probabilistic objective is
\[\Tilde{f}_{\mathrm{MST}}(p) = \Tilde{f}_2(p) + \beta \Tilde{f}_1(p)\]
with constraint coefficient $\beta > 0$.
The incremental differences can be handled by those of
$\Tilde{f}_{\mathrm{cv}}$ and 
$\Tilde{f}_{\mathrm{ms}}$.

\subsection{On cycles and trees}\label{subapp:cycles_trees}

\revise{\smallsection{Cycles.}
As discussed in Appendix~\ref{subapp:conditions_wrong}, CO problems involving cycles can be handled as follows.
The conditions that nodes should form a cycle can be seen as a combination of (1) non-binary decisions and (2) cardinality constraints. 
Specifically, if we aim to put $n$ nodes in a cycle, then this can be understood as (1) deciding a position $X_v \in \{0, 1, \ldots, n-1\}$ for each node $v \in [n]$ such that (2) each position contains exactly one node.}

\revise{
\smallsection{Trees.}
In MST (and other CO problems involving trees), an implicit condition is \textit{acyclicity}~\citep{lachapelle2019gradient}. In our way of organizing nodes into sequences of layers, acyclicity is naturally satisfied.
However, this might be tricky if we also need to decide how each node chooses its parent(s) and child(ren). For MST, this is deterministic in the sense that each non-root node should always choose the closest node in the above layer as its only parent, so that the total distance is minimized. In general, we may need additional decisions (parameters) for the choice of edges.
}

\revise{We acknowledge that we do not have in-depth empirical results for problems on cycles and trees (e.g., TSP and MST) in this work. However, many advanced heuristics are available for TSP, and there are fast exact algorithms for MST. Based on our preliminary experiments, we suspect that a general framework like probabilistic-method-based UL4CO (at least in its current stage) cannot be empirically comparable to them, even with our proposed schemes. Hence, from a practical standpoint, we found it less prioritized to develop new methods for such problems, which was also why we focused on the conditions and problems in this work. Note that we do not intend to imply that constraints for TSP and MST are less important. Instead, we suspect that addressing TSP and MST effectively enough to be practical requires sophisticated and potentially complex designs tailored specifically for such problems, which is beyond the scope of this work.
The further exploration on problems involving cycles and trees (and other conditions that cannot be trivially covered using the derivations in this work) is one of our future directions.}

%% file: A05reslt.tex
\section{Complete Experimental Settings and Results}\label{app:additional_results}
Here, we provide detailed experimental settings and some additional experimental results.

\subsection{Detailed Experimental Settings}\label{subapp:exp_settings}
Here, we provide some details of the experimental settings.

\subsubsection{Hardware}
All the experiments are run on a machine with 
two Intel Xeon\textsuperscript{\textregistered}   Silver 4210R (10 cores, 20 threads) processors,
a 256GB RAM,
and RTX2080Ti (11GB) GPUs.
For the methods using GPUs, a single GPU is used.

\subsubsection{Facility Location}\label{subsubapp:additional_results:settings:FL}
Here, we provide more details about the settings of the experiments on the facility location problem.
For the experiments on facility location and maximum coverage, we mainly follow the settings by~\cite{wang2022towards} and use their open-source implementation.\footnote{\url{https://github.com/Thinklab-SJTU/One-Shot-Cardinality-NN-Solver}}

\smallsection{Datasets.}
We consider both random synthetic graphs and real-world graphs:
\begin{itemize}[leftmargin=*]
    \item \textbf{Rand500:} We follow the way of generating random graphs by~\cite{wang2022towards}.
    We generate 100 random graphs, where each graph contains 500 nodes.
    Each node $v$ has a two-dimensional location $(x_v, y_v)$, where $x_v$ and $y_v$ are sampled in $[0, 1]$, independently, uniformly at random. 
    \item \textbf{Rand800:} The rand800 graphs are generated in a similar way. The only difference is that each rand800 graph contains 800 nodes.
    \item \textbf{Starbucks:} The Starbucks datasets were used by~\cite{wang2022towards}.
    We quote their descriptions as follows:
    \say{The datasets are built based on the project named Starbucks Location Worldwide 2021 version,\footnote{\url{https://www.kaggle.com/datasets/kukuroo3/starbucks-locations-worldwide-2021-version}} which is scraped from the open-accessible Starbucks store locator webpage.\footnote{\url{https://www.starbucks.com/store-locator}}
    We analyze and select 4 cities with more than 100 Starbucks stores, which are
    London (166 stores), New York City (260 stores), Shanghai (510 stores), and Seoul (569 stores).
    The locations considered are the real locations represented as latitude and longitude.}
    \item \textbf{MCD:} The MCD (McDonald's) dataset is available online.\footnote{\url{https://www.kaggle.com/datasets/mdmdata/mcdonalds-locations-united-states}}. The dataset contains the locations of MCD branches in the United States. We divide the dataset into multiple sub-datasets by state, where each sub-dataset contains branches in the same state.
    We use the data from 8 states with the most ranches:
    CA (1248 branches), TX (1155 branches), FL (889 branches), NY (597 branches), PA (483 branches), IL (650 branches), OH (578 branches), and GA (442 branches).
    \item \textbf{Subway:} The Subway dataset is available online.\footnote{\url{https://www.kaggle.com/datasets/thedevastator/subway-the-fastest-growing-franchise-in-the-worl}} Similar to the MCD dataset, it contains the locations of subway branches in the United States.
    We also divide the dataset into multiple sub-datasets by state, where each sub-dataset contains branches in the same state.
    We use the data from 8 states with the most ranches:
    CA (2590 branches), TX (21994 branches), FL (1490 branches), NY (1066 branches), PA (865 branches), IL (1110 branches), OH (1171 branches), and GA (852 branches).
    \item For the real-world datasets, we use min-max normalization to make sure that each coordinate of each node (location) is also in $[0, 1]$ as in the random graphs.
\end{itemize}

\smallsection{Inductive settings.}
We follow the settings by~\cite{wang2022towards}.
For random graphs, the model is trained and tested on random graphs from the same distribution, but the training set and the test set are disjoint.
For real-world graphs, the model is trained on the \textit{rand500} graphs.

\smallsection{Methods.}
We consider both traditional methods and machine-learning methods:
\begin{itemize}[leftmargin=*]
    \item \textbf{Random:} Among all the locations, $k$ locations are picked uniformly at random; 240 seconds are given on each test graph.
    \item \textbf{Greedy:} deterministic greedy algorithms. We use the implementation of~\cite{wang2022towards}.
    \item \textbf{Gurobi}~\citep{gurobi} and \textbf{SCIP}~\citep{bestuzheva2021scip,ortools}: The problems are formulated as MIPs and the two solvers are used; the time budget is set as 120 seconds, but the programs sometimes do not terminate until more time is used.
    \item \textbf{{CardNN}}~\citep{wang2022towards}: Three variants proposed in the original paper. We use the implementation of the original authors.
    \item \textbf{{CardNN}-noTTO}: In addition to training, {CardNN} also directly optimizes on each test graph in test time, and this is a variant of {CardNN} without test-time optimization. We use the implementation of the original authors. 
    \item \textbf{EGN-naive}: EGN~\citep{karalias2020erdos} with a naive objective construction and iterative rounding, which was used by~\cite{wang2022towards} as a baseline method. We use the derivation and implementation by~\cite{wang2022towards}.
    \item \textbf{RL}: A reinforcement-learning method~\citep{kool2018attention}. We adapt the implementation by~\cite{berto2023rl4co}.\footnote{\url{https://github.com/kaist-silab/rl4co}}
\end{itemize}

\smallsection{Speed-quality trade-offs.}
For the proposed method \ours, we use test-time augmentation~\citep{jin2022empowering} on the test graphs by adding perturbations into both graph topology and features to obtain additional data.
Specifically, we use edge dropout~\citep{papp2021dropgnn,shu2022understanding} and add Gaussian noise into features.
The noise scale and the edge dropout ratios are both 0.2, which we do not fine-tune.
The three variants of \ours are obtained by using different numbers of additional augmented data and taking the best objective.
Specifically, 
the ``short'' version uses only the original test graphs,
the ``middle'' version uses less time than CardNN-GS,
and the ``long'' version uses less time than CardNN-HGS.

\smallsection{Evaluation.}
Given locations $(x_v, y_v)$'s for the nodes $v \in V$, if the final selected $k$ nodes are $v_1, v_2, \ldots, v_k$, the final objective is
$\sum_{v \in V} \min_{i \in [k]} \operatorname{dist}(v_i, v)$, where the distance metric $\operatorname{dist}$ is the Euclidean squared distance used by~\cite{wang2022towards}.
We choose $k = 30$ locations in each graph, except for the rand800 graphs where we choose $k = 50$ locations.

\smallsection{Hyperparameter fine-tuning.}
For the proposed method \ours and the method CardNN by~\cite{wang2022towards}, we conduct hyperparameter fine-tuning.
For \ours, we fine-tune the learning rate (LR) and constraint coefficient (CC).
For CardNN, we fine-tune the training learning rate (LR)\footnote{CardNN uses (possibly) different learning rates for training and test-time optimization.} and the Gumbel noise scale $\sigma$.
For random graphs, we choose the best hyperparameter setting w.r.t. the objective on the training set, because the distribution of the training set and the distribution of the test set are the same.
For real-world graphs, we choose the smallest graph in each group of datasets as the validation graph, and we choose the best hyperparameter setting w.r.t. the objective on the validation graph.
There is no specific reason to choose the smallest, and we just want to have a deterministic way to choose validation graphs.

We make sure that the number of candidate combinations (which is 15) is the same for both methods.
Our hyperparameter search space is as follows:
\begin{itemize}[leftmargin=*]
    \item For \ours: 
    $\text{LR} \in \Set{1e-1, 1e-2, 1e-3, 1e-4, 1e-5}$ and
    $\text{CC} \in \Set{1e-1, 1e-2, 1e-3}$
    \item For CardNN:
    $\text{LR} \in \Set{1e-1, 1e-2, 1e-3, 1e-4, 1e-5}$ and $\sigma \in \Set{0.01, 0.15, 0.25}$    
\end{itemize}
Notably, after our fine-tuning, the performance of CardNN is at least the same and usually better than the performance using the hyperparameter settings in the open-source code of CardNN provided by the original authors.
The best hyperparameter settings for each dataset are:
\begin{itemize}[leftmargin=*]
    \item Rand500: 
    \begin{itemize}
        \item \ours: $\text{LR} = 1e-1$, $\text{CC} = 1e-1$
        \item CardNN: $\text{LR} = 1e-4$, $\sigma = 0.25$
    \end{itemize}
    \item Rand800: 
    \begin{itemize}
        \item \ours: $\text{LR} = 1e-2$, $\text{CC} = 1e-2$
        \item CardNN: $\text{LR} = 1e-4$, $\sigma = 0.25$
    \end{itemize}
    \item Starbucks: 
    \begin{itemize}
        \item \ours: $\text{LR} = 1e-1$, $\text{CC} = 1e-1$
        \item CardNN: $\text{LR} = 1e-4$, $\sigma = 0.15$
    \end{itemize}
    \item MCD:
    \begin{itemize}
        \item \ours: $\text{LR} = 1e-3$, $\text{CC} = 1e-1$
        \item CardNN: $\text{LR} = 1e-5$, $\sigma = 0.25$
    \end{itemize}
    \item Subway:
    \begin{itemize}
        \item \ours: $\text{LR} = 1e-1$, $\text{CC} = 1e-1$
        \item CardNN: $\text{LR} = 1e-5$, $\sigma = 0.01$
    \end{itemize}
\end{itemize}

\subsubsection{Maximum Coverage}\label{subsubapp:additional_results:settings:MC}
Here, we provide more details about the settings of the experiments on the maximum coverage problem.

\smallsection{Datasets.}
We consider both random synthetic graphs and real-world graphs:
\begin{itemize}[leftmargin=*]
    \item \textbf{Rand500:} We follow the way of generating random graphs by~\cite{wang2022towards}.
    Each item has a random weight chosen uniformly at random between $1$ and $100$.
    Each set contains a random number of items, and the number of items is chosen uniformly at random between $10$ and $30$.
    Each rand500 dataset contains 500 sets and 1000 items.
    \item \textbf{Rand1000:} The rand1000 graphs are generated in a similar way. The only difference is that each rand1000 dataset contains 1000 sets and 2000 items.
    \item \textbf{Twitch:} The Twitch datasets were used by~\cite{wang2022towards}. We quote their descriptions as follows:
    \say{This social network dataset is collected by~\cite{rozemberczki2021multi} and the edges represent the mutual friendships between streamers. The streamers are categorized by their streaming language, resulting in 6 social networks for 6 languages. The social networks are DE (9498 nodes), ENGB (7126 nodes), ES (4648 nodes), FR (6549 nodes), PTBR (1912 nodes), and RU (4385 nodes). The objective is to cover more viewers, measured by the sum of the logarithmic number of viewers. We took the logarithm to enforce diversity because those top streamers usually have the dominant number of viewers.}
    \item \textbf{Railway:} The railway datasets~\citep{ceria1998lagrangian} are available online.\footnote{\url{https://plato.asu.edu/ftp/lptestset/rail}.} The data were collected from real-world crew membership in Italian railways.
    We have three datasets:        
    (1) rail507 with 507 sets and 63009 items,
    (2) rail516 with 516 sets and 47311 items, and
    (3) rail582 with 582 sets and 55515 items.
\end{itemize}

\smallsection{Inductive settings.}
We follow the settings by~\cite{wang2022towards}.
For random graphs, the model is trained and tested on random graphs from the same distribution, but the training set and the test set are disjoint.
For real-world graphs, the model is trained on the \textit{rand500} graphs.

\smallsection{Methods.}
See the method descriptions above for the facility location problem in Appendix~\ref{subsubapp:additional_results:settings:FL}.

\smallsection{Speed-quality trade-offs.}
See the descriptions above for the facility location problem in Appendix~\ref{subsubapp:additional_results:settings:FL}.

\smallsection{Evaluation.}
Let $w_j$'s denote the weights of the items.
The final objective is the summation of the weights of the covered items.
An item $j$ is covered if at least one set containing $j$ is chosen.
This is the term $\sum_{j \in T_X} W_j$ in Section~\ref{subsec:problems:max_cover}.

\smallsection{Hyperparameter fine-tuning.}
The overall fine-tuning principles are the same as in the experiments on the facility location problem. See Appendix~\ref{subsubapp:additional_results:settings:FL}.

Our hyperparameter search space is as follows:
\begin{itemize}[leftmargin=*]
    \item For \ours: 
    $\text{LR} \in \Set{1e-1, 1e-2, 1e-3, 1e-4, 1e-5}$ and
    $\text{CC} \in \Set{10, 100, 500}$
    \item For CardNN:
    $\text{LR} \in \Set{1e-1, 1e-2, 1e-3, 1e-4, 1e-5}$ and $\sigma \in \Set{0.01, 0.15, 0.25}$    
\end{itemize}

The best hyperparameter settings for each dataset are:
\begin{itemize}[leftmargin=*]
    \item Rand500: 
    \begin{itemize}
        \item \ours: $\text{LR} = 1e-5$, $\text{CC} = 500$
        \item CardNN: $\text{LR} = 1e-5$, $\sigma = 0.15$
    \end{itemize}
    \item Rand1000: 
    \begin{itemize}
        \item \ours: $\text{LR} = 1e-5$, $\text{CC} = 500$
        \item CardNN: $\text{LR} = 1e-5$, $\sigma = 0.15$
    \end{itemize}
    \item Twitch: 
    \begin{itemize}
        \item \ours: $\text{LR} = 1e-1$, $\text{CC} = 10$
        \item CardNN: $\text{LR} = 1e-4$, $\sigma = 0.01$
    \end{itemize}    
    \item Railway:
    \begin{itemize}
        \item \ours: $\text{LR} = 1e-5$, $\text{CC} = 10$
        \item CardNN: $\text{LR} = 1e-5$, $\sigma = 0.15$
    \end{itemize}
\end{itemize}

\subsubsection{Robust Coloring}
Here, we provide more details about the settings of the experiments on the robust coloring problem.

\smallsection{Datasets.}
We use four real-world uncertain graphs~\citep{hu2017embedding,ceccarello2017clustering,chen2019embedding}.
They are available online.\footnote{\url{https://github.com/Cecca/ugraph/tree/master/Reproducibility/Data}; \url{https://github.com/stasl0217/UKGE/tree/master/data}}
Some basic statistics of the datasets are as follows:
\begin{itemize}[leftmargin=*]
    \item \textbf{Collins:} $n = 1004$ nodes and $m = 8323$ edges; a deterministic greedy coloring algorithm uses 18 colors for the hard conflicts, and 36 colors for all the conflicts.
    \item \textbf{Gavin:} $n = 1727$ nodes and $m = 7534$ edges; a deterministic greedy coloring algorithm uses 7 colors for the hard conflicts, and 16 for all the conflicts.
    \item \textbf{Krogan:} $n = 2559$ nodes $m = 7031$ edges; 
    a deterministic greedy coloring algorithm uses 8 colors for the hard conflicts, and 25 for all the conflicts.
    \item \textbf{PPI:} $n = 1912$ nodes $m = 22749$ edges;
    a deterministic greedy coloring algorithm uses 47 colors for the hard conflicts, and 53 for all the conflicts.
\end{itemize}
We take the largest connected component of each dataset.
For each dataset, the $20\%$ edges with the highest edge weights are chosen as the hard conflicts.

\smallsection{Methods.}
We consider four baseline methods:
\begin{itemize}[leftmargin=*]
    \item \textbf{Greedy-RD:} The method first samples a random permutation of nodes, and then following the permutation, for each node, greedily chooses the best coloring to (1) avoid all the hard conflicts and 
    (2) optimizes the objective; 300 seconds are given on each test graph.
    \item \textbf{Greedy-GA:} This is the method proposed by~\cite{yanez2003robust} in the original paper of robust coloring. The difference between greedy-RD and greedy-GA is that greedy-GA uses a genetic algorithm (GA) to learn a good permutation instead of randomly sampling permutations; in the GA algorithm, the number of iterations is 20, the population size is 20, the crossover probability is 0.6, the mutation probability is 0.1, the elite ratio is 0.01, the parents proportion is 0.3.
    \item \textbf{Deterministic coloring (DC):} a deterministic greedy coloring algorithm~\citep{kosowski2004classical} is used to satisfy all the hard conflicts,
    and the soft conflicts are included in different random orders until no more soft conflicts can be satisfied. The maximum possible number of soft conflicts that can be included is found by binary search; 300 seconds are given on each test graph.    
    \item \textbf{Gurobi}: the problem is formulated as an MIP and the solver is used; 300 seconds are given on each test graph.
\end{itemize}

\smallsection{Hyperparameters.}
For \ours, we do not fine-tune hyperparameters. We consistently use learning rate $\eta = 0.1$ and the constraint coefficient $\beta$ is set as the highest penalty on soft conflicts, i.e., $\max_{e = (u, v) \in E_s} \log (1 - P(e))$.

\smallsection{Speed-quality trade-offs.}
We record the running time of our method using only CPUs and using GPUs.
For our method, we start from multiple random initial probabilities (each entry is sampled uniformly at random in $[0, 1]$), while
making sure that even with only CPUs, our method uses less time than each baseline.

\smallsection{Evaluation.}
The recorded objective is the negative log-likelihood of no soft conflicts being violated, i.e., the function $f_2$ in Section~\ref{subsec:problems:robust_coloring}.

\subsection{Full Results}\label{subapp:full_res}
Here, we provide the full raw results on each problem, together with the standard deviations of the results obtained by five random independent trials.

In Table~\ref{tab:results_facility_location_full}, we provide the full raw results with standard deviations on the facility location problem.

In Table~\ref{tab:results_max_cover_full}, we provide the full raw results with standard deviations on the maximum coverage problem.

\input{TAB/res_full_fl}
\input{TAB/res_full_mc}

\subsection{Ablation Studies}\label{subapp:ablation_study}
Here, we provide the results of ablation studies.

\subsubsection{Q1: Are Good Probabilistic Objectives Helpful?}
Here, we check whether the probabilistic objectives derived by us are helpful.
We compare 
(a) EGN-naive (non-good objectives and iterative rounding) and 
(b) \ours-iterative (good objectives and iterative rounding).
It is difficult to compare the full-fledged version of \ours 
(good objectives and greedy derandomization) and 
a variant with 
non-good objectives and greedy derandomization,
because we find computing the incremental differences of non-good objectives nontrivial (yet less meaningful).

In Tables~\ref{tab:ablation_study:obj_fl} and \ref{tab:ablation_study:obj_mc}, we show the performance of EGN-naive and \ours-iterative on facility location and maximum coverage.
We observe that in most cases, the optimization objective with the good objectives is better.
However, we also observe that using the good objectives, the running time is sometimes higher.
This is because the good objective of cardinality constraints proposed by us is mathematically more complicated than the one used in EGN-naive formulated by~\cite{wang2022towards}, \revise{which is $\max(\sum_v p_v - k, 0)$ for the cardinality constraint that at most $k$ nodes are chosen. See also Appendix~\ref{subapp:conditions_wrong}.}
This also validates the necessity of our fast incremental derandomization scheme, which can improve the speed.

\input{TAB/ablation_obj_fl}
\input{TAB/ablation_obj_mc}

\subsubsection{Q2: Is Greedy Derandomization Better than Iterative Rounding?}
Here, we check whether the proposed greedy derandomization is helpful, especially when compared to the iterative rounding proposed by~\cite{wang2022unsupervised}.
We compare 
(a) the full-fledged version of \ours 
(good objectives and greedy derandomization) and 
(b) \ours-iterative 
(good objectives and iterative rounding).
In Tables~\ref{tab:ablation_study:grd_derand_fl} and \ref{tab:ablation_study:grd_derand_mc}, we show the performance of \ours and \ours-iterative on facility location and maximum coverage.

We observe that when using (incremental) greedy derandomization (compared to iterative rounding), \ours archives better optimization objectives within a shorter time, validating that the greedy derandomization scheme proposed by us is indeed helpful.

\input{TAB/ablation_grd_derand_fl}
\input{TAB/ablation_grd_derand_mc}

In conclusion, each component in \ours is helpful in most cases, but only when combining both good objectives with greedy derandomization can we obtain the best synergy.

\subsubsection{Q3: Does Incremental Derandomization Improve the Speed?}
Here, we want to check how much the proposed incremental derandomization scheme using incremental differences helps in improving the speed.
With greedy derandomization, we compare the running time of incremental derandomization and naive derandomization (i.e., evaluating the objective on each possible local derandomization case), on facility location and maximum coverage.

In Tables~\ref{tab:ablation_study:incre_fl} and \ref{tab:ablation_study:incre_mc}, we show the running time of \ours when using incremental derandomization and when using naive derandomization, on facility location and maximum coverage.

We observe that using incremental derandomization significantly improves the derandomization speed, and the superiority is usually more significant when the dataset sizes increase.

\input{TAB/ablation_incre_fl}
\input{TAB/ablation_incre_mc}

\subsubsection{Q4: How Does \ours Perform with Different Constraint Coefficients?}
Here, we want to check how \ours performs when using different constraint coefficients (i.e., different $\beta$ values) and fixing the other hyperparameters.

In Tables~\ref{tab:ablation_study:beta_fl} to \ref{tab:ablation_study:beta_rc}, we show the performance of \ours when using different $\beta$ values, on facility location, maximum coverage, and robust coloring.

For facility location and maximum coverage, the candidate $\beta$ values are the same as in Appendices~\ref{subsubapp:additional_results:settings:FL} and \ref{subsubapp:additional_results:settings:MC}.
We use the fastest version of \ours without test-time augmentation.

For robust coloring, let the originally used $\beta_0 \coloneqq \max_{e \in E_s} \log(1 - P(e))$, we consider three candidate values: $\frac{1}{2}\beta_0$, $\beta_0$, and $2\beta_0$.
The other hyperparameters are fixed as the same.

Our observations are as follows.
For facility location and maximum coverage:
\begin{itemize}
    \item For random graphs, since the distribution of the training set and the distribution of the test set are the same, the originally used $\beta$ values perform well, usually the best among the candidates.
    \item For real-world graphs, the originally used $\beta$ values do not achieve the best performance in some cases. In our understanding, this is because we use the smallest graph in each group of datasets as the validation graph, while the smallest graph possibly has a slightly different data distribution from the other graphs in the group, i.e., the test set.
    \item Overall, certain sensitivity w.r.t $\beta$ can be observed, but usually, multiple $\beta$ values can achieve reasonable performance.
\end{itemize}
For robust coloring:
\begin{itemize}
    \item Overall, all the candidates $\beta$ vales can achieve similar performance.
    \item In other words, the performance of our method is not very sensitive to the value of $\beta$ on robust coloring.
\end{itemize}

\input{TAB/ablation_beta_fl}
\input{TAB/ablation_beta_mc}
\input{TAB/ablation_beta_rc}

%% file: TAB/res_full_fl.tex
\begin{table}[t!]
\caption{Full raw results on facility location with the standard deviations.
Running time (time): smaller the better.
Objective (obj): smaller the better.
} \label{tab:results_facility_location_full}
\small
\centering
\scalebox{0.9}{
\begin{tabular}{l|r|r|r|r|r|r|r|r|r|r}
\hline
\multirow{2}[3]{*}{method} & \multicolumn{2}{c|}{rand500} & \multicolumn{2}{c|}{rand800} & \multicolumn{2}{c|}{starbucks} & \multicolumn{2}{c|}{mcd} & \multicolumn{2}{c}{subway} \bigstrut\\
\cline{2-11}      
& obj$\downarrow$   & time$\downarrow$
& obj$\downarrow$   & time$\downarrow$
& obj$\downarrow$   & time$\downarrow$
& obj$\downarrow$   & time$\downarrow$
& obj$\downarrow$   & time$\downarrow$
\bigstrut\\
\hline
random & 3.43  & 240.00 & 3.48  & 240.00 & 0.54  & 240.00 & 1.54  & 240.00 & 2.72  & 240.00 \bigstrut[t] \\
(std) & 0.006 & 0.000 & 0.011 & 0.000 & 0.014 & 0.000 & 0.029 & 0.000 & 0.025 & 0.000 \\
greedy & 2.85  & 2.10  & 2.67  & 5.88  & 0.35  & 6.51  & 1.12  & 11.51 & 1.99  & 26.00 \\
(std) & 0.000 & 0.012 & 0.000 & 0.025 & 0.000 & 0.032 & 0.000 & 0.054 & 0.000 & 0.115 \\
Gurobi & 2.56  & 121.86 & 2.92  & 125.04 & 0.31  & 102.48 & 1.42  & 125.20 & 4.71  & 138.85 \\
(std) & 0.009 & 0.028 & 0.019 & 0.197 & 0.013 & 2.328 & 0.110 & 0.044 & 0.633 & 0.248 \\
SCIP  & 4.16  & 94.39 & 5.43  & 191.64 & 5.73  & 80.51 & 51.79 & 485.91 & 98.47 & 736.60 \\
(std) & 0.012 & 0.289 & 0.000 & 1.234 & 2.722 & 2.511 & 0.000 & 2.755 & 0.000 & 15.188 \\
CardNN-S & 2.74  & 13.94 & 2.46  & 16.13 & 0.47  & 19.23 & 1.09  & 23.50 & 1.93  & 20.38 \\
(std) & 0.006 & 0.320 & 0.003 & 0.758 & 0.020 & 4.722 & 0.008 & 1.979 & 0.015 & 0.580 \\
CardNN-GS & 2.41  & 71.45 & 2.34  & 141.76 & 0.31  & 39.88 & 1.08  & 42.34 & 1.85  & 30.12 \\
(std) & 0.002 & 0.906 & 0.002 & 0.713 & 0.004 & 1.153 & 0.014 & 4.045 & 0.018 & 0.788 \\
CardNN-HGS & 2.41  & 100.40 & 2.34  & 181.66 & 0.31  & 90.93 & 1.08  & 96.44 & 1.83  & 57.25 \\
(std) & 0.001 & 1.474 & 0.001 & 0.849 & 0.005 & 4.547 & 0.024 & 4.519 & 0.015 & 3.947 \\
CardNN-noTTO-S & 3.44  & 2.03  & 3.57  & 2.01  & 0.97  & 2.03  & 3.67  & 1.96  & 6.33  & 2.01 \\
(std) & 0.067 & 0.020 & 0.041 & 0.032 & 0.168 & 0.217 & 0.227 & 0.254 & 0.310 & 0.016 \\
CardNN-noTTO-GS & 2.74  & 28.61 & 2.66  & 51.92 & 0.44  & 8.18  & 1.18  & 15.73 & 2.20  & 4.45 \\
(std) & 0.009 & 0.332 & 0.008 & 1.038 & 0.018 & 0.093 & 0.035 & 1.171 & 0.077 & 0.381 \\
CardNN-noTTO-HGS & 2.74  & 37.35 & 2.65  & 61.69 & 0.42  & 15.59 & 1.19  & 28.31 & 2.17  & 7.37 \\
(std) & 0.012 & 2.305 & 0.009 & 0.350 & 0.010 & 0.042 & 0.024 & 0.363 & 0.069 & 0.039 \\
EGN-naive & 2.65  & 78.80 & 2.63  & 85.30 & 0.33  & 120.87 & 1.56  & 48.08 & 2.63  & 120.87 \\
(std) & 0.254 & 9.846 & 0.134 & 0.139 & 0.020 & 8.679 & 0.128 & 8.655 & 0.777 & 1.182 \\
RL-transductive & 5.57  & 300.00 & 5.18  & 300.00 & 2.97  & 1800.00 & 2.60  & 1800.00 & 4.50  & 1800.00 \\
(std) & 0.356 & 0.000 & 0.362 & 0.000 & 0.245 & 0.000 & 0.261 & 0.000 & 0.415 & 0.000 \\
RL-inductive & 4.07  & 300.06 & 4.27  & 300.54 & 0.79  & 300.04 & 2.40  & 300.04 & 4.23  & 300.05 \\
(std) & 0.227 & 0.019 & 0.143 & 0.157 & 0.148 & 0.014 & 0.241 & 0.016 & 0.395 & 0.015 \bigstrut[b]\\
\hline
\hline
\ours-short & 2.51  & 0.91  & 2.38  & 1.91  & 0.30  & 0.52  & 0.99  & 2.56  & 1.86  & 3.56 \bigstrut[t] \\
(std) & 0.076 & 0.084 & 0.005 & 0.034 & 0.003 & 0.365 & 0.020 & 0.154 & 0.063 & 1.600 \\
\ours-middle & 2.41  & 29.68 & 2.31  & 29.90 & 0.30  & 2.26  & 0.95  & 8.77  & 1.80  & 26.23 \\
(std) & 0.065 & 1.755 & 0.002 & 1.388 & 0.004 & 1.371 & 0.009 & 0.207 & 0.068 & 3.913 \\
\ours-long & 2.40  & 73.86 & 2.31  & 59.43 & 0.29  & 10.54 & 0.94  & 38.04 & 1.79  & 45.99 \\
(std) & 0.065 & 4.383 & 0.002 & 3.894 & 0.005 & 6.114 & 0.004 & 0.985 & 0.078 & 6.808 \bigstrut[b]\\
\hline
\end{tabular}%
}
\end{table}

%% file: TAB/res_full_mc.tex
\begin{table}[t!]
\caption{Full raw results on maximum coverage with the standard deviations.
Running time (time): smaller the better.
Objective (obj): larger the better.
} \label{tab:results_max_cover_full}
\small
\centering
\scalebox{0.9}{
\begin{tabular}{l|r|r|r|r|r|r|r|r}
\hline
\multirow{2}[3]{*}{method} & \multicolumn{2}{c|}{rand500} & \multicolumn{2}{c|}{rand1000} & \multicolumn{2}{c|}{twitch} & \multicolumn{2}{c}{railway} \bigstrut \\
\cline{2-9}      
& obj$\uparrow$   & time$\downarrow$ 
& obj$\uparrow$   & time$\downarrow$ 
& obj$\uparrow$   & time$\downarrow$ 
& obj$\uparrow$   & time$\downarrow$
\bigstrut \\
\hline
random & 36874.94 & 240.00 & 70756.03 & 240.00 & 17756.52 & 240.00 & 7333.90 & 240.00 \bigstrut[t] \\
(std) & 22.534 & 0.000 & 24.965 & 0.000 & 213.655 & 0.000 & 3.774 & 0.000 \\
greedy & 44312.81 & 0.09  & 88698.89 & 0.33  & 33822.40 & 0.69  & 7603.00 & 0.76 \\
(std) & 0.000 & 0.005 & 0.000 & 0.006 & 0.000 & 0.012 & 0.000 & 0.015 \\
Gurobi & 44880.59 & 120.05 & 89636.32 & 120.10 & 33840.40 & 0.65  & 7586.40 & 120.74 \\
(std) & 7.132 & 0.001 & 22.677 & 0.004 & 0.000 & 0.012 & 3.878 & 0.017 \\
SCIP  & 43805.35 & 120.07 & 86274.66 & 119.49 & 33840.40 & 3.28  & 7585.50 & 121.48 \\
(std) & 4.420 & 0.002 & 0.000 & 0.052 & 0.000 & 0.004 & 0.000 & 0.025 \\
CardNN-S & 42037.90 & 11.73 & 83434.44 & 11.86 & 33836.16 & 7.96  & 7397.10 & 2.82 \\
(std) & 75.443 & 0.465 & 143.252 & 0.808 & 1.541 & 0.253 & 8.834 & 0.369 \\
CardNN-GS & 44737.28 & 40.33 & 89313.37 & 55.95 & 33840.08 & 16.50 & 7616.70 & 17.64 \\
(std) & 7.150 & 0.430 & 42.942 & 0.070 & 0.640 & 0.495 & 4.411 & 1.221 \\
CardNN-HGS & 44742.53 & 55.64 & 89330.76 & 81.95 & 33840.32 & 30.73 & 7619.90 & 27.25 \\
(std) & 5.967 & 0.972 & 36.536 & 0.142 & 0.160 & 1.513 & 4.923 & 4.501 \\
CardNN-noTTO-S & 31283.61 & 1.83  & 62120.08 & 2.04  & 246.12 & 0.93  & 7148.10 & 1.17 \\
(std) & 3431.866 & 0.149 & 6841.483 & 0.334 & 322.740 & 0.081 & 0.490 & 0.020 \\
CardNN-noTTO-GS & 37010.18 & 10.40 & 71171.64 & 20.19 & 844.52 & 1.93  & 7324.40 & 5.72 \\
(std) & 171.453 & 0.915 & 464.520 & 0.293 & 722.439 & 0.134 & 5.571 & 0.413 \\
CardNN-noTTO-HGS & 37012.94 & 11.93 & 71198.82 & 24.80 & 6594.08 & 2.35  & 7324.70 & 9.23 \\
(std) & 173.545 & 0.635 & 437.536 & 0.348 & 4753.903 & 0.056 & 6.508 & 0.723 \\
EGN-naive & 41259.13 & 120.11 & 81689.73 & 120.27 & 4425.52 & 120.39 & 7376.60 & 120.94 \\
(std) & 187.784 & 0.005 & 325.560 & 0.032 & 5480.695 & 0.435 & 7.303 & 0.420 \\
RL-transductive & 41461.65 & 300.00 & 73597.20 & 300.00 & 32143.20 & 1800.00 & 7307.00 & 1800.00 \\
(std) & 580.240 & 0.000 & 957.157 & 0.000 & 315.444 & 0.000 & 53.139 & 0.000 \\
RL-inductive & 34536.00 & 300.06 & 69155.00 & 300.18 & 19840.60 & 301.91 & 7320.50 & 301.78 \\
(std) & 22.154 & 0.017 & 42.216 & 0.022 & 55.146 & 0.254 & 9.578 & 0.195 \bigstrut[b] \\
\hline
\hline
ours-short & 44622.79 & 0.96  & 89130.65 & 1.83  & 33828.40 & 1.82  & 7607.10 & 2.00 \bigstrut[t] \\
(std) & 9.158 & 0.106 & 33.987 & 0.116 & 0.000 & 0.145 & 4.329 & 0.055 \\
ours-middle & 44971.97 & 15.16 & 89496.45 & 7.84  & 33828.40 & 11.43 & 7617.80 & 16.05 \\
(std) & 16.313 & 1.057 & 29.199 & 0.117 & 0.000 & 0.301 & 4.007 & 0.135 \\
ours-long & 45000.41 & 30.02 & 89721.94 & 73.54 & 33828.40 & 19.35 & 7620.50 & 16.09 \\
(std) & 16.399 & 2.097 & 23.553 & 1.007 & 0.000 & 0.578 & 5.206 & 0.118 \bigstrut[b]\\
\hline
\end{tabular}%
}
\end{table}

%% file: TAB/ablation_obj_fl.tex
\begin{table}[t!]
\caption{Ablation study on facility location: are good probabilistic objectives helpful?
Running time (time): smaller the better.
Objective (obj): smaller the better.
} \label{tab:ablation_study:obj_fl}
\small
\centering
\scalebox{0.9}{
\begin{tabular}{l|r|r|r|r|r|r|r|r|r|r}
\hline
\multirow{2}[3]{*}{method} & \multicolumn{2}{c|}{rand500} & \multicolumn{2}{c|}{rand800} & \multicolumn{2}{c|}{starbucks} & \multicolumn{2}{c|}{mcd} & \multicolumn{2}{c}{subway} \bigstrut\\
\cline{2-11}      
& obj$\downarrow$   & time$\downarrow$
& obj$\downarrow$   & time$\downarrow$
& obj$\downarrow$   & time$\downarrow$
& obj$\downarrow$   & time$\downarrow$
& obj$\downarrow$   & time$\downarrow$
\bigstrut\\
\hline
EGN-naive & 2.65  & 78.80 & 2.63  & 85.30 & 0.33  & 120.87 & 1.56  & 48.08 & 2.63  & 120.87 \bigstrut[t]\\
\ours-iterative & 2.65  & 169.56 & 2.67  & 205.12 & 0.33  & 162.15 & 1.05  & 87.45 & 1.96  & 52.37 \bigstrut[b]\\
\hline
\end{tabular}%
}
\end{table}

%% file: TAB/ablation_obj_mc.tex
\begin{table}[t!]
\caption{Ablation study on maximum coverage: are good probabilistic objectives helpful?
Running time (time): smaller the better.
Objective (obj): larger the better.
} \label{tab:ablation_study:obj_mc}
\small
\centering
\scalebox{0.9}{
\begin{tabular}{l|r|r|r|r|r|r|r|r}
\hline
\multirow{2}[3]{*}{method} & \multicolumn{2}{c|}{rand500} & \multicolumn{2}{c|}
{rand1000} & \multicolumn{2}{c|}{twitch} & \multicolumn{2}{c}{railway} \bigstrut\\
\cline{2-9}      
& obj$\uparrow$   & time$\downarrow$
& obj$\uparrow$   & time$\downarrow$
& obj$\uparrow$   & time$\downarrow$
& obj$\uparrow$   & time$\downarrow$
\bigstrut\\
\hline
EGN-naive & 41378.43 & 120.72 & 81393.77 & 101.23 & 15448.20 & 120.72 & 7290.00 & 120.76 \bigstrut[t]\\
\ours-iterative & 42820.04 & 131.78 & 84397.79 & 209.95 & 16093.80 & 137.08 & 7304.50 & 120.21 \bigstrut[b]\\
\hline
\end{tabular}%
}
\end{table}

%% file: TAB/ablation_grd_derand_fl.tex
\begin{table}[t!]
\caption{Ablation study on facility location: is greedy derandomization better than iterative rounding?
Running time (time): smaller the better.
Objective (obj): smaller the better.
} \label{tab:ablation_study:grd_derand_fl}
\small
\centering
\scalebox{0.9}{
\begin{tabular}{l|r|r|r|r|r|r|r|r|r|r}
\hline
\multirow{2}[3]{*}{method} & \multicolumn{2}{c|}{rand500} & \multicolumn{2}{c|}{rand800} & \multicolumn{2}{c|}{starbucks} & \multicolumn{2}{c|}{mcd} & \multicolumn{2}{c}{subway} \bigstrut\\
\cline{2-11}      
& obj$\downarrow$   & time$\downarrow$
& obj$\downarrow$   & time$\downarrow$
& obj$\downarrow$   & time$\downarrow$
& obj$\downarrow$   & time$\downarrow$
& obj$\downarrow$   & time$\downarrow$
\bigstrut\\
\hline
\ours-iterative & 2.65  & 169.56 & 2.67  & 205.12 & 0.33  & 162.15 & 1.05  & 87.45 & 1.96  & 52.37 \bigstrut\\
\hline
\hline
\ours-short & 2.51  & 0.91  & 2.38  & 1.91  & 0.30  & 0.52  & 0.99  & 2.56  & 1.86  & 10.35 \bigstrut[t]\\
\ours-middle & 2.41  & 29.68 & 2.31  & 29.90 & 0.30  & 2.26  & 0.95  & 8.77 & 1.80  & 26.23 \\
\ours-long & 2.40  & 73.86 & 2.31  & 59.43 & 0.29  & 10.54 & 0.94  & 38.04 & 1.79  & 45.99 \bigstrut[b]\\
\hline
\end{tabular}%
}
\end{table}

%% file: TAB/ablation_grd_derand_mc.tex
\begin{table}[t!]
\caption{Ablation study on maximum coverage: is greedy derandomization better than iterative rounding?
Running time (time): smaller the better.
Objective (obj): larger the better.
} \label{tab:ablation_study:grd_derand_mc}
\small
\centering
\scalebox{0.9}{
\begin{tabular}{l|r|r|r|r|r|r|r|r}
\hline
\multirow{2}[3]{*}{method} & \multicolumn{2}{c|}{rand500} & \multicolumn{2}{c|}{rand1000} & \multicolumn{2}{c|}{twitch} & \multicolumn{2}{c}{railway} \bigstrut \\
\cline{2-9}      
& obj$\uparrow$   & time$\downarrow$
& obj$\uparrow$   & time$\downarrow$
& obj$\uparrow$   & time$\downarrow$
& obj$\uparrow$   & time$\downarrow$
\bigstrut\\
\hline
\ours-iterative & 42820.04 & 131.78 & 84397.79 & 209.95 & 16093.80 & 137.08 & 7304.50 & 120.21 \bigstrut\\
\hline
\hline
\ours-short & 44622.80 & 0.96  & 89130.70 & 1.83  & 33828.40 & 1.82  & 7607.10 & 2.00 \bigstrut[t]\\
\ours-middle & 44972.00 & 15.16 & 89496.50 & 7.84  & 33828.40 & 11.43 & 7616.00 & 8.17 \\
\ours-long & 45000.40 & 30.02 & 89721.90 & 73.54 & 33828.40 & 19.35 & 7620.50 & 16.04 \bigstrut[b]\\
\hline
\end{tabular}%
}
\end{table}

%% file: TAB/ablation_incre_fl.tex
\begin{table}[t!]
\caption{Ablation study on facility location: does incremental derandomization improve the speed?
} \label{tab:ablation_study:incre_fl}
\small
\centering
\scalebox{0.9}{
\begin{tabular}{l|rrrrr}
\hline
      & rand500 & rand800 & starbucks & mcd   & subway \bigstrut\\
\hline
naive derandomization & 317.46 & 1061.02 & 231.85 & 1710.84 & 10196.05 \bigstrut[t]\\
incremental derandomization & 0.37  & 1.70  & 1.28  & 3.56  & 11.25 \bigstrut[b]\\
\hline
\hline
speed-up ratio & 849.65 & 623.30 & 180.77 & 480.14 & 906.30 \bigstrut\\
\hline
\end{tabular}%
}
\end{table}

%% file: TAB/ablation_incre_mc.tex
\begin{table}[t!]
\caption{Ablation study on maximum coverage: does incremental derandomization improve the speed?
} \label{tab:ablation_study:incre_mc}
\small
\centering
\scalebox{0.9}{
\begin{tabular}{l|rrrr}
\hline
      & rand500 & rand1000 & twitch & railway \bigstrut \\
\hline
naive derandomization & 240.77 & 1186.06 & 2247.88 & 359.86 \bigstrut[t]\\
incremental derandomization & 0.91  & 2.48  & 1.82  & 1.90 \bigstrut[b]\\
\hline
\hline
speed-up ratio & 265.49 & 478.14 & 1231.81 & 189.52 \bigstrut\\
\hline
\end{tabular}%
}
\end{table}

%% file: TAB/ablation_beta_fl.tex
\begin{table}[t!]
\caption{Ablation study on facility location: how does \ours perform with different constraint coefficients? The results with the constraint coefficient originally used in our experiments are marked in bold. The numbers here are objectives (smaller the better).}\label{tab:ablation_study:beta_fl}
\small
\centering
\scalebox{0.9}{
\begin{tabular}{l|r|r|r|r|r}
\hline
$\beta$ & rand500 & rand800 & starbucks & mcd   & subway \bigstrut\\
\hline
1e-1 & 2.50  & 2.47  & 0.31  & 1.02  & 1.75 \bigstrut\\
\hline
1e-2 & \textbf{2.51} & \textbf{2.38} & \textbf{0.30} & \textbf{0.99} & \textbf{1.86} \bigstrut\\
\hline
1e-3 & 3.19  & 2.79  & 1.85  & 1.41  & 3.83 \bigstrut\\
\hline
\end{tabular}%
}
\end{table}

%% file: TAB/ablation_beta_mc.tex
\begin{table}[t!]
\caption{Ablation study on maximum coverage: how does \ours perform with different constraint coefficients? The results with the constraint coefficient originally used in our experiments are marked in bold. The numbers here are objectives (larger the better).}\label{tab:ablation_study:beta_mc}
\small
\centering
\scalebox{0.9}{
\begin{tabular}{l|r|r|r|r}
\hline
$\beta$ & rand500 & rand1000 & twitch & railway \bigstrut\\
\hline
10 & 43744.80 & 87165.08 & 33801.80 & \textbf{7607.10} \bigstrut\\
\hline
100 & 44382.36 & 88543.73 & \textbf{33828.40} & 7602.00 \bigstrut\\
\hline
500 & \textbf{44622.80} & \textbf{89130.70} & 33825.80 & 7575.50 \bigstrut\\
\hline
\end{tabular}%
}
\end{table}

%% file: TAB/ablation_beta_rc.tex
\begin{table}[t!]
\caption{Ablation study on robust coloring: how does \ours perform with different constraint coefficients? The results with the constraint coefficient originally used in our experiments are marked in bold. The numbers here are objectives (smaller the better).}\label{tab:ablation_study:beta_rc}
\small
\centering
\scalebox{0.9}{
\begin{tabular}{l|r|r|r|r|r|r|r|r}
\hline
\multirow{2}[3]{*}{$\beta$} & \multicolumn{2}{c|}{collins} & \multicolumn{2}{c|}{gavin} & \multicolumn{2}{c|}{krogan} & \multicolumn{2}{c}{ppi} \bigstrut\\
\cline{2-9}      & 18 colors & 25 colors & 8 colors & 15 colors & 8 colors & 15 colors & 47 colors & 50 colors \bigstrut\\
\hline
$\frac{1}{2}\beta_0$ & 78.32 & 15.61 & 46.56 & 6.70  & 52.04 & 0.87  & 2.93  & 1.01 \bigstrut\\
\hline
$\beta_0$ (originally used) & \textbf{82.26} & \textbf{15.16} & \textbf{42.99} & \textbf{6.72}  & \textbf{52.44} & \textbf{0.87}  & \textbf{2.93}  & \textbf{1.01} \bigstrut\\
\hline
$2\beta_0$ & 81.17 & 15.83 & 44.96 & 6.77  & 55.25 & 0.87  & 2.93  & 1.01 \bigstrut\\
\hline
\end{tabular}%
}
\end{table}

%% file: A06dscss.tex
\section{Additional discussions}\label{app:discussions}

\subsection{Inductive Settings and Transductive Settings}\label{subapp:discussions:inductive_and_transductive}
As discussed in Appendix~\ref{subapp:background:diff_optim}, the differentiable optimization in the pipeline can be done either in an inductive setting or in a transductive setting.
Although ideally, a well-trained encoder can save much time without degrading the performance,
in practice,
inductive settings can be less effective~\citep{li2023distribution},
especially when the training set and the test set have very different distributions~\citep{drakulic2023bq}.

As shown in our experimental results, the performance of CardNN~\citep{wang2022towards} highly relies on test-time optimization (compare CardNN and CardNN-noTTO), which implies that the training is actually less essential than the direct optimization on test instances.

For \ours, we also observe that,
when the training set and the test set are from different distributions, the training can be less helpful.
Even applying derandomization on random probabilities can work well sometimes (but not always).

\subsection{Reinforcement Learning and Probabilistic-Method-Based UL4CO}\label{subapp:discussions:rl_and_egn}
The connections between reinforcement learning and 
probabilistic-method-based UL4CO have been discussed by~\cite{wang2022unsupervised}.
The direct connection comes from the fact that
the policy gradient tries to approximate expectations by sampling, while probabilistic-method-based UL4CO aims to directly evaluate expectations.

Differences also exist.
In many cases, RL methods generate decisions in an autoregressive manner, while UL4CO methods try to do it in a one-shot manner~\citep{wang2022towards}, although one-shot RL has also been recently considered~\citep{viquerat2023policy}.
Both the overhead of sampling and the autoregressive decision-encoding
can potentially explain why UL4CO is usually more efficient than RL methods.

We focus on cases under \revise{prevalent} conditions in this work.
In RL, there are also similar subfields studying RL under constraints.
On top of the basic difficulties of ``sampling'', constrained sampling for RL is even trickier and less efficient.
Moreover, the analysis has been limited to simple constraints, e.g., linear and convex ones~\citep{miryoosefi2022simple}.
We believe that this work shows that UL4CO is especially promising in cases under \revise{prevalent} conditions.

\subsection{Local decomposability}\label{subapp:discussions:local_decomp}
\revise{As discussed in Section~\ref{subsec:analy_conds:uncertainty}, a common technique we used in our derivations is decomposing objectives or constraints into sub-terms and analyzing the sub-terms.
Here, we would like to further discuss the importance and implications of this ``local decomposability''.
As discussed in existing works~\citep{ahn2020learning,Jo2023robust}, local decomposability allows us to deal with each sub-term separately, and it is convenient for constructing loss functions.
We would like to point out that this is especially useful for probabilistic-method-based UL4CO, since we can thus use linearity of expectation to take the expectation of each sub-term.}

\revise{Moreover, this technique can be used in combination with another common idea when we construct tight upper-bounds (TUBs), i.e., relaxing the binary ``a constraint is violated'' to ``the number of violations (which was also mentioned in Section~\ref{sec:analy_conds}).
Typically, with such relaxation, we obtain a group of sub-terms, where each sub-term represents the probability of a violation.}

\revise{Notably, not all decomposable objectives are easy to handle, since we also need to consider the number of sub-terms.
Specifically, if an objective (or constraint) can be represented as a polynomial $f(X)$ of degree $d$ with $t$ terms, then the naive computation of $\mathbb{E}[f(X)]$ takes $O(td)$ time, assuming independent Bernoulli variables and multiplication is $O(1)$. Even when the degree $d$ is low, this might still become prohibitive when the number $t$ of terms is high. Regarding the conditions covered in this work:
\begin{itemize}[leftmargin=*]
    \item For cardinality constraints, even for the simplest case where we aim to choose exactly $k$ nodes, the naive computation of expectation takes $O(n^{k+1})$ by enumerating all the $k$-subsets, which would quickly become computationally prohibitive as $k$ increases.\footnote{$f(X)=\sum_{V_k \in \binom{[n]}{k}}\prod_{v \in V_k} X_v \prod_{u \in [n] \setminus V_{k}} (1 - X_u)$, where the degree $d=n$ and the number of terms $t = \binom{n}{k}$.}
    \item For minimum (or maximum) w.r.t. a subset, the naive computation of expectation requires considering all possible decisions, which takes $O(2^n)$.
    \item For cliques, the naive computation of expectation requires considering all possible decisions, which takes $O(2^n)$.
    \item For covering, the constraint can be represented as a polynomial with the degree being the number of neighbors of the target node; even naive evaluation is doable, but our derivation of incremental differences is still nontrivial.
\end{itemize}
}